\documentclass[11pt]{article}
\usepackage[english]{babel}
\usepackage[left=1in,right=1in,top=1in,bottom=1.5in]{geometry} 
\usepackage[skip=5pt, indent=10pt]{parskip}
\usepackage{amssymb,amsmath,amsthm}
\usepackage{xcolor}
\usepackage{nicefrac}
\usepackage[margin=1cm, font=small]{caption}
\usepackage{natbib}
\usepackage[makeroom]{cancel}
\usepackage{bbm}
\usepackage{mathrsfs}
\usepackage{comment}
\usepackage{algorithm}
\usepackage{algpseudocode}
\usepackage{times}
\usepackage{graphicx}
\usepackage{subfig}
\usepackage[most]{tcolorbox}
\usepackage{setspace}
\usepackage{makecell}
\usepackage{upgreek}
\usepackage{mathabx}
\usepackage{hyperref}

\renewcommand{\le}{\leqslant}
\renewcommand{\leq}{\leqslant}
\renewcommand{\ge}{\geqslant}
\renewcommand{\geq}{\geqslant}

\newcommand{\eps}{\varepsilon}
\newcommand{\ups}{\upsilon}

\newcommand{\dd}{\mathrm{d}}

\renewcommand{\Im}{\mathrm{Im}}
\newcommand{\size}{\mathrm{size}}

\newcommand{\F}{\mathsf{F}}
\newcommand{\f}{\mathsf{f}}
\newcommand{\sfp}{\mathsf{p}}

\newcommand{\sfq}{\mathsf{q}}

\newcommand{\E}{\mathbb E}

\newcommand{\N}{\mathbb N}
\newcommand{\p}{\mathbb P}
\newcommand{\R}{\mathbb R}
\newcommand{\bbH}{\mathbb H}

\newcommand{\1}{\mathbbm 1}
\newcommand{\I}{\mathbb I}

\newcommand{\bb}{\mathbf{b}}

\newcommand{\bu}{\mathbf{u}}
\newcommand{\bv}{\mathbf{v}}

\newcommand{\bx}{\textbf{\textit{x}}}

\newcommand{\bX}{\mathbf{X}}
\newcommand{\bY}{\mathbf{Y}}

\newcommand{\bups}{\boldsymbol \ups}
\newcommand{\btheta}{{\boldsymbol \theta}}

\newcommand{\bmu}{\boldsymbol \mu}

\newcommand{\bPsi}{\boldsymbol \Psi}

\newcommand{\bnabla}{\boldsymbol \nabla}
\newcommand{\bzero}{\mathbf{0}}

\newcommand{\cE}{{\mathcal{E}}}
\newcommand{\cF}{{\mathcal{F}}}

\newcommand{\cL}{{\mathcal{L}}}
\newcommand{\cN}{{\mathcal{N}}}
\newcommand{\cO}{{\mathcal{O}}}
\newcommand{\cP}{{\mathcal{P}}}
\newcommand{\cR}{{\mathcal{R}}}

\newcommand{\cX}{{\mathcal{X}}}

\newcommand{\z}{\mathfrak z}

\newcommand{\ttR}{\mathtt R}
\newcommand{\ttZ}{\mathtt Z}

\newcommand{\myendproof}{\hfill$\blacksquare$}

\def\argmin{\operatornamewithlimits{argmin}}

\newtheorem{Th}{Theorem}[section]
\newtheorem{Alg}[Th]{Algorithm}
\newtheorem{As}[Th]{Assumption}

\newtheorem{Lem}[Th]{Lemma}
\newtheorem{Prop}[Th]{Proposition}

\title{Score-based Change Point Detection via\\ Tracking the Best of Infinitely Many Experts}

\author{
Anna Markovich\thanks{HSE University, Russian Federation, aamarkovich@hse.ru}
\and
Nikita Puchkin\thanks{HSE University, Russian Federation, npuchkin@hse.ru}
}

\date{}

\begin{document}
\maketitle

\begin{abstract}%
  We propose an algorithm for nonparametric online change point detection based on sequential score function estimation and the tracking the best expert approach. The core of the procedure is a version of the fixed share forecaster tailored to the case of infinite number of experts and quadratic loss functions. The algorithm shows promising results in numerical experiments on artificial and real-world data sets. Its performance is supported by rigorous high-probability bounds describing behaviour of the test statistic in the pre-change and post-change regimes.%
\end{abstract}


\section{Introduction}

Let $\bX_1, \dots, \bX_T$ be a sequence of i.i.d. random vectors from a Euclidean space $\cX$ such that $\bX_1, \dots, \bX_{\tau^*}$ have a probability density $\sfp$ (with respect to the Lebesgue measure) while $\bX_{\tau^* + 1}, \dots, \bX_T$ are drawn from a density $\sfq \neq \sfp$. The parameter $\tau^* \in \{1, \dots, T - 1\}$ as well as the densities $\sfp$ and $\sfq$ are unknown and, based on the successively arriving observations, our goal is to determine the moment of the distribution change as soon as possible. This problem referred to as online change point detection has a long history going back to classical works of \citet{page54, page55}, \citet{shiryaev61, shiryaev63} and \citet{roberts66} but still continues its extensive development. Researchers thoroughly examined various setups. The most popular one is the problem of mean shift detection \citep{pein17, eichinger18, eh13, wang20, yu20a, yu20b, rinaldo21, chen22, sun22}, that is, $\sfq(\bx) = \sfp(\bx - \bmu)$ for some vector $\bmu$. In addition, it is usually assumed that both the densities $\sfp$ and $\sfq$ are Gaussian or at least have sub-Gaussian tails. A natural extension of the mean shift detection is the framework of parametric change point detection \citep{cao18, maillard19, chen20, dette20, yu20a, chaudhuri21, chaudhuri22, corradin22, sun22, titsias22, lee23, xie23, wu23, wang24}, where the authors suppose that $\sfp$ and $\sfq$ belong to a fixed family of distributions induced by a finite-dimensional parameter. In this case, one has to detect the moment $\tau^*$ when the value of the corresponding parameter changes. We would like to note that, in the literature on parametric change point detection, the authors usually impose strong modeling assumptions on the underlying densities $\sfp$ and $\sfq$. This significantly affects the practical use of such procedures. In the present paper, we are interested in a more general nonparametric change point setup \citep{hero06, hbm08, liu13, xie13, zou14, li15, biau16, ga18, arlot19, kurt21, xie21b, padilla21b, shin22, ferrari23, puchkin23}. We do not make any presumptions on how $\sfq$ differs from $\sfp$. For our purposes, it is enough to suppose that the scores $\bnabla \log \sfp$ and $\bnabla \log \sfq$ corresponding to the pre-change and post-change distributions are smooth and can be approximated well by an appropriate system of functions.
Finally, it is worth mentioning that we are interested in sequential change point detection. This means that we focus our attention on rapid detection of change emergence. There is a vast of literature devoted to localization and estimation of $\tau^*$ from a retrospectively observed time series \citep{ds01, zou14, mj14, ds15, biau16, korkas17, ga18, arlot19, padilla21, corradin22, malte22}. This problem called offline change point detection goes beyond the scope of the present paper.

Let us briefly describe the idea of our algorithm.  We consider the problem of change point detection through the lens of sequential prediction with expert advice and online convex optimization. In \citep{cao18, goldman23}, the authors have already designed the quickest detection algorithms based on online mirror descent and follow-the-approximate-leader strategy, respectively, but our approach is different. We rely on the standard parametric modelling to approximate the unknown densities $\sfp$ and $\sfq$. For this purpose, we use a reference class of densities $\cP = \{ \sfp_\btheta : \btheta \in \R^d \}$
with $\sfp_\btheta$ of the form
\[
    \log \sfp_\btheta(\bx) = \bPsi(\bx)^\top \btheta - \Phi(\btheta),
    \quad \text{where} \quad
    \Phi(\btheta) = \log \int\limits_{\cX} e^{\btheta^\top \bPsi(\bx)} \, \dd \bx.
\]
Here $\bPsi : \bx \mapsto \big(\psi_1(\bx), \dots, \psi_d(\bx) \big)^\top$ is a fixed mapping. From now on, we can apply online prediction algorithms to determine the best parametric fit for the pre-change and post-change densities $\sfp$ and $\sfq$. In \citep{cao18}, the authors measured the quality of forecasts with the negative logarithmic loss. As a consequence, they had to calculate $\Phi(\btheta)$ for various values of $\btheta$ on each iteration. Unfortunately, numerical computation of an integral over a high-dimensional space is extremely time consuming. Hence, the algorithm of \cite{cao18} is useful only when $\Phi(\btheta)$ admits a closed-form expression. This drastically narrows down the practical use of their procedure. To avoid such issues, we measure the closeness of $\sfp_\btheta$ to the target densities $\sfp$ and $\sfq$ with the Fisher divergence (see eq.~\eqref{eq:fisher_divergence} for the definition). Given $\bX_t$, it is straightforward to construct an unbiased estimate $\ell_t(\btheta)$ of
\[
    \F(\sfp, \sfp_\btheta) - \frac12 \E_{\bX \sim \sfp} \left\| \bnabla \log \sfp(\bX) \right\|^2
    \quad \text{or} \quad
    \F(\sfq, \sfp_\btheta) - \frac12 \E_{\bX \sim \sfq} \left\| \bnabla \log \sfq(\bX) \right\|^2
\]
using Green's first identity (see Section~\ref{sec:application} for the details). This estimate is then used for construction of a loss function on the $t$-th round. As long as $t \leq \tau^*$ the problem of online prediction can be easily solved using, for instance, the exponentially weighted average forecaster with Gaussian prior (see Section~\ref{sec:ew}). This algorithm will have a small regret compared to the best static expert $\btheta_{1 : \tau^*}^\circ$. However, when $t > \tau^*$, the best expert $\btheta_{\tau^* + 1 : T}^\circ$ on the segment $[\tau^* + 1, T]$ will differ from $\btheta_{1 : \tau^*}^\circ$. This suggests us
to use online forecasters, which are able to compete with combinations of experts. For this reason, we develop a version of the fixed share algorithm \citep{herbster98} for infinite number of experts and construct the test statistic $\widehat S_t = \widehat L_{1 : t}^{EW} - \widehat L_{1 : t}^{FS}$, where $\widehat L_{1 : t}^{EW}$ and $\widehat L_{1 : t}^{FS}$ stand for cumulative losses of the exponentially weighted average and the fixed share forecasters, respectively. If $t \leq \tau^*$, then the standard exponential weighting performs nearly as well as the fixed share predictor and $\widehat S_t$ stays small. The situation changes when $t$ starts exceeding $\tau^*$. In this case $\widehat S_t$ rapidly grows, because of the ability of fixed share to adapt to changes in the underlying distribution instantly. We would like to mention that similar ideas were used in \citep{gokcesu18, yamanishi18} in the context of anomaly detection. \cite{gokcesu18} use the idea of tracking the best expert and derives bounds on the dynamic regret. However, the algorithm of \citet{gokcesu18} has the same limitations as the one of \cite{cao18} and is applicable to parametric setups only. Besides, \citet{gokcesu18} does not discuss how the regret bounds are related to type-I and type-II errors of their procedure. In \citep{yamanishi18}, the authors consider possible changes in the reference class $\cP$ rather than in switching of the best static predictor.

\medskip

\noindent
\textbf{Contribution.}\quad The main contribution of the paper can be summarized as follows.
\begin{itemize}
    \item We propose an algorithm for nonparametric online change point detection. It is based on a version of the fixed share forecaster tailored for the problem of online prediction with expert advice with infinite number of static experts and quadratic losses $\ell_t : \R^d \rightarrow \R$ of the form
    \begin{equation}
        \label{eq:loss}
        \ell_t(\btheta) = \frac12 \btheta^\top A_t \btheta - \bb_t^\top \btheta,
        \quad \text{where $A_t \succeq O_d$, $\bb_t \in \R^d$ for all $t \in \{1, \dots, T\}$}.
    \end{equation}
    \item We carry out theoretical study of the algorithm and derive non-asymptotic high-probability bounds on its test statistic in the pre-change and post-change regimes under realistic assumptions.
    \item We illustrate algorithm's performance with numerical experiments on both artificial and real-world data sets.
\end{itemize}

\noindent
\textbf{Paper structure.}\quad The rest of the paper is organized as follows. In Section~\ref{sec:online_learning}, we briefly introduce the problem of online prediction with expert advice, discuss the exponentially weighted average forecaster and present a version of the fixed share algorithm. In Section~\ref{sec:application}, we suggest an online change-point detection algorithm and elaborate on its theoretical properties. We also illustrate its performance with numerical experiments in Section~\ref{sec:numerical}. Proofs of the theoretical results and additional numerical experiments are moved to Appendix.

\medskip

\noindent
\textbf{Notation.}\quad Throughout the paper we use the following notation. Boldface font is reserved for vectors while scalars and matrices are displayed in regular font. For any \(x \in \R\) we denote \((x)_+ = \max(x, 0)\).
For any $t \in \{1, \dots, T\}$, $\ell_t : \R^d \rightarrow \R$ stands for the loss function observed on the round $t$. For any $1 \leq s \leq t \leq T$,
\[
    L_{s:t}(\btheta) = \sum\limits_{j = s}^t \ell_j(\btheta)
\]
denotes the cumulative loss suffered by the expert $\btheta$ on rounds $s, s + 1, \dots, t$. We also adopt the notation $L_{s:t}(\btheta) = 0$ for all $s > t$ and $\btheta \in \R^d$. Given a prior distribution $\pi$ on $\R^d$ (always clear from context), we introduce the exponentially weighted average
\begin{equation}
    \label{eq:z_st}
    \widehat \btheta_{s:t}(\eta) = \frac1{Z_{s:t}(\eta)} \int\limits_{\R^d} \btheta e^{-\eta L_{s:t}(\btheta)} \pi(\btheta) \, \dd \btheta,
    \quad \text{where} \quad
    Z_{s:t}(\eta) = \int\limits_{\R^d} e^{-\eta L_{s:t}(\btheta)} \pi(\btheta) \, \dd \btheta.
\end{equation}
For any probability densities $\sfp_0$ and $\sfp_1$, $\F(\sfp_0, \sfp_1)$ stands for the Fisher divergence defined as follows:
\begin{equation}
    \label{eq:fisher_divergence}
    \F(\sfp_0, \sfp_1)
    = \frac12 \int\limits_{\cX} \left\| \bnabla \log \sfp_0(\bx) - \bnabla \log \sfp_1(\bx) \right\|^2 \sfp_0(\bx) \, \dd \bx.
\end{equation}
The identity matrix of size $d \times d$ and the zero matrix of the same shape are denoted by $I_d$ and $O_d$, respectively. For a matrix $A$, $\Im(A)$ denotes its range and $A^\dag$ is its pseudoinverse. Finally, for any two functions $f$ and $g$, the relations $f \lesssim g$ and $g \gtrsim f$ are equivalent to $f = \cO(g)$.

\section{Online Learning with Quadratic Loss}
\label{sec:online_learning}

In the problem of prediction with expert advice, a learner aims to predict a sequence of outcomes based on a finite or infinite set $\Theta$ of reference forecasts. In the present paper, we focus on a bit more specific setup than described in the book \citep[Section 2]{cesa-bianchi06a}. To be more precise, we assume that $\Theta = \R^d$ for some $d \in \N$, that is, each expert is indexed with a $d$-dimensional vector. Moreover, we suppose that, for any $\btheta \in \R^d$, the expert $\btheta$ makes a prediction $\btheta$ on each round. In other words, we deal with the case of constant experts. We would like to emphasize that the choice $\Theta = \R^d$ allows us to handle scenarios of abrupt changes, when the densities $\sfp$ and $\sfq$ (and, hence, their corresponding approximations from the parametric class $\cP$) are very different.  
The game proceeds as follows. An adversary generates loss functions $\ell_t : \R^d \rightarrow \R$, $1 \leq t \leq T$, but does not reveal them to the learner. On a round $t \in \{1, \dots, T\}$, the forecaster makes his prediction $\widehat \btheta_t \in \R^d$ based on experts' performance on the previous rounds $1, \dots, t - 1$. After that, the adversary discloses $\ell_t$ and the learner suffers the loss $\ell_t(\widehat \btheta_t)$. This is the so-called \emph{oblivious opponent} setup, which is usual for analysis of the exponential weighting and fixed share algorithms. The performance after $T$ rounds is measured by the excess cumulative loss compared to the best expert, also referred to as \emph{regret}:
\begin{equation}
    \label{eq:static_regret}
    R_{1:T}
    = \widehat L_{1:T} - \inf\limits_{\btheta \in \R^d} L_{1:T}(\btheta)
    = \sum\limits_{t = 1}^T \ell_t(\widehat \btheta_t) - \inf\limits_{\btheta \in \R^d} \sum\limits_{t = 1}^T \ell_t(\btheta).
\end{equation}
In what follows, we assume that, for each $t \in \{1, \dots, T\}$, the loss function $\ell_t$ is quadratic~\eqref{eq:loss}.

\subsection{Exponentially Weighted Average Forecaster}
\label{sec:ew}  

Exponential weighting is one of the most popular strategies for prediction of individual sequences. Let $\pi$ be a prior distribution over $\R^d$. Let us recall that, for any $\btheta \in \R^d$ and $t \in \{1, \dots, T\}$, we denote the cumulative loss of the expert $\btheta$ after $t$ rounds by
\[
    L_{1:t}(\btheta) = \sum\limits_{j = 1}^t \ell_j(\btheta).
\]
We also adopt the notation $L_{1:0}(\btheta) = 0$ for any $\btheta \in \R^d$. Given a non-increasing sequence $\{\eta_t : 1 \leq t \leq T\}$ of positive numbers, on the round $t$, the exponentially weighted forecaster makes a prediction $\widehat \btheta{}_t^{EW} = \widehat \btheta_{1:t-1}(\eta_t)$ with $\widehat \btheta_{1:t-1}(\eta_t)$ given by \eqref{eq:z_st}.
In other words, $\widehat \btheta{}_t^{EW}$ is the mean with respect to the posterior measure
\begin{equation}
    \label{eq:posterior}
    \frac{e^{-\eta_t L_{1:t-1}(\btheta)} \pi(\btheta)}{Z_{1:t-1}(\eta_t)}.
\end{equation}
When the loss functions $\ell_1, \dots, \ell_T$ are  quadratic, the standard choice for the prior $\pi$ is the Gaussian distribution $\cN(\bzero, \lambda^{-1} I_d)$. The reason is that the exponentiallly weighted forecast $\widehat \btheta{}_t^{EW}$ can be computed explicitly. We provide the corresponding lemma below.

\begin{Lem}
    \label{lem:ew}
    Assume that the loss function $\ell_t$ is of the form~\eqref{eq:loss} for any $t \in \{1, \dots, T\}$. For any $\eta > 0$, let $\widehat \btheta_{s:t}(\eta)$ be the exponentially weighted average \eqref{eq:z_st} with the Gaussian prior
    \begin{equation}
        \label{eq:ew_prior}
        \pi(\btheta) = \left( \frac{\lambda}{2 \pi} \right)^{d / 2} e^{-\lambda \|\btheta\|^2 / 2}.
    \end{equation}
    Then it holds that
    \begin{equation}
        \label{eq:ew_st}
        \widehat \btheta_{s:t}(\eta) = \left( \sum\limits_{j = s}^t A_j + \frac{\lambda}\eta I_d \right)^{-1} \sum\limits_{j = s}^t \bb_j.
    \end{equation}
\end{Lem}

The proof of Lemma~\ref{lem:ew} is straightforward. However, we provide it in Section~\ref{sec:lem_ew_proof} to make the paper self-contained. The expression~\eqref{eq:ew_st} helps to formulate the exponential weighting procedure in a simple and closed form.

\subsection{Fixed Share Forecaster}

In the standard online learning framework, the performance of a forecaster is measured in terms of the regret~\eqref{eq:static_regret} with respect to the best expert. In \citep{auer98, herbster98}, the authors compared the cumulative loss $\widehat L_{1:T}$ of the forecaster with the one of a combination of experts. For any $\btheta_1, \dots, \btheta_T \in \R^d$, let $(\btheta_1, \dots, \btheta_T)$ stand for a \emph{compound expert}, that is, a forecaster, which follows the prediction of the expert $\btheta_t$ on the round $t$. It is straightforward to observe that the cumulative loss of the compound expert $(\btheta_1, \dots, \btheta_T)$ after $t$ rounds is equal to
\[
    \cL_t(\btheta_1, \dots, \btheta_T) = \sum\limits_{s = 1}^t \ell_s(\btheta_s).
\]
Of course, it would be unfair to compare the forecaster with all compound experts. For this reason, \cite{auer98} and \cite{herbster98} introduced the notion of size of a compound expert defined as
\[
    \size(\btheta_1, \dots, \btheta_T)
    = \sum\limits_{t = 2}^T \1(\btheta_t \neq \btheta_{t - 1})
\]
and restricted their attention on \emph{dynamic} (also called switching or shifting) regret with respect to the best compound expert of size at most $m$:
\[
    \cR_{1:T} = \widehat L_{1:T} - \inf\limits_{\substack{(\btheta_1, \dots, \btheta_T):\\ \size(\btheta_1, \dots, \btheta_T) \leq m}} \cL_T(\btheta_1, \dots, \btheta_T).
\]
Here $m \leq T$ is a predefined parameter. Obviously, if $m = 0$, then the switching regret $\cR_{1:T}$ coincides with $R_{1:T}$. The ambitious problem of tracking the best expert gained much attention in the last three decades starting from the pioneering works of \cite{willems96, willems97, herbster98, shamir99, vovk99}. While \cite{willems96, willems97, shamir99} were bothered with a particular prediction problem with logarithmic loss, \citet{herbster98} designed a universal algorithm predicting nearly as well as the best compound expert of size $m$ under quite general assumptions on the loss function. They suggested the fixed share forecaster, which is nothing but the exponential weighting in the space of all compound experts with a special prior. Their idea was simple and elegant and helped to overcome computational issues arising during the brute force implementation of the exponentially weighted averaging, which requires a lot of resources even in the case of finite set of experts (see, for instance, \citeauthor{cesa-bianchi06a}, \citeyear[Section 5.2]{cesa-bianchi06a}). The main advantage of the fixed share is that, despite the aggregation of all possible compound experts, it remains computationally efficient due to the apposite prior. The algorithm became very popular with the research community, and its properties (of both randomized and deterministic versions) were then intensively studied in numerous subsequent works (for example, \citealp{vovk99, bousquet02, gyorgy05, gyorgy08, cesa-bianchi12, gyorgy12, adamskiy16} to name a few). In \citep{bousquet02, adamskiy16}, the authors went further and examined general schemes for construction of online learning algorithms with low switching regret. Unfortunately, the algorithm of \cite{herbster98} (as well as many follow-up papers) deals with the case of finite number of experts.
One of the ways to overcome this limit is to consider setups with a growing pool of experts. For instance, in \citep{hazan07b, hazan09, mourtada17}, the authors propose meta-algorithms for aggregation of an increasing number of predictions at each step while maintaining strong dynamic regret bounds. These methods can also be adapted for construction of change point detection procedures in the same spirit as we describe in Section \ref{sec:application}. In our approach, however, we use a slightly different strategy. We develop a version of the fixed share forecaster tailored for the specific setup to ensure that its cumulative loss is as close as possible to that of the exponentially weighted average forecaster in the pre-change regime. The advantage of this strategy becomes evident in Section \ref{sec:numerical}, where our algorithm outperforms a procedure based on the Follow-the-Leading-History (FLH) meta-algorithm \citep[Section 3]{hazan07b}.

\subsubsection{Fixed Share Forecaster for the Infinite Class of Experts}
\label{sec:fixed-share}

Following the general idea of \cite{herbster98}, we define the fixed share forecaster as exponentially weighted predictor with a special prior $\rho$ over the set of compound experts $(\btheta_1, \dots, \btheta_T)$. We suggest using
\begin{equation}
    \label{eq:fs_prior}
    \rho(\btheta_1, \dots, \btheta_T) = \pi(\btheta_1) \prod\limits_{t = 2}^T \f(\btheta_t \,\vert\, \btheta_{t - 1}),
    \; \text{where} \;
    \f(\btheta \,\vert\, \btheta_{t - 1}) = \alpha \pi(\btheta) + (1 - \alpha) \delta(\btheta - \btheta_{t - 1}).
\end{equation}
Here and further in the paper, $\delta$ denotes the Dirac delta function. As before, $\pi(\btheta)$ stands for the probability density function of the Gaussian distribution $\cN(\bzero, \lambda^{-1} I_d)$, $\lambda > 0$, as defined in~\eqref{eq:ew_prior}.
Given a non-increasing sequence $\{\eta_t : 1 \leq t \leq T\}$ of positive numbers, the fixed share forecaster makes a prediction $\widehat\btheta{}_t^{FS} = \widetilde \btheta_t(\eta_t)$ on the round $t$, where, for any $\eta > 0$ and $t \in \{1, \dots, T\}$,
\begin{equation}
    \label{eq:exponential_weights_compound}
    \widetilde \btheta_t(\eta) = \frac{\int \btheta_t e^{-\eta \cL_{t - 1}(\btheta_1, \dots, \btheta_T)} \rho(\btheta_1, \dots, \btheta_T) \, \dd \btheta_1 \dots \dd \btheta_T}{\int e^{-\eta \cL_{t - 1}(\btheta_1, \dots, \btheta_T)} \rho(\btheta_1, \dots, \btheta_T) \, \dd \btheta_1 \dots \dd \btheta_T}.
\end{equation}

Straightforward computation of the integrals in~\eqref{eq:exponential_weights_compound} requires a lot of efforts even in the one-dimensional case. Fortunately, there is a much simpler way to calculate $\widetilde \btheta_t(\eta)$ in a recursive manner.

\begin{Lem}
    \label{lem:v}
    Let us fix any $\eta > 0$. For any $s, t \in \{1, \dots, T\}$, $s \leq t$, let $\widehat \btheta_{s:t}(\eta)$ and $Z_{s:t}(\eta)$ be as defined in~\eqref{eq:z_st}. Then $\widetilde \btheta_1(\eta) = \bzero$, $\widetilde \btheta_2(\eta) = (1 - \alpha) \widehat \btheta_{1:1}(\eta)$, and for any $3 \leq t \leq T$ the forecast $\widetilde \btheta_t(\eta)$ obeys the recurrent relation
    \begin{align}
        \label{eq:widetilde_theta}
        \widetilde \btheta_t(\eta)
        = \frac{1 - \alpha}{V_{t - 1}(\eta)}
        &\notag
        \Bigg( (1 - \alpha)^{t - 2} \, Z_{1 : t - 1}(\eta) \widehat\btheta_{1:t - 1}(\eta)
        \\&\quad
        + \alpha \sum\limits_{s = 0}^{t - 3} (1 - \alpha)^s \, V_{t - 2 - s}(\eta) \, Z_{t - 1 - s : t - 1}(\eta) \widehat\btheta_{t - 1 - s : t - 1}(\eta) \Bigg),
    \end{align}
    where $V_1(\eta) = Z_{1:1}(\eta)$ and 
    \begin{equation}
        \label{eq:v}
        V_t(\eta) = (1 - \alpha)^{t - 1} \, Z_{1 : t}(\eta) + \alpha \sum\limits_{s = 0}^{t - 2} (1 - \alpha)^s \, V_{t - 1 - s}(\eta) \, Z_{t - s : t}(\eta)
        \quad \text{for any $2 \leq t \leq T$.}
    \end{equation}
\end{Lem}

The proof of Lemma~\ref{lem:v} is moved to Section~\ref{sec:lem_v_proof}. The expressions~\eqref{eq:v} and~\eqref{eq:widetilde_theta} resemble of the result of \cite{gyorgy08} (Lemma 1, see also \citeauthor{cesa-bianchi06a}, \citeyear[Lemma 5.4]{cesa-bianchi06a}). Finally, for any $s, t \in \{1, \dots, T\}$, $s \leq t$, $Z_{s:t}(\eta)$ can be computed explicitly.

\begin{Lem}
    \label{lem:z}
    Let us fix any positive numbers $\eta$ and $\lambda$. For any $s, t \in \{1, \dots, T\}$, $s \leq t$, let $Z_{s:t}(\eta)$ be as defined in~\eqref{eq:z_st} with the prior $\pi$, given by~\eqref{eq:ew_prior}. Then it holds that
    \begin{equation}
        \label{eq:z_st_explicit}
        \hspace{-0.18cm}
        Z_{s : t}(\eta)
        = \left( \frac{\lambda}{\eta} \right)^{d/2} \det\left( \; \sum\limits_{j = s}^t A_j + \frac{\lambda}{\eta} I_d \right)^{-1/2} \exp\left\{ \frac{\eta}2 \left\| \left( \; \sum\limits_{j = s}^t A_j + \frac{\lambda}{\eta} I_d \right)^{-1/2} \sum\limits_{j = s}^t \bb_j \right\|^2 \right\}.
    \end{equation}
\end{Lem}
We provide the proof of Lemma~\ref{lem:z} in Section~\ref{sec:lem_z_proof}.

\section{Application to Change Point Detection}
\label{sec:application}

The exponentially weighted average and the fixed share forecasters can be applied to online change point detection.
Let us elaborate on the idea briefly described in the introduction. Let us recall that we observe a sequence of independent random vectors $\bX_1, \dots, \bX_T$ such that $\bX_1, \dots, \bX_{\tau^*} \sim \sfp$ and $\bX_{\tau^* + 1}, \dots, \bX_T \sim \sfq$, where $\sfp$ and $\sfq$ are unknown probability densities on a Euclidean space $\cX$. We approximate $\sfp$ and $\sfq$ using a parametric reference class of densities $\cP = \{ \sfp_\btheta : \btheta \in \R^d \}$ with $\sfp_\btheta$ of the form
\[
    \log \sfp_\btheta(\bx) = \bPsi(\bx)^\top \btheta - \Phi(\btheta),
    \quad \text{where} \quad
    \Phi(\btheta) = \log \int\limits_{\cX} e^{\btheta^\top \bPsi(\bx)} \, \dd \bx
\]
and $\bPsi : \bx \mapsto \big(\psi_1(\bx), \dots, \psi_d(\bx) \big)^\top$ is a fixed mapping. We consider the problem of change point detection through the lens of online learning. To avoid time-consuming computations of $\Phi(\btheta)$, we measure the closeness of $\sfp_\btheta$  to the target densities $\sfp$ and $\sfq$ with the Fisher divergence defined as 
\begin{align*}
    \F(\sfp, \sfp_\btheta)
    &
    = \frac12 \int\limits_{\cX} \left\| \bnabla \log \sfp(\bx) - \bnabla \log \sfp_\btheta(\bx) \right\|^2 \sfp(\bx) \, \dd \bx
    \\&
    = \frac12 \E_{\bX \sim \sfp} \left\| \bnabla \log \sfp(\bX) - \nabla \bPsi(\bX)^\top \btheta \right\|^2.
\end{align*}
Using Green's first identity, we can rewrite the right-hand side in the following form:
\begin{align*}
    &
    \frac12 \E_{\bX \sim \sfp} \left\| \bnabla \log \sfp(\bX) - \nabla \bPsi(\bX)^\top \btheta \right\|^2
    = \frac12  \E_{\bX \sim \sfp} \left( \left\| \nabla \bPsi(\bX)^\top \btheta \right\|^2 + 2 \Delta \bPsi(\bX)^\top \btheta + \left\| \bnabla \log \sfp(\bX) \right\|^2 \right).
\end{align*}
Similarly, it holds that 
\[
    \F(\sfq, \sfp_\btheta) - \frac12 \E_{\bX \sim \sfq} \left\| \bnabla \log \sfq(\bX) \right\|^2
    = \frac12  \E_{\bX \sim \sfq} \left( \left\| \nabla \bPsi(\bX)^\top \btheta \right\|^2 + 2 \Delta \bPsi(\bX)^\top \btheta \right).
\]
For any $t \in \{1, \dots, T\}$, let us consider
\begin{equation}
    \label{eq:ellt}
    \ell_t(\btheta) = \frac12 \btheta^\top A_t \btheta - \bb_t^\top \btheta,
    \quad \text{where} \quad
    A_t = \nabla \bPsi(\bX_t) \nabla \bPsi(\bX_t)^\top + \gamma I_d 
    \text{ and }
    \bb_t = -\Delta \bPsi(\bX_t).
\end{equation}
Here $\Delta$ stands for the Laplacian operator applied to the vector-valued function $\bPsi$ in a componentwise fashion. Regardless the distribution of $\bX_t$, the loss $\ell_t(\btheta)$ is an unbiased estimate of
\[
    \frac12  \E_{\bX_t} \left( \left\| \nabla \bPsi(\bX_t)^\top \btheta \right\|^2 + 2 \Delta \bPsi(\bX_t)^\top \btheta \right).
\]
We play the game of online prediction with expert advice with the loss functions~\eqref{eq:ellt} and compare the cumulative losses suffered by the exponentially weighted average and the fixed share forecasters. If $t \leq \tau^*$, then $\widehat L_{1:t}^{EW}$ and $\widehat L_{1:t}^{FS}$ will not differ too much. However, the situation dramatically changes when $t$ exceeds $\tau^*$. While the exponentially weighted average sticks to the expert $\btheta_{1:\tau^*}^\circ \in \argmin_{\btheta} L_{1:\tau^*}(\btheta)$ successfully performed during the first $\tau^*$ rounds, the fixed share algorithm rapidly switches to a new expert with a small post-change cumulative loss $L_{\tau^* + 1:t}(\btheta)$. For this reason, the difference between $\widehat L_{1:t}^{EW}$ and $\widehat L_{1:t}^{FS}$ starts growing right after $\tau^*$. When it exceeds a predefined threshold, we report a change point occurrence.
We provide a pseudocode of the described online change point detection algorithm below.

\medskip
\begin{tcolorbox}[enhanced] 
\begin{spacing}{0.3}
    \begin{Alg}[Score-based change point tracking]
    \hfill
    \label{alg:score-based}
    \begin{itemize}
        \item \textbf{Input:} a non-increasing sequence $\{\eta_t : 1 \leq t \leq T\}$ of positive numbers, a regularization parameter $\lambda > 0$,  a shifting parameter $\alpha \in [0, 1]$, and a threshold $\z > 0$.
        \item \textbf{Initialization:} $\widehat S_0 = 0$.
        \item \textbf{For} $t = 1, 2, \dots, T$ \textbf{do} the following.
        \begin{enumerate}
            \item Compute the prediction $\widehat \btheta{}_t^{EW} = \widehat \btheta_{1:t-1}(\eta_t)$ according to~\eqref{eq:ew_st}.
            \item Calculate the prediction $\widehat \btheta{}_t^{FS} = \widetilde \btheta_{t}(\eta_t)$ according to~\eqref{eq:widetilde_theta} and~\eqref{eq:v} (see Lemma~\ref{lem:v}).
            \item Suffer the losses $\ell_t(\widehat \btheta{}_t^{EW})$ and $\ell_t(\widehat \btheta{}_t^{FS})$ with $\ell_t$ defined in~\eqref{eq:ellt}.
            \item Compute the test statistic:
                \[
                    \widehat{S}_t = \widehat{L}_{1:t}^{EW} - \widehat{L}_{1:t}^{FS}
                    = \widehat{S}_{t - 1} + \left(\ell_t(\widehat \btheta{}_t^{EW}) - \ell_t(\widehat \btheta{}_t^{FS}) \right).
                \]
            \item If $\widehat{S}_t > \z$, terminate the procedure and report the change point occurrence.
        \end{enumerate}
        \item \textbf{Return.}
    \end{itemize}
    \end{Alg} 
    \end{spacing}
\end{tcolorbox}
\medskip

Let us elaborate on the complexity of Algorithm~\ref{alg:score-based}. In a general case $\eta_1 \geq \dots \geq \eta_T$, one needs $\cO(d^3 t)$ operations to compute $\widehat \btheta_{s:t - 1}(\eta_t)$ and  $Z_{s:t - 1}(\eta_t)$ for all $s \in \{1, \dots, t - 1\}$, where $\widehat \btheta_{1:t - 1}(\eta_t) = \widehat \btheta{}_t^{EW}$. Here we assume that one needs $\cO(d^2)$ operations to calculate $\sum_{j = s}^t A_j$ and $\sum_{j = s}^t \bb_j$ based on the sums $\sum_{j = s}^{t - 1} A_j$ and $\sum_{j = s}^{t - 1} \bb_j$ computed during the previous iterations. The most time-consuming part is to compute $V_{t - 1}(\eta_t)$ using the recurrence relation~\eqref{eq:v}. That requires recalculating  $Z_{s:k}(\eta_t)$ for all $1 \le s \le k \le t-1$, so the cost is  $\cO(d^3 t^2)$ operations. The fixed share prediction $\widehat \btheta{}_t^{FS}$ then requires additional $\cO(dt)$ operations (vector sums), which is dominated by the $\cO(d^3 t)$ cost of the earlier steps. Computation of the losses  $\ell_t(\widehat \btheta{}_t^{EW})$, $\ell_t(\widehat \btheta{}_t^{FS})$ requires \(\cO(d^2)\) operations and the test statistic $\widehat{S}_t$ costs $\cO(1)$, since the cumulative losses from previous rounds are stored. Hence, the total runtime of Algorithm~\ref{alg:score-based} is $\cO(d^3 T^3)$. The computational complexity is much lower when we use the constant learning rate $\eta_1 = \ldots = \eta_T = \eta$. In this case, we can compute $V_{t - 1}(\eta)$ in just $\cO(t)$ operations using the values $V_{t - 2}(\eta), \dots, V_{1}(\eta)$ obtained in the previous rounds. This leads to $\cO(d^3 t)$ operations per $t$-th round and the total runtime $\cO(d^3 T^2)$.

We proceed with theoretical properties of the test statistic $\widehat S_t = \widehat L_{1:t}^{EW} - \widehat L_{1:t}^{FS}$. We impose the following assumptions on the outcomes $(A_t, \bb_t)$, $1 \leq t \leq T$.
\begin{As}
    \label{as:subexp}
    Let \(\overline{A}_t = \E A_t\) be a symmetric positive definite matrix. For some \(B > 0\) for all \(t = 1, \dots, T\) we assume that
    \begin{align*}
        \sup_{\|\bu\| = 1} \left\| \bu^\top \big(\overline{A}_t\big)^{-1/2} A_t \big(\overline{A}_t\big)^{-1/2} \bu \right\|_{\psi_1} \le B
        \quad \text{and} \quad
        \sup_{\|\bu\| = 1} \left\| \bu^\top \big(\overline{A}_t\big)^{-1/2} \bb_t \right\|_{\psi_1} \le B.
    \end{align*}
\end{As}
In our experiments, we usually choose $\bPsi(\bx) = (x_1, \dots, x_d, x_1^2, \dots, x_d^2)^\top$ (see Section \ref{sec:numerical} and Appendix \ref{sec:numerical_appendix}). In this case, Assumption \ref{as:subexp} will be satisfied if $\bX_t$ is a sub-Gaussian random vector. This is a sufficiently mild assumption on the underlying densities $\sfp$ and $\sfq$.

In the pre-change regime (that is, when $\bX_1, \dots, \bX_t$ are i.i.d. random vectors), the cumulative losses $\widehat L_{1:t}^{EW}$ and $\widehat L_{1:t}^{FS}$ will be close to each other. This will result into moderate values of the test statistic $\widehat S_t$. To be more specific, the following high-probability bound holds.

\begin{Th}
    \label{th:main_rl}
    Grant Assumption \ref{as:subexp}. For any $\delta \in (0, 1/7)$, let us define
    \begin{equation}
        \label{eq:r_stationary}
        \ttR = \left\|\big(\overline A_1\big)^{-1/2} \overline \bb_1 \right\| \vee \left( \frac{4B \big\| \overline{A}_1 \big\|}{\gamma} \big( d \log 6 + 2 \log \tau^* + \log(3/\delta) \big) \right).
    \end{equation}
    Assume that the learning rates $\eta_1 \geq \dots \geq \eta_{\tau^*}$ are chosen in such a way that
    \begin{equation}
        \label{eq:eta_condition_stationary}
        4 \ttR B \big(1 \vee 4 \gamma^2 \ttR \big) \big( d \log 6 + \log \tau^* + \log(3/\delta) \big) \leq \frac1{\eta_{\tau^*}}.
    \end{equation}
    Then, with probability at least $(1 - 7\delta)$, the test statistic $\widehat{S}_t = \widehat{L}_{1:t}^{EW} - \widehat{L}_{1:t}^{FS}$ satisfies the inequality
    \begin{align}
        \label{eq:test_stat_stationary}
        \widehat S_t
        &\notag
        \leq \frac{\lambda \, \ttR^2}{2 \big\| \overline A_1 \big\| \eta_t}
        + \frac{d}{2\eta_t} \log\left(1 + \frac{\gamma B \, \ttR \, \eta_t \tau^*}{\lambda} \right)
        \\&\quad
        + 4B \left(3e \ttR^2 \log(8B) + 4B\right) \log(\tau^*/\delta)
        + 12B \, \ttR \ttZ (1 + \ttR) \log(\tau^*/\delta)
    \end{align}
    simultaneously for all $t \in \{1, \dots, \tau^*\}$, where
    \[
        \ttZ = 1 \vee \log \frac{e(1 + \ttR)\sqrt{2B}}{\sqrt{3e \ttR^2 \log(8B) + 4B}}.
    \]
\end{Th}

Theorem \ref{th:main_rl} claims that in the pre-change regime
\[
    \widehat S_t
    \lesssim \ttR^2 + \log(1 + \ttR \tau^*) + \ttR^2 \log(\tau^*/\delta)
    \lesssim \log^3(\tau^*/\delta)
    \quad \text{for all $1 \leq t \leq \tau^*$}
\]
with high probability. In other words, when $\bX_1, \dots, \bX_t$ are i.i.d. random vectors, the value of $\widehat S_t$ does not exceed $\cO(\log^3 \tau^*)$. According to our numerical experiments (see Appendix \ref{sec:numerical_appendix}), its actual behaviour is even more optimistic. For instance, in the experiment with synthetic Gaussian sequence, the value of the test statistic is very close to zero. However, the situation drastically changes when a change point occurs.

\begin{Th}
    \label{th:main_dd}
    Grant Assumption \ref{as:subexp} and set $\alpha = 1/T$ in the fixed share algorithm. For any $t \in \{1, \dots, T\}$ define $\overline{A}_t = \E A_t$, $\overline \bb_t = \E \bb_t$, $\overline{\ell}_t(\btheta) = \E_{\bX_t} \ell_t(\btheta)$, and 
    \[
        \btheta_t^\star
        = \argmin_{\btheta \in \R^d} \overline{\ell}_t(\btheta)
        = \argmin_{\btheta \in \R^d} \E_{\bX_t} \ell_t(\btheta)
        = \big( \overline{A}_t \big)^{-1} \overline{\bb}_t.
    \]
    For any $\delta \in (0, 1/8)$, let us introduce
    \begin{equation}
        \label{eq:r_change_point}
        \ttR = \left\|\big(\overline A_1\big)^{-1/2} \overline \bb_1 \right\| \vee \left( \frac{4B}{\gamma} \left( \big\| \overline{A}_1 \big\|^{1/2} \vee \big\| \overline{A}_{\tau^* + 1} \big\|^{1/2} \right) \big( d \log 6 + 2 \log \tau^* + \log(3/\delta) \big) \right).
    \end{equation}
    Assume that the learning rates $\eta_1 \geq \dots \geq \eta_{\tau^*}$ are chosen in such a way that
    \begin{equation}
        \label{eq:eta_condition_change_point}
        4 \ttR B \big(1 \vee 4 \gamma^2 \ttR \big) \big( d \log 6 + \log \tau^* + \log(3/\delta) \big) \leq \frac1{\eta_{\tau^*}}.
    \end{equation}
    Then, with probability at least $(1 - 8\delta)$, the test statistic $\widehat{S}_t = \widehat{L}_{1:t}^{EW} - \widehat{L}_{1:t}^{FS}$ satisfies the inequality
    \begin{align}
        \label{eq:test_stat_change_point}
        \widehat S_t
        &\notag
        \geq \sum\limits_{j = \tau^* + 1}^t \E_{\bX_j} \left\| \bnabla \log \sfp_{\widehat\btheta{}_j^{EW}}(\bX_j) - \bnabla \log \sfp_{\btheta_{\tau^* + 1}^\star}(\bX_j) \right\|^2
        \\&\quad\notag
        + \gamma \sum\limits_{j = \tau^* + 1}^t \left\| \widehat\btheta{}_j^{EW} - \btheta_{\tau^* + 1}^\star \right\|^2 - \frac{1 + \log T}{\eta_t} - \frac{\lambda \, \ttR^2}{2 \eta_t} \left( \frac1{\big\| \overline A_1 \big\|}
        + \frac1{\big\| \overline A_{\tau^* + 1} \big\|} \right)
        \\&\quad
        - \frac{d}{2\eta_t} \log\left(1 + \frac{\gamma B \, \ttR \, \eta_t \tau^*}{\lambda} \right)
        - \frac{d}{2\eta_t} \log\left(1 + \frac{\gamma B \, \ttR \, \eta_t (T - \tau^*)}{\lambda} \right)
        \\&\quad\notag
        - 4B \left(3e \ttR^2 \log(8B) + 4B\right) \log(T/\delta)
        - 12B \, \ttR \ttZ (1 + \ttR) \log(T/\delta)
    \end{align}
    simultaneously for all $t \in \{\tau^* + 1, \dots, T\}$, where
    \[
        \ttZ = 1 \vee \log \frac{e(1 + \ttR)\sqrt{2B}}{\sqrt{3e \ttR^2 \log(8B) + 4B}}.
    \]
\end{Th}

Let us elaborate on the expression in the right-hand side of \eqref{eq:test_stat_change_point}. If $\tau^*$ is sufficiently large, then the value of $\smash{\widehat \btheta{}_j^{EW}}$ will be close to $\btheta_1^\star$. Moreover, it will take significant time before the exponentially weighted average $\smash{\widehat \btheta{}_j^{EW}}$ forgets the pre-change history and moves towards  $\btheta_{\tau^* + 1}^\star$. Hence, the first two terms in the right-hand side of \eqref{eq:test_stat_change_point} will grow rapidly while the other ones are of just logarithmic order.

The proofs of Theorems \ref{th:main_rl} and \ref{th:main_dd} are moved to Appendix \ref{sec:th_proofs}. 
The section starts with general ideas and proof insights. Our approach is based on combination of regret bounds for the exponentially weighted average and fixed share forecasters with recent results on concentration of unbounded martingales.

\section{Numerical Experiments} \label{sec:numerical}

We conduct an empirical study to evaluate the performance of Algorithm~\ref{alg:score-based}. Computational details concerning the exponentially weighted average forecaster and the fixed share forecaster, as well as threshold tuning procedure and the experiments on artificially generated Gaussian sequences and two real-world time series are presented in Appendix~\ref{sec:numerical_appendix}. 
In this section we discuss the change point detection procedures chosen for comparison and demonstrate the results on a human activity detection data set. 

For Algorithm~\ref{alg:score-based} during all of the trials we use $\bPsi(\bx) = (x_1, \dots, x_d, x_1^2, \dots, x_d^2)^\top$. Concerning the parameter tuning, the regularization constants $\lambda > 0, \gamma > 0$, the shifting constant $\alpha \in [0, 1]$ and the learning rate \(\eta_t\) are chosen empirically. To evaluate the procedure as objectively as possible, we control the rate of false alarms, the number of undetected change points and track the average detection delay. 

We compare Algorithm~\ref{alg:score-based} with four non-parametric online change point detection procedures. The first one is based on the Follow-the-Leading-History (FLH) approach presented in \citet[Section 3]{hazan07b}. As discussed before, FLH aggregates a growing set of online prediction algorithms, seen as experts, and tracks the best one using a scheme adapted from \citet{herbster98}. In the experiments we consider Algorithm~\ref{alg:score-based}, where we calculate the FLH prediction instead of  \(\widehat \btheta{}_t^{FS}\) and denote this competitor as FLH. Hyperparameters \(\lambda, \gamma, \eta\) and \(\alpha\) throughout the experimental study were tuned independently for Algorithm~\ref{alg:score-based} and FLH. We would like to emphasize that \citeauthor{hazan07b} have developed an advanced implementation of FLH, which does not take into account the predictions of all experts, instead operating on subsets of diverse experts where members within each subset are similar. However, to ensure a fair comparison with our approach based on computational effort, we consider the basic method to be more appropriate.

This study also considers other several powerful methods. The Kullback–Leibler importance estimation procedure (KLIEP), proposed in \citet{sugiyama08}, focuses on direct importance estimation, that is, the ratio of the pre- and post-change density functions, without explicitly estimating neither of them. The method uses Gaussian kernels as basis functions for modeling the importance.  Another non-parametric approach, covered in \citet{li15}, 
uses a kernel-based estimate of the squared maximum mean discrepancy between the pre- and post-change distributions as the test statistic for online change point detection. Both KLIEP and the kernel change point detector with M-statistic require maintaining a sliding window to test each timestamp as a potential change point. Too large or too small window size may lead to a longer detection delay, as the performance of the algorithms strongly depends on the amount of data being processed. Both the kernel bandwidth and window size are adjustable hyperparameters. Additionally, we compare Algorithm~\ref{alg:score-based} to the fast change point detection procedure based on contrastive approach (FALCON), described in \citet{goldman23}. The main idea behind the method is to maximize the cross-entropy between the pre- and post-change distributions using a discriminator function that depends on the log ratio of the corresponding densities. As we mentioned earlier, FALCON uses ideas from online convex optimization to reduce computational complexity.
The code for the experiments is available on \href{https://github.com/lipperr/score_based_change_point_detection}{GitHub}\footnote{https://github.com/lipperr/score\_based\_change\_point\_detection}.

\subsection{Activity Detection}

This experiment aims to test the effectiveness of change point detection algorithms in recognizing changes in human activity using wearable sensors. The WISDM \citep{weiss19} data set consists of 3-dimensional data from a smartphone accelerometer sampled at 20 Hz. Throughout the time series, that was used for the trial, the subject's activity changed 16 times, leading to 16 potential change points in the data. 

To prepare the data set, we selected every 20th observation to reduce its length, while still maintaining more than 3,000 data points. Four stationary segments were isolated to set algorithm thresholds and define normalization factors. We divided the test of the data into training and testing sets with nine and seven change points respectively. Both training and test data were normalized per axis using the maxima from the stationary segments. The resulting observations are shown in Figure \ref{fig:wisdm_data}. The goal of this experiment was to detect these annotated change points using selected algorithms as quickly as possible.

\begin{table}[ht]
    \centering
    \begin{tabular}{lcccc}
         \textbf{Method} & $\z$ & \textbf{PARAMETER}& \textbf{FA} & \textbf{DD} \\
    \hline 
        Algorithm~\ref{alg:score-based} & $17.15$ & $\alpha=0.001$, $\lambda=0.1$  $\eta= 0.2$, $\gamma=0.3$ & $1$ & $\mathbf{15.43\pm10.39}$\\
    \hline 
        FLH & 517.78 & $\alpha=1 $, $\lambda=5$, $\eta=1.2$, $\gamma=0.001$ & 1 & $ 16.29\pm20.15 $\\
    \hline
        FALCON  &$18.5$ & $p = 1$, $\beta=0.1$ & $1$ & $\mathbf{18.43 \pm 7.94}$\\
    \hline
        KLIEP & $19.55$& $ws=50, b=0.9$& $1$ & $20.7 \pm 12.3$\\
    \hline
        M-statistic & $8.75$ & $ws=50, b=2.5$ & $2$ & $25.9 \pm 13.7$\\
    \hline
    \end{tabular}
    \caption{The thresholds $\z$, the values of hyperparameters, the average detection delays (DD), and the number of false alarms (FA) of Algorithm~\ref{alg:score-based}, FLH,  FALCON with Hermite polynomials, KLIEP and the kernel change point detector with M-statistic on the WISDM (human activity detection) data set. The two best results are boldfaced.}
    \label{tab:wisdm_params_fa_dd}
\end{table}

\begin{figure}[ht]
    \centering {{\includegraphics[width=0.8\textwidth]{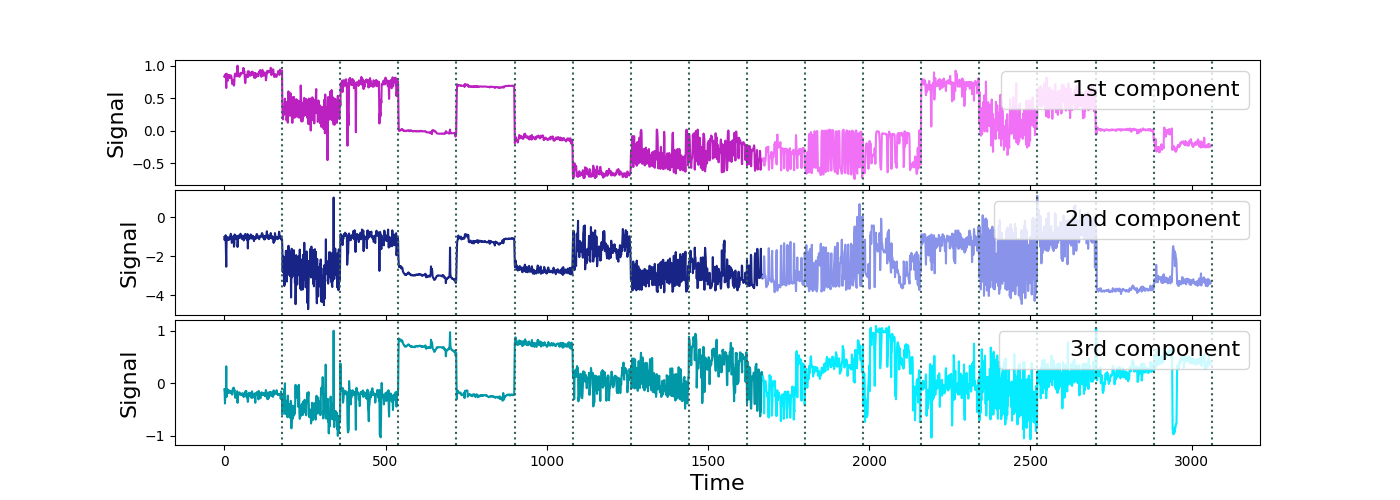} }}
    \caption{Activity detection data set (WISDM) after preprocessing. The dotted lines correspond to the annotated change points. Validation and test parts of the series are depicted in darker and lighter colors respectively.}
    \label{fig:wisdm_data}
\end{figure}

As mentioned before, we compared our approach to FLH, FALCON, KLIEP, and the kernel-based change point detector with M-statistic. Table~\ref{tab:wisdm_params_fa_dd} presents the thresholds, parameter values, the number of false alarms and the average detection delays for these algorithms. Based on the results, Algorithm~\ref{alg:score-based} outperformed the other methods in terms of shorter detection delays and fewer false positives.

\bibliographystyle{abbrvnat}
\bibliography{references}

\appendix

\newpage
\tableofcontents

\newpage

\section{Numerical Experiments}
\label{sec:numerical_appendix}

\subsection{Computational details}
In this section we introduce pseudocodes for the online prediction algorithms employed in Algorithm~\ref{alg:score-based} and briefly discuss the implementation details. Lemma~\ref{lem:ew} provides an explicit form of the exponentially weighted average. Thus, the computation of the EW forecast is straightforward. We outline it in Algorithm~\ref{alg:ew}.

\medskip
\begin{tcolorbox}[enhanced]
    \begin{Alg}[Exponentially weighted forecaster]
    \hfill
    \label{alg:ew}
    \begin{itemize}
        \item \textbf{Input:} a non-increasing sequence $\{\eta_t : 1 \leq t \leq T\}$ of positive numbers and regularization parameters $\lambda > 0, \gamma > 0$.
        \item \textbf{For} $t = 1, 2, \dots, T$ \textbf{do} the following.
        \begin{enumerate}
            \item Compute the prediction $\widehat \btheta{}_t^{EW} = \widehat \btheta_{1:t-1}(\eta_t)$ according to~\eqref{eq:ew_st}:
            \[
                \widehat \btheta{}_t^{EW}
                = \widehat \btheta_{1:t-1}(\eta_t)
                = \left( \sum\limits_{j = 1}^{t - 1} A_j + \frac{\lambda}{\eta_t} I_d \right)^{-1} \sum\limits_{j = 1}^{t - 1} \bb_j.
            \]\vspace{-0.5cm}
            \item Suffer the loss $\ell_t(\widehat \btheta{}_t^{EW})$.
        \end{enumerate}
        \item \textbf{Return.}
    \end{itemize}
    \end{Alg}
\end{tcolorbox}
\medskip

Regarding the fixed share algorithm, Lemmata~\ref{lem:v} and~\ref{lem:z} introduce closed-form expressions of the the recurrence relations for the auxiliary variables \(Z_{s:t}, V_t\) and prediction \(\widehat \btheta^{FS}_t\).
At each time step \(t\) we calculate \(\widehat \btheta_{s:t-1}(\eta_t)\) and \(Z_{s:t-1}(\eta_t)\) for all \(1 \le s \le t-1\). Recurrence relation for \(V_{t-1}(\eta_t)\) requires the values \(V_{s}(\eta_s)\) and \(Z_{k:s-1}(\eta_t)\) for all \(1 \le k \le s-1 < t-1\). In the case of a constant learning rate \(\eta_1= \ldots = \eta_t = \eta\), we store the previously computed variables for reuse in subsequent rounds. With a varying \(\eta_t\), they need to be recomputed on each round. This reduces the space complexity, but significantly increases the running time of the procedure. The complexity of the algorithm was specified at the end of Section~\ref{sec:application} for both scenarios. 
Ultimately, the learning rate is chosen to be constant in all of the experimental setups. We summarize the fixed share algorithm below. To the best of our knowledge, Algorithm~\ref{alg:fs} has not appeared in the literature.

\medskip
\begin{tcolorbox}[breakable, enhanced]
    \begin{Alg}[Fixed share forecaster]
    \hfill
    \label{alg:fs}
    \begin{itemize}
        \item \textbf{Input:} a non-increasing sequence $\{\eta_t : 1 \leq t \leq T\}$ of positive numbers, regularization parameters $\lambda > 0, \gamma > 0$, and a shifting parameter $\alpha \in [0, 1]$.
        \item Predict $\widehat \btheta{}_1^{FS} = \widetilde \btheta_1(\eta_1) = \bzero$ and suffer the loss $\ell_1(\bzero)$.
        \item Compute $\widehat\btheta_{1:1}(\eta_2)$ and $Z_{1:1}(\eta_2)$ according to~\eqref{eq:ew_st} and~\eqref{eq:z_st_explicit}, respectively.
        \item Predict $\widehat \btheta{}_2^{FS} = \widetilde \btheta_2(\eta_2) = (1 - \alpha) \widehat\btheta_{1:1}(\eta_2)$ and suffer the loss $\ell_1(\widehat \btheta{}_1^{FS})$.
        \item \textbf{For} $t = 3, 4, \dots, T$ \textbf{do} the following.
        \begin{enumerate}
            \item Compute $Z_{s:t - 1}(\eta_t)$ for all $s \in \{1, \dots, t - 1\}$ according to~\eqref{eq:z_st_explicit}.
            \item Find $V_{t - 1}(\eta_t)$ using the recurrence relation~\eqref{eq:v}.
            \item Compute $\widehat \btheta_{s:t - 1}(\eta_t)$ for all $s \in \{1, \dots, t - 1\}$ according to~\eqref{eq:ew_st}.
            \item Compute the prediction $\widehat \btheta{}_t^{FS} = \widetilde \btheta_{t}(\eta_t)$ according to~\eqref{eq:widetilde_theta}.
            \item Suffer the loss $\ell_t(\widehat \btheta{}_t^{FS})$.
        \end{enumerate}
        \item \textbf{Return.}
    \end{itemize}
    \end{Alg}
\end{tcolorbox}
\medskip

\subsection{Synthetic Data Sets}
\label{sec:synthetic}

We move to the empirical study of the algorithm presented in this paper.
The experiments conducted on Gaussian sequences were designed to assess the effectiveness of the procedure in identifying alterations in the mean and variance of the data distribution. To evaluate the performance of the method thoroughly, we implemented and tested several distinct experimental setups. They enabled us to validate the procedure and fine-tune its parameters, thereby enhancing our understanding of its capabilities and limitations in detecting changes in statistical properties.
Settings of the experiments are listed below.

\medskip\noindent
\textbf{Example 1: mean shift detection in a univariate sequence.}\quad
We generated a series of $T = 300$ i.i.d. observations of Gaussian random variable. First $\tau = 150$ samples were drawn from $\cN(0, \sigma^2)$ with $\sigma=0.2$, the rest were derived from the same distribution, shifted in the expectation by 2 standard deviations: $\cN(\mu, \sigma^2), \; \mu=2\sigma$.

\medskip\noindent
\textbf{Example 2: variance shift detection in a univariate sequence.} The sequence of i.i.d. observations of length $T = 300$ came from Gaussian distribution. The first $150$ samples were from $\cN(0, \sigma_1)$ with $\sigma_1=0.1$ and the last $150$ from the scaled distribution $\cN(0, \sigma_2), \; \sigma_2=0.3$.

\medskip\noindent
\textbf{Example 3: mean shift detection in a multivariate sequence.} With data dimension set to $d=3$, we took $T = 300$ i.i.d. observations of Gaussian vector with non-correlated components. Half of the series was taken from $\cN(\bzero, \Sigma)$ with $\Sigma=\text{diag}(\sigma_1, \sigma_2, \sigma_3)$, $\sigma_1 = 0.1$, $\sigma_2 =0.2$, $\sigma_3 = 0.3$, the rest were drawn from the shifted in mean distribution $\cN(\bmu, \Sigma)$, $\bmu = 3 (\sigma_1, \sigma_2, \sigma_3)^\top$.

\medskip\noindent
\textbf{Example 4: variance shift detection in a multivariate sequence.} As before, the data consists of $300$ i.i.d. observations of 3d Gaussian vector with non-correlated components. First $\tau=150$ values were drawn from $\cN(\bzero, \Sigma)$ with $\Sigma=\text{diag}(\sigma_1, \sigma_2, \sigma_3)$, $\sigma_1 = 0.1$, $\sigma_2 =0.2$, $\sigma_3 = 0.3$, the rest were taken
from the same distribution only with standard deviation scaled by 3 component-wise: $\cN(\bzero, 9 \Sigma)$.

For our artificial experiments we used a threshold tuning strategy as described in \citep{puchkin23}. We generated $T=150$ i.i.d. samples from the corresponding initial distribution J times and ran the algorithm on each sequence to
obtain the threshold value $\z = \underset{1\le j\le J}{\text{max}}  \underset{1\le t \le T}{\text{max}} \widehat{S}_t^{(j)}$. This approach ensures that the running length of Algorithm \ref{alg:score-based} does not exceed \(T=150\). If we run the procedure in a stationary regime and obtain the corresponding test statistics \(\{\widehat{S}_t\}^T_{t=1}\), the choice of \(\z\) guarantees
\[
    \p \left( \underset{1\le t \le T}{\text{max}} \widehat{S}_t > \z \right) = \frac1{J + 1},
\]
provided that there are no change points. 

For each example we present thresholds and the values of hyperparameters for the competing algorithms in Table~\ref{tab:gaussian_params}.
The hyperparameters were tuned to keep the false alarm rate at zero while still detecting all of the change points as quickly as possible.

\begin{table}[ht]
\centering
\small{\begin{tabular}{lllllllll}
    \textbf{Method}& \multicolumn{2}{c}{\textbf{ Example 1}} & \multicolumn{2}{c}{\textbf{Example 2} } & \multicolumn{2}{c}{\textbf{ Example 3}} & \multicolumn{2}{c}{\textbf{ Example 4}} \\
    \hline
    &  parameters & $\z$ & parameters & $\z$ & parameters & $\z$ &parameters  & $\z$  \\
    \hline
    \makecell[tl]{Algorithm~\ref{alg:score-based}} & 
    \makecell[tl]{$\alpha=10^{-3}$ \\ $\lambda=0.05$ \\ $\eta=0.1$ \\ $\gamma=0.1$} & 0.74 & 
    \makecell[tl]{$\alpha=10^{-4}$ \\ $\lambda=0.1$ \\ $\eta=0.2$ \\ $\gamma=0.1$} &  0.0 & 
    \makecell[tl]{$\alpha=10^{-6}$ \\ $\lambda=1$ \\ $\eta=0.2$ \\ $\gamma=0.1$} & 0.0  & 
    \makecell[tl]{$\alpha=10^{-4}$ \\ $\lambda=0.1$ \\ $\eta=0.2$ \\ $\gamma=0.1$} & 0.0  \\
    \hline
    \makecell[tl]{FLH}  &
    \makecell[tl]{$\alpha=0.05$ \\ $\lambda=0.5$ \\ $\eta=0.2$\\ $\gamma=10^{-6}$} & 0.0 &
    \makecell[tl]{$\alpha=0.2$ \\ $\lambda=1.5$  \\ $\eta=0.2$\\ $\gamma=10^{-6}$} & 0.05 & 
    \makecell[tl]{$\alpha=0.2$ \\ $\lambda=1.5$  \\ $\eta=0.2$\\ $\gamma=10^{-6}$} & 0.0  & 
    \makecell[tl]{$\alpha=0.5$ \\ $\lambda=1.5$  \\ $\eta=0.2$\\ $\gamma=10^{-6}$} & 0.0 \\
    \hline
    \makecell[tl]{FALCON } & 
    \makecell[tl]{$p=2$ \\ $\beta=0.3$} & 2.26 &
    \makecell[tl]{$p=3$ \\ $\beta=0.3$} & 2.87 &
    \makecell[tl]{$p=1$ \\ $\beta=0.9$} & 5.06 &
    \makecell[tl]{$p=1$ \\ $\beta=0.1$} & 6.46 \\

    \hline
    \makecell[tl]{KLIEP} & 
    \makecell[tl]{$ws=20$ \\ $b=0.3$} &  6.39 & 
    \makecell[tl]{$ws=20$ \\$b=0.1$} & 11.31 & 
    \makecell[tl]{$ws=20$ \\$b=0.4$} & 13.45 & 
    \makecell[tl]{$ws=20$ \\$b=0.2$} & 28.23 \\
    \hline
    \makecell[tl]{M-statistic}& 
    \makecell[tl]{$ws=20$ \\$b=2$} & 2.04 &
    \makecell[tl]{$ws=20$ \\$b=0.4$} & 11.27 &
    \makecell[tl]{$ws=20$ \\$b=2$} & 5.23 & 
    \makecell[tl]{$ws=20$ \\$b=2$} & 5.23\\

\end{tabular}}
\caption{ The thresholds $\z$ and the values of hyperparameters of the competing algorithms in the experiments on synthetic data sets. }
\label{tab:gaussian_params}
\end{table}

The average detection delays are shown in Table~\ref{tab:gaussian_fa_dd}. 
The two best results for each example are boldfaced. In Examples 1, the variation in the metrics was not significant, but Algorithm~\ref{alg:score-based} still managed to detect the shifts in the distribution after seeing roughly 3 observations in its best configuration. Changes in the signal variance were more difficult for FALCON and the M-statistic method to detect. 

\begin{table}[ht]
\centering
\begin{tabular}{lcccccccc}
    \textbf{Method}& \multicolumn{2}{c}{\textbf{ Example 1}} & \multicolumn{2}{c}{\textbf{Example 2} } & \multicolumn{2}{c}{\textbf{ Example 3}} & \multicolumn{2}{c}{\textbf{ Example 4}} \\
    \hline
    &  FA & DD & FA & DD & FA & DD & FA  & DD \\
    \hline
    \makecell[tl]{Algorithm~\ref{alg:score-based}}& 
    0 & $\mathbf{3.5\pm2.1}$ & 0 &  $\mathbf{3.4\pm1.9}$ & 0 & $\mathbf{1.5\pm0.7}$  & 0 & $\mathbf{1.6\pm0.5}$\\
    \hline
    \makecell[tl]{FLH}       
    & 0 & $\mathbf{4.7 \pm 1.0}$ & 0 & $\mathbf{4.0\pm2.6}$ & 0 & $\mathbf{1.2\pm0.4}$ & 0 & $\mathbf{2.0\pm0.9}$\\
    \hline
    \makecell[tl]{FALCON}     
    & 0 & $5.5 \pm 1.6$ & 0 & $7.7\pm4.6$ & 0 & $5.2\pm 1.2$  & 0 & $10.5\pm9.8$ \\
    \hline     
    KLIEP 
    & 0 & $8.0 \pm 3.2$ & 0 & $5.1 \pm 3.1$ & 0 & $4.5\pm1.7$  & 0 & $6.0\pm2.0$\\
    \hline
    M-statistic 
    & 0 & $8.5\pm3.3$ & 0 & $12.9\pm6.6$ & 0 & $5.5 \pm 0.9$ & 0 & $12.3 \pm 4.1$\\
\end{tabular}
\caption{False alarms (FA) and average detection delays (DD) of Algorithm~\ref{alg:score-based}, FLH,  FALCON, KLIEP, and the kernel change point detector with M-statistic on synthetic data sets. The two best results for each example are boldfaced.}
\label{tab:gaussian_fa_dd}
\end{table}

For multivariate data, Algorithm \ref{alg:score-based} and FLH  performed almost equally well, while other nonparametric approaches showed significantly longer detection delays. 
The behavior of the test statistics for the considered algorithms in Example 4 is depicted in Figure~\ref{fig:gaussian-plots}. 
Overall, Algorithm~\ref{alg:score-based} outshines its competitors in each example. 

\begin{figure}[ht]
    \centering
    \includegraphics[width=0.8\textwidth]{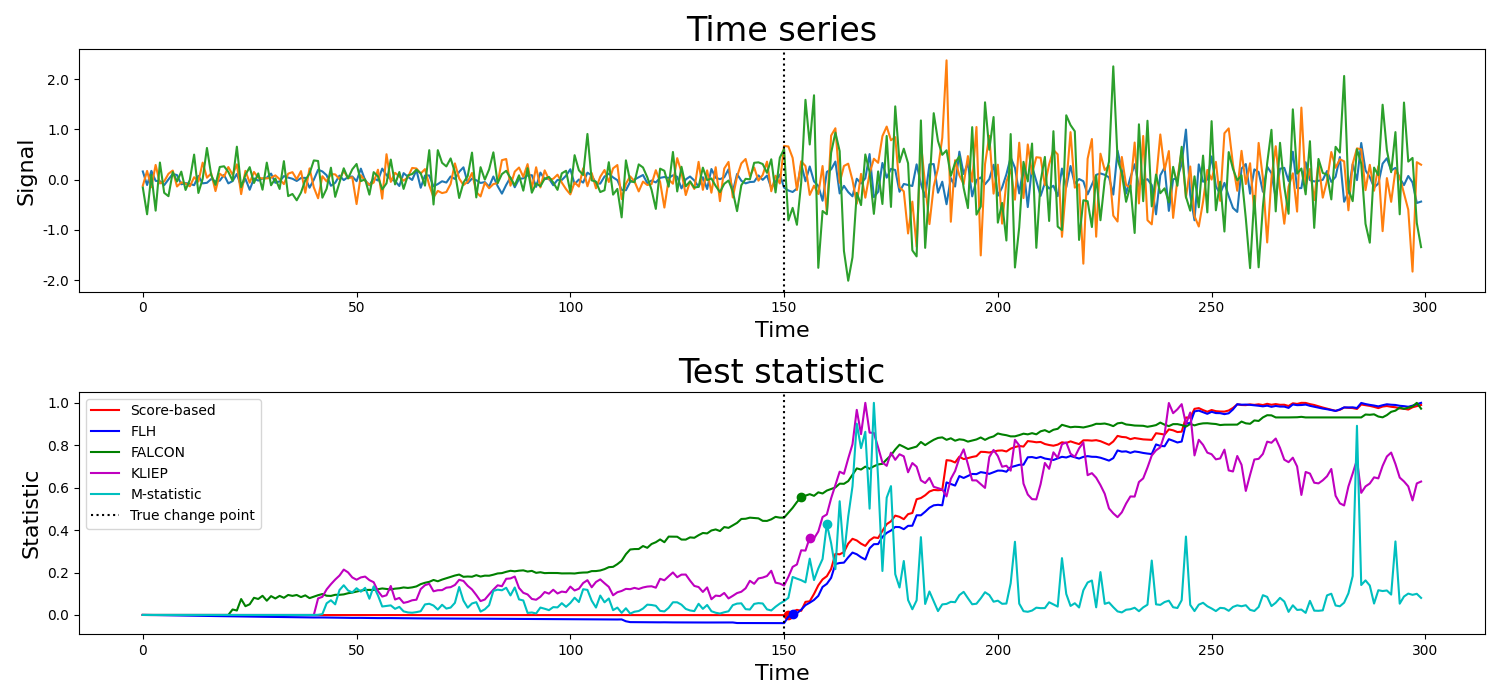}
    \caption{Example 4: variance shift detection in a multivariate Gaussian sequence. The top plot shows the generated time series with a single change point $\tau^*$ highlighted with a black dotted line. The coordinates of a multivariate sequence are defined by different colors. The bottom plot demonstrates the behavior of the test statistic for the competing methods: Algorithm~\ref{alg:score-based} (red), FLH (blue), FALCON (green), KLIEP (magenta), the kernel change point with M-statistic (cyan). Detected change points $\tau$ are marked with bold dots. Values of the test statistics were scaled for a better visualization.}
    \label{fig:gaussian-plots}
\end{figure}

\subsection{Speech Detection}

We proceed with the real-world data experiments using an audio recording data set \href{http://research.nii.ac.jp/src/en/CENSREC-1-C.html}{CENSREC-1-C}\footnote{http://research.nii.ac.jp/src/en/CENSREC-1-C.html},
which was taken from the Speech Resource Consortium (SRC), provided by the National Institute of Informatics (NII). This data set includes a clean speech recording, denoted as MAH clean, as well as versions of the same recording with added noise at different signal-to-noise ratio levels (MAH N1 SNR 20 and MAH N1 SNR 15).

To convert the recording into a more suitable format, we normalized the data. After that, we identified 10 sections that contained a single transition from silence or noise to speech. Then, we extracted every 10th observation. The first four sections were used to fine-tune the hyperparameters and thresholds, while the remaining six were used as the testing set. We determined the accurate change point values for the clean MAN data set and incorporated them into noisy iterations of recordings. Data set CENSREC-1-C SNR 20 is depicted in Figure~\ref{fig:CENSREC-plots}.

\begin{figure}[ht]
    \centering
    \includegraphics[width=0.8\textwidth]{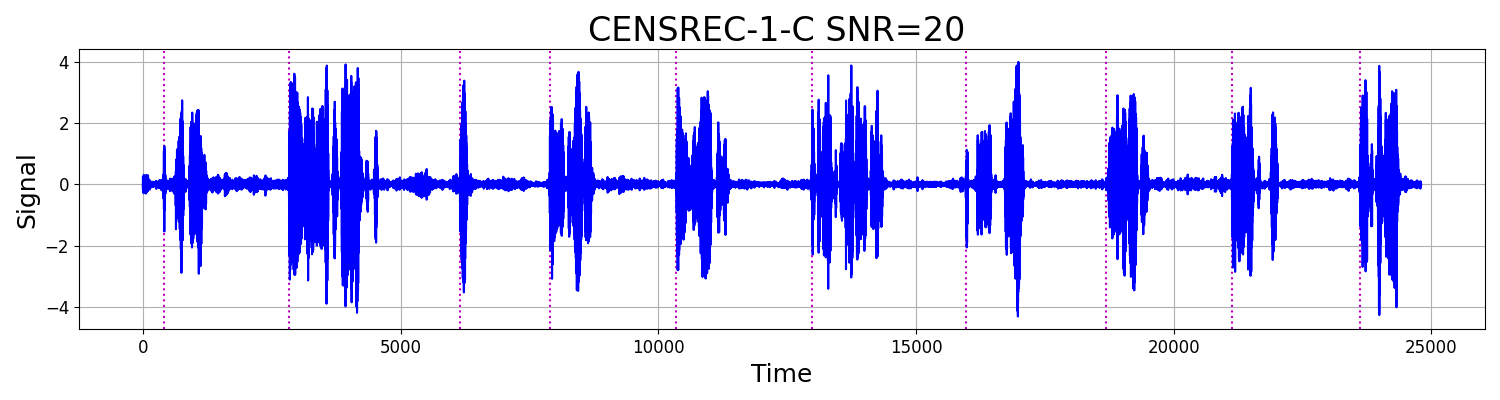}
    \caption{Data points of the CENSREC-1-C audio recording with SNR = 20. Dotted pink lines mark the annotated change points.}
    \label{fig:CENSREC-plots}
\end{figure}

\begin{table}[ht]
    \centering
    \small{\begin{tabular}{lllllll}
        \textbf{Method} &  \multicolumn{2}{c}{\textbf{CLEAN RECORD}} &  \multicolumn{2}{c}{\textbf{SNR 20}} &  \multicolumn{2}{c}{\textbf{SNR 15}} \\
        \hline
         &  parameters & $\z$ & parameters & $\z$ & parameters & $\z$ \\
        \hline
        Algorithm~\ref{alg:score-based} &  
        \makecell[tl]{$\alpha=10^{-5}$, $\lambda=0.3$\\ $\eta=0.1$, $\gamma=0.01$} & 0.0  &
        \makecell[tl]{$\alpha=10^{-5}$, $\lambda=0.7$\\ $\eta=0.5$, $\gamma=0.01$} & 418.7  &
        \makecell[tl]{$\alpha=10^{-5}$, $\lambda=1$\\ $\eta=0.1$, $\gamma=0.01$} & 0.0  \\ 
        \hline
        FLH  &  
        \makecell[tl]{$\alpha=0.1$, $\lambda=0.3$\\ $\eta=0.1$, $\gamma=0.001$ } & 0.0  &
        \makecell[tl]{$\alpha=0.6$, $\lambda=1$\\ $\eta=0.1$, $\gamma=0.001$} & 214.1  &
        \makecell[tl]{$\alpha=0.1$, $\lambda=0.3$\\ $\eta=0.2$, $\gamma=0.1$ } &  10.0 \\
        \hline
        FALCON  &
        \makecell[tl]{$p = 3$, $\beta = 0.5$} & 0.66 &
        \makecell[tl]{$p = 3$, $\beta = 0.5$} & 4.39 &
        \makecell[tl]{$p = 3$, $\beta = 1$} & 0.71 \\
        \hline 
        KLIEP & 
        \makecell[tl]{$ws = 50$, $b=0.9$} & $10^{-5}$ & 
        \makecell[tl]{$ws = 50$, $b=0.01$} & $80.59$ & 
        \makecell[tl]{$ws = 50$, $b=0.6$} & $0.21$ \\
        \hline
        M-statistic & 
        \makecell[tl]{$ws = 50$,  $b=0.25$}  & $0.11$ &
        \makecell[tl]{$ws = 50$,  $b=5$} & $0.0015$ &
        \makecell[tl]{$ws = 50$,  $b=1.2$} & $0.11$ \\
        \hline
    \end{tabular}}
    \caption{The thresholds $\z$ and the values of hyperparameters of the competing algorithms in the experiments on the CENSREC-1-C (speech recognition) data set. }
    \label{tab:CENSREC_params}
\end{table}

The values of the hyperparameters were tuned using four extracted segments, and the results are provided in Table~\ref{tab:CENSREC_params}. The algorithms were tested on the other six segments of the recording. Our procedure significantly outperforms other methods. On the clean record, M-statistic method shows a slightly faster detection, while reporting 2 false alarms. According to the results, only with KLIEP we observe false alarms occurring just a few observations prior to the actual change point in the clean data. 

\begin{table}[ht]
    \centering
    \begin{tabular}{lcccccc}
        \textbf{Method} &  \multicolumn{2}{c}{\textbf{CLEAN RECORD}} &  \multicolumn{2}{c}{\textbf{SNR 20}} &  \multicolumn{2}{c}{\textbf{SNR 15}} \\
    \hline
         &  FA & DD & FA & DD & FA & DD \\
    \hline
        Algorithm~\ref{alg:score-based}& 
         0 & $\mathbf{2.8 \pm 3.7}$ &
         0 & $\mathbf{12.0\pm18.9}$ &
         0 & $\mathbf{10.5\pm12.9}$ \\
    \hline
        FLH & 
         0 & $9.3 \pm 19.1$ &
         0 & $\mathbf{15.8 \pm 17.2}$ & 
         0 & $15.3 \pm 19.4$\\
    \hline
        FALCON  & 
         0 & $7.3 \pm 13.3$ & 
         0 & $20.8 \pm 19.0$ &
         0 & $\mathbf{12.3 \pm 11.4}$ \\
    \hline
        KLIEP &
        $1$ & $\mathbf{6.2 \pm 4.9}$  & 
        $0$ & $20.3 \pm 14.2$ & 
        $0$ & $17.0 \pm 21.3$\\
    \hline
        M-statistic & 
        $2$ & $2.5 \pm 1.5$  & 
        $0$ & $18.3 \pm 20.1$ & 
        $0$ & $16.7 \pm 16.8$\\
    \hline
    \end{tabular}
    \caption{The average detection delays (DD) and the number of false alarms (FA) of Algorithm~\ref{alg:score-based}, FLH, FALCON with Hermite polynomials, KLIEP and the kernel change point detector with M-statistic on the CENSREC-1-C recording with different noise levels. The two best results for each example are boldfaced.}
    \label{tab:new_CENSREC_dd}
\end{table}

\subsection{Room Occupancy Detection}

We applied Algorithm~\ref{alg:score-based} to detect changes in the room occupancy based on temperature, humidity, light, and the $CO_2$ level. A four-dimensional time series was obtained from the UCI repository \citep{Bache13}. The data was preprocessed in three sequential steps. First, we selected every 16th observation to reduce the length of the time series. Next, to convert a non-stationary series into a stationary one, we took the first-order difference and normalized the series by the last of the adjacent values. Additionally, we log-transformed the sequence and scaled the coordinates using the stationary parts of the time series to ensure a similar range of the data. Ultimately, the time series comprised approximately 500 data points, nine of which were labeled as change points. The change point annotation description can be found in \citep[Section Annotation Collection]{burg22}. The time series is displayed in Figure~\ref{fig:occ_data}.

\begin{figure}[ht]
    \centering {{\includegraphics[width=\textwidth]{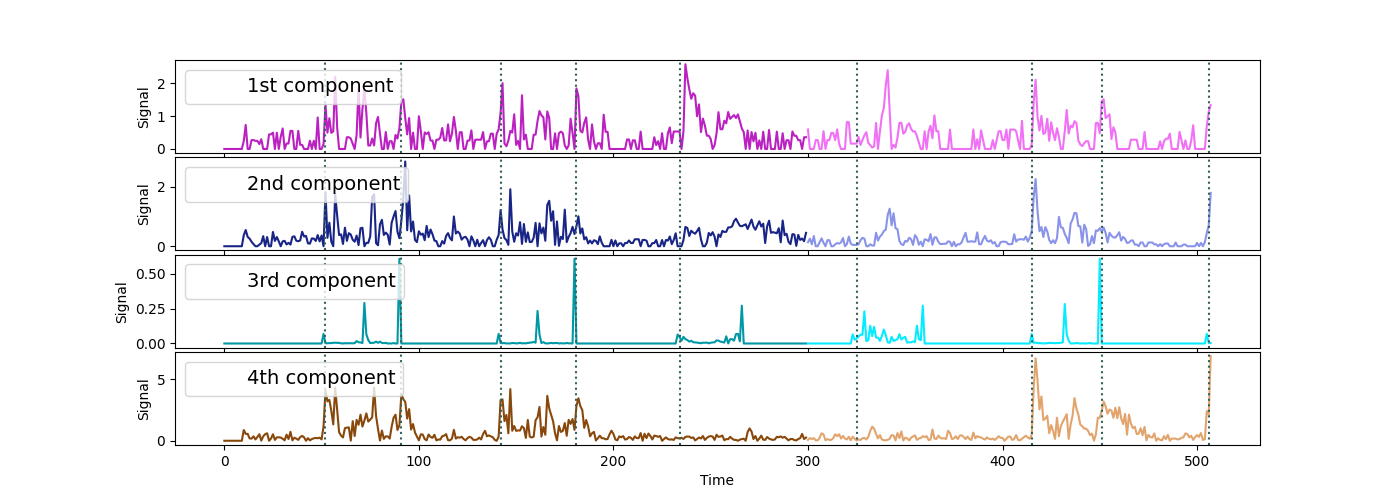} }}
    \caption{Room occupancy data set after preprocessing. The dotted lines correspond to the annotated change points. Validation and test parts of the series are depicted in darker and lighter colors respectively.}
    \label{fig:occ_data}
\end{figure}

The algorithm application pipeline differed from the preceding experiments.
Since there were several manual annotations of the change points, it is possible that some of them might not have been detected accurately, leading to incorrect change point identification in the methods' performance. To avoid false alarms caused by inaccurate annotations, we have introduced the parameter $\text{min\_diff} = 10$. This parameter controls the distance between a detected change point and an annotated one. A false alarm is not triggered if a change point is detected before the annotated one but the distance is less than min\_diff. In this case, we assume that the annotated change point was detected with zero delay.

\begin{table}[ht]
    \centering
    \begin{tabular}{lcccc}
         \textbf{Method} & $\z$ & \textbf{PARAMETER} & \textbf{FA} & \textbf{DD} \\
    \hline
        Algorithm~\ref{alg:score-based} & $2.36$ & $\alpha=10^{-5}$, $\lambda=1$, $\eta=0.5$, $\gamma=0.4$& $1$ & $\mathbf{4.0 \pm 6.9}$\\
    \hline
        FLH & $3.0$ & $\alpha=0.1$, $\lambda=0.3$, $\eta=0.8$, $\gamma=0.5$ & $2$ & $8.25 \pm 8.25$\\
    \hline
        FALCON  &$2.05$ & $p = 1$, $\beta=8$ & $1$ & $\mathbf{3.5 \pm 6.1}$\\
    \hline 
        KLIEP & 2.0 & $ws=20, b=2$ & 0 & $ 10.33\pm7.4 $\\
    \hline 
        M-statistic & 2.0 & $ws=20, b=2.5$ & 1 & $9.67\pm6.94 $\\
    \hline
    \end{tabular}
    \caption{ The thresholds $\z$, the values of hyperparameters, the average detection delays (DD), and the number of false alarms (FA)
    of Algorithm~\ref{alg:score-based}, FLH, FALCON with Hermite polynomials, KLIEP and the kernel
    change point detector with M-statistic on the room occupancy data set. The two best results are boldfaced.}
    \label{tab:occupancy_params_fa_dd}
\end{table}

We applied the algorithms to the validation part of the data set in order to fine-tune the hyperparameters. During training, we took into account several factors. First, we attempted to keep the number of undetected change points to zero. Second, we selected hyperparameters that produced the minimum average detection delay while allowing for a few false positives.
The thresholds, parameter values, false alarm rates, and average detection delays for the experiment using the Occupancy data set are presented in Table \ref{tab:occupancy_params_fa_dd}. Based on the results, FALCON detects change points slightly faster than Algorithm \ref{alg:score-based}, although both of these methods report one false positive throughout the testing phase. FLH takes almost twice as many observations on average, and has the highest number of false alarms of all the competitors. KLIEP and the kernel change point detector with M-statistic perform almost perfectly on the validation set, but in the test phase both of them miss the last change point.

\section{Proofs of the results from Section \ref{sec:online_learning}}

\subsection{Proof of Lemma~\ref{lem:ew}}
\label{sec:lem_ew_proof}

It holds that
\begin{align*}
    e^{-\eta L_{s:t}(\btheta)} \pi(\btheta)
    &
    \; \propto \; \exp\left\{-\frac{\eta}2 \sum\limits_{j = s}^t (\btheta^\top A_j \btheta - \bb_j^\top \btheta) - \frac{\lambda \|\btheta\|^2}2 \right\}
    \\&
    = \exp\left\{-\frac{\eta}2 \; \btheta^\top \left( \sum\limits_{j = s}^t A_j + \frac{\lambda}\eta I_d \right) \btheta + \btheta^\top \left(\sum\limits_{j = s}^t \bb_j\right) \right\}.
\end{align*}
Using the representation
\begin{align*}
    &
    \btheta^\top \left( \sum\limits_{j = s}^t A_j + \frac{\lambda}\eta I_d \right) \btheta - 2 \btheta^\top \left(\sum\limits_{j = s}^t \bb_j\right)
    \\&
    = \left\| \left( \sum\limits_{j = s}^t A_j + \frac{\lambda}\eta I_d \right)^{1/2} \btheta - \left( \sum\limits_{j = s}^t A_j + \frac{\lambda}\eta I_d \right)^{-1/2} \sum\limits_{j = s}^t \bb_j \right\|^2 \\&\quad
    - \left\| \left( \sum\limits_{j = s}^t A_j + \frac{\lambda}\eta I_d \right)^{-1/2} \sum\limits_{j = s}^t \bb_j \right\|^2,
\end{align*}
we obtain that
\begin{align*}
    e^{-\eta L_{s:t}(\btheta)} \pi(\btheta)
    &
    \propto \exp\left\{-\frac{\eta}2 \; \left\| \left( \sum\limits_{j = s}^t A_j + \frac{\lambda}\eta I_d \right)^{1/2} \btheta - \left( \sum\limits_{j = s}^t A_j + \frac{\lambda}\eta I_d \right)^{-1/2} \sum\limits_{j = s}^t \bb_j \right\|^2 \right\}
    \\&
    \propto \exp\left\{-\frac{\eta}2 \; \left\| \left( \sum\limits_{j = s}^t A_j + \frac{\lambda}\eta I_d \right)^{1/2} \left[ \btheta - \left( \sum\limits_{j = s}^t A_j + \frac{\lambda}\eta I_d \right)^{-1} \sum\limits_{j = s}^t \bb_j \right] \right\|^2 \right\}.
\end{align*}
Hence, the posterior measure
\[
    \frac1{Z_{s:t}} e^{-\eta L_{s:t}(\btheta)} \pi(\btheta)
\]
coincides with
\[
    \cN\left( \left( \sum\limits_{j = s}^t A_j + \frac{\lambda}\eta I_d \right)^{-1} \sum\limits_{j = s}^t \bb_j, \frac1{\eta} \left( \sum\limits_{j = s}^t A_j + \frac{\lambda}\eta I_d \right)^{-1} \right),
\]
and the prediction of the exponentially weighted forecaster is equal to
\[
    \widehat \btheta_t = \left( \sum\limits_{j = s}^t A_j + \frac{\lambda}\eta I_d \right)^{-1} \sum\limits_{j = s}^t \bb_j.
\]
\myendproof

\subsection{Proof of Lemma~\ref{lem:v}}
\label{sec:lem_v_proof}

\noindent
\textbf{Step 1: definition of $V_t(\eta)$ and preliminaries.}
\quad
Before we move to the proof of \eqref{eq:widetilde_theta} and \eqref{eq:v}, let us elaborate on the sequence $V_1(\eta), \dots, V_T(\eta)$.
For any $t \in \N$, we define $V_t(\eta)$ through a weight function $w_{t - 1}(\btheta, \eta)$:
\begin{equation}
    \label{eq:vt_recursion}
    V_t(\eta) = \int\limits_{\R^d} v_t(\btheta, \eta) \, \dd \btheta,
    \quad \text{where} \quad
    v_t(\btheta, \eta) = w_{t - 1}(\btheta, \eta) e^{-\eta \ell_t(\btheta)}.
\end{equation}
In its turn, the weight function is given by
\begin{equation}
    \label{eq:wt}
    w_0(\btheta, \eta) = \pi(\btheta),
    \quad w_t(\btheta, \eta) = (1 - \alpha) v_t(\btheta, \eta) + \alpha \pi(\btheta) V_t(\eta),
    \quad t \in \{1, \dots, T\}.
\end{equation}
The function $w_{t - 1}(\btheta, \eta)$ plays a key role in the proof of Lemma \ref{lem:v}, because of the following identity.

\begin{Prop}
    \label{prop:representation}
    For any $t \in \{1, \dots, T\}$ and any $\btheta \in \R^d$, it holds that
    \begin{align*}
        &
        \int \exp\big\{-\eta \cL_{t - 1}(\btheta_1, \dots, \btheta_{t - 1}, \btheta, \btheta_{t + 1}, \dots, \btheta_T) \big\}
        \\&\quad
        \cdot \rho(\btheta_1, \dots, \btheta_{t - 1}, \btheta, \btheta_{t + 1}, \dots, \btheta_T) \, \dd \btheta_1 \dots \dd \btheta_{t - 1} \dd \btheta_{t + 1} \dots \dd \btheta_T
        = w_{t - 1}(\btheta, \eta).
    \end{align*}
\end{Prop}
The proof of Proposition~\ref{prop:representation} is deferred to Section~\ref{sec:prop_representation_proof}. It is similar to \citet[Theorem 5.1]{cesa-bianchi06a} and allows us to simplify the formula~\eqref{eq:exponential_weights_compound} significantly. Indeed, the proposition yields that
\begin{equation}
    \label{eq:widetilde_theta_simplified}
    \widetilde \btheta_t(\eta) = \frac{1}{W_{t - 1}(\eta)} \int\limits_{\R^d} \btheta_t w_{t - 1}(\btheta_t, \eta) \dd \btheta_t,
    \quad \text{where} \quad
    W_{t - 1}(\eta) = \int\limits_{\R^d} w_{t - 1}(\btheta_t, \eta) \dd \btheta_t.
\end{equation}
The last expression can be simplified even further. Since $w_t(\btheta, \eta) = (1 - \alpha) v_t(\btheta, \eta) + \alpha \pi(\btheta) V_t(\eta)$, we have
\begin{align}
    \label{eq:w_v_equality}
    W_{t - 1}(\eta)
    = \int\limits_{\R^d} w_{t - 1}(\btheta, \eta) \,\dd \btheta
    &\notag
    = (1 - \alpha) \int\limits_{\R^d} v_{t - 1}(\btheta, \eta) \, \dd \btheta + \alpha V_{t - 1}(\eta) \int\limits_{\R^d} \pi(\btheta) \, \dd \btheta
    \\&
    = (1 - \alpha) V_{t - 1}(\eta) + \alpha V_{t - 1}(\eta)
    = V_{t - 1}(\eta)
\end{align}
and
\begin{align*}
    \int\limits_{\R^d} \btheta w_{t - 1}(\btheta, \eta) \,\dd \btheta
    &
    = (1 - \alpha) \int\limits_{\R^d} \btheta v_{t - 1}(\btheta, \eta) \, \dd \btheta + \alpha V_{t - 1}(\eta) \int\limits_{\R^d} \btheta \pi(\btheta) \, \dd \btheta
    \\&
    = (1 - \alpha) \int\limits_{\R^d} \btheta v_{t - 1}(\btheta, \eta) \, \dd \btheta.
\end{align*}
Hence, it holds that
\begin{equation}
    \label{eq:fs_estimate}
    \widetilde \btheta_t(\eta) = \frac{1 - \alpha}{V_{t - 1}(\eta)} \int\limits_{\R^d} \btheta v_{t - 1}(\btheta, \eta) \, \dd \btheta
    \quad \text{for any $t \geq 2$.}
\end{equation}

\medskip

\noindent
\textbf{Step 2: recurrence relation for $V_t(\eta)$.}\quad
The proof relies on~\eqref{eq:wt} and~\eqref{eq:vt_recursion}, yielding the recurrence relation
\begin{equation}
    \label{eq:vt_vt-1_recurrence}
    v_t(\btheta, \eta) = (1 - \alpha) v_{t - 1}(\btheta, \eta) + \alpha V_{t - 1}(\eta) \pi(\btheta).
\end{equation}
Applying~\eqref{eq:vt_vt-1_recurrence}, we obtain that
\begin{align*}
    V_t(\eta)
    = \int\limits_{\R^d} v_t(\btheta, \eta) \, \dd \btheta
    &
    = \int\limits_{\R^d} \big[ (1 - \alpha) v_{t - 1}(\btheta, \eta) + \alpha V_{t - 1}(\eta) \pi(\btheta) \big] e^{-\eta \ell_t(\btheta)} \, \dd \btheta
    \\&
    = \alpha V_{t - 1}(\eta) \int\limits_{\R^d} e^{-\eta \ell_t(\btheta)} \pi(\btheta) \, \dd \btheta
    + (1 - \alpha) \int\limits_{\R^d} v_{t - 1}(\btheta, \eta) e^{-\eta \ell_t(\btheta)} \pi(\btheta) \, \dd \btheta
    \\&
    = \alpha V_{t - 1}(\eta) Z_{t : t}(\eta) + (1 - \alpha) \int\limits_{\R^d} v_{t - 1}(\btheta, \eta) e^{-\eta \ell_t(\btheta)} \pi(\btheta) \, \dd \btheta,
\end{align*}
where the last line follows from the definition of $Z_{t:t}(\eta)$ (see eq.~\ref{eq:z_st}).
Similarly, it holds that
\begin{align*}
    &
    \int\limits_{\R^d} v_{t - 1}(\btheta, \eta) e^{-\eta \ell_t(\btheta)} \pi(\btheta) \, \dd \btheta
    \\&
    = \int\limits_{\R^d} \big[ (1 - \alpha) v_{t - 2}(\btheta, \eta) + \alpha V_{t - 2}(\eta) \pi(\btheta) \big] e^{-\eta \ell_t(\btheta) - \eta \ell_{t - 1}(\btheta)} \, \dd \btheta
    \\&
    = \alpha V_{t - 2}(\eta) Z_{t - 1 : t}(\eta) + (1 - \alpha) \int\limits_{\R^d} v_{t - 2}(\btheta, \eta) e^{-\eta \ell_t(\btheta) - \eta \ell_{t - 1}(\btheta)} \pi(\btheta) \, \dd \btheta,
\end{align*}
and, hence,
\begin{align*}
    V_t(\eta)
    &
    = \alpha V_{t - 1}(\eta) Z_{t : t}(\eta)
    + \alpha (1 - \alpha) V_{t - 2}(\eta) Z_{t - 1 : t}(\eta)
    \\&\quad
    + (1 - \alpha)^2 \int\limits_{\R^d} v_{t - 2}(\btheta, \eta) e^{-\eta \ell_t(\btheta) - \eta \ell_{t - 1}(\btheta)} \pi(\btheta) \, \dd \btheta.
\end{align*}
Repeating the same argument $(t - 2)$ times, we get that
\[
    V_t(\eta) = (1 - \alpha)^{t - 1} \, Z_{1 : t}(\eta) + \alpha \sum\limits_{s = 0}^{t - 2} (1 - \alpha)^s \, V_{t - 1 - s}(\eta) \, Z_{t - s : t}(\eta).
\]
Thus, it only remains to prove~\eqref{eq:widetilde_theta}.

\medskip

\noindent
\textbf{Step 2: recurrence relation for $\widetilde \btheta_t(\eta)$.}\quad
Let us introduce
\begin{equation}
    \label{eq:y}
    \bY_{s : t}(\eta) = \int\limits_{\R^d} \btheta e^{-\eta L_{s:t}(\btheta)} \pi(\btheta) \, \dd \btheta,
    \quad 1 \leq s \leq t \leq T.
\end{equation}
Note that, due to~\eqref{eq:z_st}, $\widehat \btheta_{s:t}(\eta) = \bY_{s : t}(\eta) / Z_{s:t}(\eta)$ for any $1 \leq s \leq t \leq T$.
The equalities $\widetilde \btheta_1(\eta) = \bzero$ and 
\[
    \widetilde \btheta_2(\eta)
    = (1 - \alpha) \bY_{1 : 1}(\eta) / V_1(\eta)
    = (1 - \alpha) \bY_{1 : 1}(\eta) / Z_{1:1}(\eta)
    = (1 - \alpha) \widehat \btheta_{1:1}(\eta)
\]
follow directly from~\eqref{eq:widetilde_theta_simplified}, \eqref{eq:fs_estimate}, and~\eqref{eq:v}. Hence, to finish the proof, it is enough to show that, for any $t \geq 3$,
\[
    \widetilde \btheta_t(\eta) = \frac{1 - \alpha}{V_{t - 1}(\eta)} \left( (1 - \alpha)^{t - 2} \, \bY_{1 : t - 1}(\eta) + \alpha \sum\limits_{s = 0}^{t - 3} (1 - \alpha)^s \, V_{t - 2 - s}(\eta) \, \bY_{t - 1 - s : t - 1}(\eta)\right).
\]
Using the recurrence relation~\eqref{eq:vt_vt-1_recurrence} and invoking the definition of $\bY_{s:t}$ (see eq.~\ref{eq:y}), we obtain that
\begin{align}
    \label{eq:y_t_recurrence}
    \int\limits_{\R^d} \btheta v_t(\btheta, \eta) \, \dd \btheta
    &\notag
    = \int\limits_{\R^d} \btheta \big[ (1 - \alpha) v_{t - 1}(\btheta, \eta) + \alpha V_{t - 1}(\eta) \pi(\btheta) \big] e^{-\eta \ell_t(\btheta)} \, \dd \btheta
    \\&
    = \alpha V_{t - 1}(\eta) \int\limits_{\R^d} \btheta e^{-\eta \ell_t (\btheta)} \pi(\btheta) \, \dd \btheta
    + (1 - \alpha) \int\limits_{\R^d} \btheta v_{t - 1}(\btheta, \eta) e^{-\eta \ell_t(\btheta)} \, \dd \btheta
    \\&\notag
    = \alpha V_{t - 1}(\eta) \, \bY_{t : t}(\eta) + (1 - \alpha) \int\limits_{\R^d} \btheta v_{t - 1}(\btheta, \eta) e^{-\eta \ell_t(\btheta)} \, \dd \btheta.
\end{align}
The latter term in the right-hand side can be examined in a similar manner:
\begin{align}
    \label{eq:y_t-1_recurrence}
    \int\limits_{\R^d} \btheta v_{t - 1}(\btheta, \eta) e^{-\eta \ell_t(\btheta)} \, \dd \btheta
    &\notag
    = \int\limits_{\R^d} \btheta \big[ (1 - \alpha) v_{t - 2}(\btheta, \eta) + \alpha V_{t - 2}(\eta) \pi(\btheta) \big] e^{-\eta \ell_t(\btheta) - \eta \ell_{t - 1}(\btheta)} \, \dd \btheta
    \\&\notag
    = \alpha V_{t - 2}(\eta) \int\limits_{\R^d} \btheta e^{-\eta \ell_t(\btheta) - \eta \ell_{t - 1}(\btheta)} \pi(\btheta) \, \dd \btheta
    \\&\quad
    + (1 - \alpha) \int\limits_{\R^d} \btheta v_{t - 2}(\btheta, \eta) e^{-\eta \ell_t(\btheta) - \eta \ell_{t - 1}(\btheta)} \, \dd \btheta
    \\&\notag
    = \alpha V_{t - 2}(\eta) \, \bY_{(t - 1) : t}(\eta) + (1 - \alpha) \int\limits_{\R^d} \btheta v_{t - 2}(\btheta, \eta) e^{-\eta \ell_t(\btheta) - \eta \ell_{t - 1}(\btheta)} \, \dd \btheta.
\end{align}
Summing up~\eqref{eq:y_t_recurrence} and~\eqref{eq:y_t-1_recurrence}, we get that
\begin{align*}
    \int\limits_{\R^d} \btheta v_t(\btheta) \, \dd \btheta &
    = \alpha V_{t - 1}(\eta) \, \bY_{t : t}(\eta) + \alpha (1 - \alpha) V_{t - 2}(\eta) \, \bY_{t - 1 : t}(\eta) \\&\quad + (1 - \alpha)^2 \int\limits_{\R^d} \btheta v_{t - 2}(\btheta, \eta) e^{-\eta \ell_t(\btheta) - \eta \ell_{t - 1}(\btheta)} \, \dd \btheta.
\end{align*}
Repeating the same argument and using the identity
\[
    \int\limits_{\R^d} \btheta v_1(\btheta, \eta) e^{-\eta \sum\limits_{j = 2}^t \ell_j(\btheta)} \, \dd \btheta
    = \int\limits_{\R^d} \btheta e^{-\eta \sum\limits_{j = 1}^t \ell_j(\btheta)} \pi(\btheta) \, \dd \btheta
    = \bY_{1 : t}(\eta),
\]
we finally obtain that
\[
    \int\limits_{\R^d} \btheta v_t(\btheta, \eta) \, \dd \btheta
    = (1 - \alpha)^{t - 1} \, \bY_{1 : t}(\eta) + \alpha \sum\limits_{s = 0}^{t - 2} (1 - \alpha)^s \, V_{t - 1 - s}(\eta) \, \bY_{t - s : t}(\eta)
    \quad \text{for all $t \geq 2$.}
\]
Hence, for any $t \geq 3$, it holds that
\begin{align*}
    \widetilde \btheta_t(\eta)
    &
    = \frac{1 - \alpha}{V_{t - 1}(\eta)} \int\limits_{\R^d} \btheta v_{t - 1}(\btheta, \eta) \, \dd \btheta
    \\&
    = \frac{1 - \alpha}{V_{t - 1}(\eta)} \left( (1 - \alpha)^{t - 2} \, \bY_{1 : t - 1}(\eta) + \alpha \sum\limits_{s = 0}^{t - 3} (1 - \alpha)^s \, V_{t - 2 - s}(\eta) \, \bY_{t - 1 - s : t - 1}(\eta)\right)
    \\&
    = \frac{1 - \alpha}{V_{t - 1}(\eta)}
    \Bigg( (1 - \alpha)^{t - 2} \, Z_{1 : t - 1}(\eta) \widehat\btheta_{1:t - 1}(\eta)
    \\&\qquad\qquad\qquad
    + \alpha \sum\limits_{s = 0}^{t - 3} (1 - \alpha)^s \, V_{t - 2 - s} \, Z_{t - 1 - s : t - 1}(\eta) \widehat\btheta_{t - 1 - s : t - 1}(\eta) \Bigg).
\end{align*}
\myendproof

\subsection{Proof of Lemma~\ref{lem:z}}
\label{sec:lem_z_proof}

Let us represent
\[
    \lambda \|\btheta\|^2 + \eta \sum\limits_{j = s}^t \left( \btheta^\top A_j \btheta - 2 \bb_j^\top \btheta \right)
\]
in the following form:
\begin{align*}
    &
    \lambda \|\btheta\|^2 + \eta \sum\limits_{j = s}^t \left( \btheta^\top A_j \btheta - 2 \bb_j^\top \btheta \right)
    \\&
    = \eta \, \btheta^\top \left( \sum\limits_{j = s}^t A_j + \frac{\lambda}{\eta} I_d \right) \btheta - 2 \eta \sum\limits_{j = s}^t \bb_j^\top \btheta
    \\&
    = \eta \left\| \left( \sum\limits_{j = s}^t A_j + \frac{\lambda}{\eta} I_d \right)^{1/2} \btheta - \left( \sum\limits_{j = s}^t A_j + \frac{\lambda}{\eta} I_d \right)^{-1/2} \sum\limits_{j = s}^t \bb_j \right\|^2
    \\&\quad
    - \frac{\eta}2 \left\| \left( \sum\limits_{j = s}^t A_j + \frac{\lambda}{\eta} I_d \right)^{-1/2} \sum\limits_{j = s}^t \bb_j \right\|^2.
\end{align*}
Then it holds that
\begin{align*}
    &
    \pi(\btheta) \exp\left\{-\eta \sum\limits_{j = s}^t \ell_j(\btheta) \right\}
    \\&
    = \left( \frac{\lambda}{2 \pi} \right)^{d/2} \exp\left\{ -\frac{\lambda \|\btheta\|^2}2 - \frac{\eta}2 \sum\limits_{j = s}^t \left( \btheta^\top A_j \btheta - 2 \bb_j^\top \btheta \right) \right\}
    \\&
    = \left( \frac{\lambda}{2 \pi} \right)^{d/2} \exp\left\{ \frac{\eta}2 \left\| \left( \sum\limits_{j = s}^t A_j + \frac{\lambda}{\eta} I_d \right)^{-1/2} \sum\limits_{j = s}^t \bb_j \right\|^2 \right\}
    \\&\quad
    \cdot \exp\left\{-\frac{\eta}2 \left\| \left( \sum\limits_{j = s}^t A_j + \frac{\lambda}{\eta} I_d \right)^{1/2} \btheta - \left( \sum\limits_{j = s}^t A_j + \frac{\lambda}{\eta} I_d \right)^{-1/2} \sum\limits_{j = s}^t \bb_j \right\|^2 \right\}.
\end{align*}
The integral of the expression in the right-hand side admits a closed-form representation. Indeed, it holds that
\begin{align*}
    Z_{s : t}(\eta)
    &
    = \left( \frac{\lambda}{2 \pi} \right)^{d/2} \exp\left\{ \frac{\eta}2 \left\| \left( \sum\limits_{j = s}^t A_j + \frac{\lambda}{\eta} I_d \right)^{-1/2} \sum\limits_{j = s}^t \bb_j \right\|^2 \right\}
    \\&\quad
    \cdot \int\limits_{\R^d} \exp\left\{-\frac{\eta}2 \left\| \left( \sum\limits_{j = s}^t A_j + \frac{\lambda}{\eta} I_d \right)^{1/2} \btheta - \left( \sum\limits_{j = s}^t A_j + \frac{\lambda}{\eta} I_d \right)^{-1/2} \sum\limits_{j = s}^t \bb_j \right\|^2 \right\} \, \dd \btheta
    \\&
    = \lambda^{d/2} \det\left( \eta \sum\limits_{j = s}^t A_j + \lambda I_d\right)^{-1/2} \exp\left\{ \frac{\eta}2 \left\| \left( \sum\limits_{j = s}^t A_j + \frac{\lambda}{\eta} I_d \right)^{-1/2} \sum\limits_{j = s}^t \bb_j \right\|^2 \right\}
    \\&
    = \left( \frac{\lambda}{\eta} \right)^{d/2} \det\left( \sum\limits_{j = s}^t A_j + \frac{\lambda}\eta I_d\right)^{-1/2} \exp\left\{ \frac{\eta}2 \left\| \left( \sum\limits_{j = s}^t A_j + \frac{\lambda}{\eta} I_d \right)^{-1/2} \sum\limits_{j = s}^t \bb_j \right\|^2 \right\}.
\end{align*}
\myendproof

\subsection{Proof of Proposition~\ref{prop:representation}}
\label{sec:prop_representation_proof}

We conduct the proof by the induction in $t$. For convenience, we split the proof into two steps.

\medskip

\noindent\textbf{Step 1: base case.}\quad
We start with the case $t = 1$. Let us fix an arbitrary $\btheta \in \R^d$ and show that
\[
    \int \rho(\btheta, \btheta_2, \dots, \btheta_T) \, \dd \btheta_2 \dots \dd \btheta_T 
    = \pi(\btheta)
    = w_0(\btheta, \eta).
\]
According to the definition of the prior distribution $\rho$ (see eq.~\ref{eq:fs_prior}), it holds that
\[
    \int \rho(\btheta, \btheta_2, \dots, \btheta_T) \, \dd \btheta_2 \dots \dd \btheta_T
    = \pi(\btheta) \int \left( \prod\limits_{t = 3}^T \f(\btheta_t \,\vert\, \btheta_{t - 1}) \right) \f(\btheta_2 \,\vert\, \btheta) \, \dd \btheta_2 \dots \dd \btheta_T.
\]
Note that
\[
    \int\limits_{\R^d} \f(\btheta_T \,\vert\, \btheta_{T - 1}) \, \dd \btheta_T
    = (1 - \alpha) \int\limits_{\R^d} \delta(\btheta_T - \btheta_{T - 1}) \, \dd \btheta_T
    + \alpha \int\limits_{\R^d} \pi(\btheta_T) \, \dd \btheta_T
    = 1 - \alpha + \alpha = 1.
\]
Then the integration with respect to $\btheta_T$ leads to
\begin{align*}
    \int \rho(\btheta, \btheta_2, \dots, \btheta_T) \, \dd \btheta_2 \dots \dd \btheta_T
    &
    = \pi(\btheta) \int \left( \prod\limits_{t = 3}^T \f(\btheta_t \,\vert\, \btheta_{t - 1}) \right) \f(\btheta_2 \,\vert\, \btheta) \, \dd \btheta_2 \dots \dd \btheta_T
    \\&
    = \pi(\btheta) \int \left( \prod\limits_{t = 3}^{T - 1} \f(\btheta_t \,\vert\, \btheta_{t - 1}) \right) \f(\btheta_2 \,\vert\, \btheta) \, \dd \btheta_2 \dots \dd \btheta_T.
\end{align*}
Repeating the same argument $(T - 3)$ more times, we obtain that
\begin{align*}
    \int \rho(\btheta, \btheta_2, \dots, \btheta_T) \, \dd \btheta_2 \dots \dd \btheta_T 
    &
    = \pi(\btheta) \int\limits_{\R^d} \f(\btheta_2 \,\vert\, \btheta) \dd \btheta_2
    \\&
    = \pi(\btheta) \left( (1 - \alpha) \int\limits_{\R^d} \delta(\btheta_2 - \btheta) \, \dd \btheta_2
    + \alpha \int\limits_{\R^d} \pi(\btheta_2) \, \dd \btheta_2 \right)
    \\&
    = \pi(\btheta) \left( 1 - \alpha + \alpha \right)
    = w_0(\btheta, \eta).
\end{align*}

\medskip

\noindent\textbf{Step 2: induction step.}\quad
Assume that, for any $\btheta \in \R^d$, the integral
\[
    \int e^{-\eta \cL_{t - 1}(\btheta_1, \dots, \btheta_{t - 1}, \btheta, \btheta_{t + 1}, \dots, \btheta_T)} \rho(\btheta_1, \dots, \btheta_{t - 1}, \btheta, \btheta_{t + 1}, \dots, \btheta_T) \, \dd \btheta_1 \dots \dd \btheta_{t - 1} \dd \btheta_{t + 1} \dots \dd \btheta_T 
\]
is equal to $w_{t - 1}(\btheta, \eta)$.
We are going to show that
\[
    \int e^{-\eta \cL_{t}(\btheta_1, \dots, \btheta_t, \btheta, \btheta_{t + 2}, \dots, \btheta_T)} \rho(\btheta_1, \dots, \btheta_t, \btheta, \btheta_{t + 2}, \dots, \btheta_T) \, \dd \btheta_1 \dots \dd \btheta_t \dd \btheta_{t + 2} \dots \dd \btheta_T
    = w_t(\btheta, \eta).
\]
Let us introduce the marginal probability density of $\btheta_1, \dots, \btheta_t$ and the conditional density of $\btheta_{t + 2}, \dots, \btheta_T$ given $\btheta$ as follows:
\[
    \rho(\btheta_1, \dots, \btheta_t) = \pi(\btheta_1) \prod\limits_{s = 2}^t \f(\btheta_s \,\vert\, \btheta_{s - 1})
\]
and
\[
    \rho(\btheta_{t + 2}, \dots, \btheta_T \,\vert\, \btheta) = \f(\btheta_{t + 2} \,\vert\, \btheta) \prod\limits_{s = t + 3}^T \f(\btheta_s \,\vert\, \btheta_{s - 1}),
\]
where, as before, $\f(\btheta_s \,\vert\, \btheta_{s - 1}) = (1 - \alpha) \delta(\btheta_s - \btheta_{s - 1}) + \alpha \pi(\btheta_s)$. It is straightforward to check that $\rho(\btheta_1, \dots, \btheta_t)$ and $\rho(\btheta_{t + 2}, \dots, \btheta_T \,\vert\, \btheta)$ are probability densities, that is,
\begin{equation}
    \label{eq:density_integrals}
    \int \rho(\btheta_1, \dots, \btheta_t) \, \dd \btheta_1 \dots \dd \btheta_t = 1
    \quad \text{and} \quad
    \int \rho(\btheta_{t + 2}, \dots, \btheta_T \,\vert\, \btheta) \, \dd \btheta_{t + 2} \dots \dd \btheta_T = 1.
\end{equation}
On the other hand, $\rho(\btheta_1, \dots, \btheta_t, \btheta, \btheta_{t + 2}, \dots, \btheta_T)$ can be represented as a product of three terms:
\begin{align*}
    \rho(\btheta_1, \dots, \btheta_t, \btheta, \btheta_{t + 2}, \dots, \btheta_T)
    &
    = \rho(\btheta_{t + 2}, \dots, \btheta_T \,\vert\, \btheta) \, \f(\btheta \,\vert\, \btheta_t) \, \rho(\btheta_1, \dots, \btheta_t)
    \\&
    = \rho(\btheta_{t + 2}, \dots, \btheta_T \,\vert\, \btheta) \, \rho(\btheta_1, \dots, \btheta_t) \big( (1 - \alpha) \delta(\btheta - \btheta_t) + \alpha \pi(\btheta) \big). 
\end{align*}
Taking this expression into account and using the equality
\[
    \cL_t(\btheta_1, \dots, \btheta_t, \btheta, \btheta_{t + 2}, \dots, \btheta_T)
    = \sum\limits_{s = 1}^t \ell_s(\btheta_s),
\]
we obtain that
\begin{align*}
    &
    \int e^{-\eta \cL_{t}(\btheta_1, \dots, \btheta_t, \btheta, \btheta_{t + 2}, \dots, \btheta_T)} \rho(\btheta_1, \dots, \btheta_t, \btheta, \btheta_{t + 2}, \dots, \btheta_T) \, \dd \btheta_1 \dots \dd \btheta_t \dd \btheta_{t + 2} \dots \dd \btheta_T
    \\&
    = \int \rho(\btheta_{t + 2}, \dots, \btheta_T \,\vert\, \btheta) \, \dd \btheta_{t + 2} \dots \dd \btheta_T
    \\&\quad
    \cdot \int \big( (1 - \alpha) \delta(\btheta - \btheta_t) + \alpha \pi(\btheta) \big) e^{-\eta \ell_t(\btheta_t)} \, \rho(\btheta_1, \dots, \btheta_t) e^{-\eta \sum\limits_{s = 1}^{t - 1} \ell_s(\btheta_s)} \, \dd \btheta_1 \dots \dd \btheta_t.
\end{align*}
In view of~\eqref{eq:density_integrals}, the first term in the right-hand side is equal to $1$. Thus, it holds that
\begin{align}
    \label{eq:repeated_integral}
    &\notag
    \int e^{-\eta \cL_{t}(\btheta_1, \dots, \btheta_t, \btheta, \btheta_{t + 2}, \dots, \btheta_T)} \rho(\btheta_1, \dots, \btheta_t, \btheta, \btheta_{t + 2}, \dots, \btheta_T) \, \dd \btheta_1 \dots \dd \btheta_t \dd \btheta_{t + 2} \dots \dd \btheta_T
    \\&
    = \int \big( (1 - \alpha) \delta(\btheta - \btheta_t) + \alpha \pi(\btheta) \big) e^{-\eta \ell_t(\btheta_t)} \, \rho(\btheta_1, \dots, \btheta_t) e^{-\eta \sum\limits_{s = 1}^{t - 1} \ell_s(\btheta_s)} \, \dd \btheta_1 \dots \dd \btheta_t
    \\&\notag
    = \int\limits_{\R^d} \big( (1 - \alpha) \delta(\btheta - \btheta_t) + \alpha \pi(\btheta) \big) e^{-\eta \ell_t(\btheta_t)} \left( \int \rho(\btheta_1, \dots, \btheta_t) e^{-\eta \sum\limits_{s = 1}^{t - 1} \ell_s(\btheta_s)} \, \dd \btheta_1 \dots \dd \btheta_{t - 1} \right) \, \dd \btheta_t.
\end{align}
The induction hypothesis yields that
\begin{align*}
    &
    \int \rho(\btheta_1, \dots, \btheta_t) e^{-\eta \sum\limits_{s = 1}^{t - 1} \ell_s(\btheta_s)} \, \dd \btheta_1 \dots \dd \btheta_{t - 1}
    \\&
    = \int \rho(\btheta_1, \dots, \btheta_t) e^{-\eta \sum\limits_{s = 1}^{t - 1} \ell_s(\btheta_s)} \cdot \rho(\btheta_{t + 1}, \dots, \btheta_T \,\vert\, \btheta_t) \, \dd \btheta_1 \dots \dd \btheta_{t - 1} \dd \btheta_{t + 1} \dots \dd \btheta_T
    \\&
    = \int e^{-\eta \cL_{t - 1} (\btheta_1, \dots, \btheta_T)} \rho(\btheta_1, \dots, \btheta_T) \, \dd \btheta_1 \dots \dd \btheta_{t - 1} \dd \btheta_{t + 1} \dots \dd \btheta_T
    \\&
    = w_{t - 1}(\btheta_t, \eta).
\end{align*}
Then the integral in the left-hand side of~\eqref{eq:repeated_integral} simplifies to
\begin{align*}
    &
    \int e^{-\eta \cL_{t}(\btheta_1, \dots, \btheta_t, \btheta, \btheta_{t + 2}, \dots, \btheta_T)} \rho(\btheta_1, \dots, \btheta_t, \btheta, \btheta_{t + 2}, \dots, \btheta_T) \, \dd \btheta_1 \dots \dd \btheta_t \dd \btheta_{t + 2} \dots \dd \btheta_T
    \\&
    = \int\limits_{\R^d} \big( (1 - \alpha) \delta(\btheta - \btheta_t) + \alpha \pi(\btheta) \big) w_{t - 1}(\btheta_t, \eta) e^{-\eta \ell_t(\btheta_t)} \, \dd \btheta_t
\end{align*}
Let us recall that $w_{t - 1}(\btheta, \eta) e^{-\eta \ell_t(\btheta)} = v_t(\btheta, \eta)$ by the definition of $v_t(\btheta, \eta)$ (see eq.~\ref{eq:vt_recursion}).
Hence, due to the definition of $w_t(\btheta, \eta)$ (see eq.~\ref{eq:wt}), it holds that
\begin{align*}
    &
    \int e^{-\eta \cL_{t}(\btheta_1, \dots, \btheta_t, \btheta, \btheta_{t + 2}, \dots, \btheta_T)} \rho(\btheta_1, \dots, \btheta_t, \btheta, \btheta_{t + 2}, \dots, \btheta_T) \, \dd \btheta_1 \dots \dd \btheta_t \dd \btheta_{t + 2} \dots \dd \btheta_T
    \\&
    = \int\limits_{\R^d} \big( (1 - \alpha) \delta(\btheta - \btheta_t) + \alpha \pi(\btheta) \big) v_t(\btheta_t, \eta) \, \dd \btheta_t
    \\&
    = (1 - \alpha) v_t(\btheta, \eta) + \alpha \pi(\btheta) \int\limits_{\R^d} v_t(\btheta_t, \eta) \, \dd \btheta_t
    \\&
    = (1 - \alpha) v_t(\btheta, \eta) + \alpha \pi(\btheta) V_t(\eta)
    = w_t(\btheta, \eta).
\end{align*}
The proof is finished.

\myendproof

\section{Proofs of the Results from Section \ref{sec:application}}
\label{sec:th_proofs}

\subsection{Proof insights and preliminaries}

The proofs of Theorems \ref{th:main_rl} and \ref{th:main_dd} are quite similar to each other and rely on upper and lower bounds on the cumulative losses $\widehat L_{1:t}^{EW}$ and $\widehat L_{1:t}^{FS}$. The upper bounds on $\widehat L_{1:t}^{EW}$ and $\widehat L_{1:t}^{FS}$ follow from Theorems \ref{th:ew_regret} and \ref{th:fs_regret} studying regret of the exponentially weighted average and fixed share forecasters. These theorems rely on the notion of mixability, which helps us to deal with quadratic losses on $\R^d$. We would like to remind a reader that $\ell_t(\btheta)$ is not exp-concave on $\R^d$. For this reason, we cannot rely on the arguments used by \cite{hazan07, mourtada17} in the exp-concave scenario.

The lower bounds on $\widehat L_{1:t}^{EW}$ and $\widehat L_{1:t}^{FS}$ are more intricate. Their proof utilizes the stochastic structure of the outcomes $(A_t, \bb_t)$, $1 \leq t \leq T$ and a recent result of \cite{kroshnin25} on concentration of unbounded martingales. In particular, we establish the following inequality holding with high probability and suitable for analysis of both exponential weighting and fixed share. 

\begin{Th}
    \label{th:lower_bound}
    Grant Assumption \ref{as:subexp}. Let  \( \{ \widecheck{\btheta}_t : 1 \leq t \leq T\}\) be a non-anticipating sequence adapted to a natural filtration \(\{\cF_t : 0 \leq t \leq T\}\)\footnote{Here $\cF_0$ is a trivial sigma-algebra.} induced by the sequence of observations \(\bX_1, \dots, \bX_T\). For any $t \in \{1, \dots, T\}$ define $\overline{A}_t = \E A_t$, $\overline \bb_t = \E \bb_t$, $\overline{\ell}_t(\btheta) = \E_{\bX_t} \ell_t(\btheta)$, and 
    \begin{equation}
        \label{eq:theta_star}
        \btheta_t^\star
        = \argmin_{\btheta \in \R^d} \overline{\ell}_t(\btheta)
        = \argmin_{\btheta \in \R^d} \E_{\bX_t} \ell_t(\btheta)
        = \big( \overline{A}_t \big)^{-1} \overline{\bb}_t.
    \end{equation}
    Take an arbitrary $R > 0$ satisfying the inequalities
    \[
        \left\| \overline{A}_t^{1/2} \btheta_t^\star \right\| \le R \quad \text{for all $1 \leq t \leq T$}  
    \]
    and define an event
    \[
        \cE = \left\{ \left\|\overline{A}_t^{1/2}\widecheck{\btheta}_t\right\| \le R \quad \text{for all $1 \leq t \leq T$} \right\}.
    \]
    Then, for any $s \in \{1, \dots, T\}$ and any \(\delta \in (0,1)\), with probability at least \(\p(\cE) - \delta\) it holds that
    \begin{align*}
        \sum\limits_{t=1}^s \big(\ell_t(\widecheck{\btheta}_t) - \ell_t(\btheta_t^\star)\big)
        &
        \geq \frac12 \sum\limits_{t = 1}^s \big( \overline{\ell}_t(\widecheck{\btheta}_t) - \overline{\ell}_t(\btheta_t^\star) \big) - 4B \left(3eR^2\log(8B) + 4B\right) \log(1/\delta)
        \\&\quad
        - 12BRz (1 + R) \log(1/\delta),
    \end{align*}
    where
    \[
        z = 1 \vee \log \frac{e(1 + R)\sqrt{2B}}{\sqrt{3eR^2\log(8B) + 4B}}.
    \]
\end{Th}

\subsection{Proof of Theorem \ref{th:main_rl}}

Let us note that, if $t \leq \tau^*$, then $\bX_1, \dots, \bX_t$ are i.i.d. random elements (as well as the pairs $(A_1, \bb_1), \dots, (A_t, \bb_t)$). This yields that
\[
    \btheta_1^\star = \ldots = \btheta_t^\star = \argmin\limits_{\btheta \in \R^d} \overline{\ell}_1(\btheta) = \big(\overline{A}_1 \big)^{-1} \overline{\bb}_1.
\]
Let us represent the test statistic $\widehat S_t = \widehat L_{1:t}^{EW} - \widehat L_{1:t}^{FS}$ in the form
\[
    \widehat L_{1:t}^{EW} - \widehat L_{1:t}^{FS}
    = \left( \widehat L_{1:t}^{EW} - L_{1:t}(\btheta_1^\star) \right) - \left( \widehat L_{1:t}^{FS} - L_{1:t}(\btheta_1^\star) \right)
\]
and consider the terms in the right-hand side one-by-one.

\medskip

\noindent
\textbf{Step 1: upper bound on $\widehat L_{1:t}^{EW} - L_{1:t}(\btheta_1^\star)$.}
\quad
The upper bound on $\widehat L_{1:t}^{EW} - L_{1:t}(\btheta_1^\star)$ easily follows from Theorem \ref{th:ew_regret} studying regret of the exponential weighting forecaster. Let us note that, due to Lemma \ref{lem:ew_prediction_high-probability_bound} and the union bound, with probability at least $(1 - 2\delta)$, we have
\[
    \left\|A_t^{1/2} \widehat{\btheta}_{1:t-1}(\eta_t)\right\|
    \le 8B^{3/2} \|\overline{A}_1\| \big( d\log6 + \log(3\tau^*/\delta) \big)^{3/2}
    \leq 2 \gamma \ttR B^{1/2} \big( d\log6 + \log(3\tau^*/\delta) \big)^{1/2}
\]
simultaneously for all $1 \leq t \leq \tau^*$. Similarly, thanks to Lemma \ref{lem:scaled_bt_bound} and the union bound, it holds that
\[
    \left\| A_t^{-1/2} \bb_t \right\|
    \le 2B\sqrt{\frac{\|\overline{A}_t\|}{\gamma}} \big( d\log6 + \log(2\tau^*/\delta) \big)
    \leq \ttR^{1/2} B^{1/2} \big( d\log6 + \log(3\tau^*/\delta) \big)^{1/2}
\]
with probability at least $(1 - \delta)$ simultaneously for all $t \in \{1, \dots, \tau^*\}$. This and \eqref{eq:eta_condition_stationary} yield that there is an event $\mathcal E_1$ of probability at least $(1 - 3\delta)$, such that
\[
    \left\|A_t^{1/2} \widehat{\btheta}_{1:t-1}(\eta_t)\right\|^2 \vee \left\| A_t^{-1/2} \bb_t \right\|^2
    \leq \ttR B \left(1 \vee 4\gamma^2 \ttR \right) \big( d\log6 + \log(3\tau^*/\delta) \big)
    \leq \frac1{4 \eta_{t}}
\]
for all $t \in \{1, \dots, \tau^*\}$. In other words, the conditions of Theorem \ref{th:ew_regret} are satisfied on $\cE_1$. Hence, on this event, we have
\begin{equation}
    \label{eq:ew_loss_upper_bound}
    \widehat L_{1:t}^{EW} - L_{1:t}(\btheta_1^\star) 
    \leq R_{1:t}^{EW} 
    \leq \frac{\lambda \|\btheta_{1:t}^\circ\|^2}{2 \eta_t} + \frac1{2\eta_t} \log \det\left( I_d + \frac{\eta_t}{\lambda} \sum\limits_{j = 1}^t A_j \right),
\end{equation}
where
\[
    \btheta_{1:t}^\circ
    = \argmin\limits_{\btheta \in \R^d} L_{1:t}(\btheta)
    = \left( \sum\limits_{j = 1}^t A_j \right)^{-1} \sum\limits_{j = 1}^t \bb_j.
\]
The right-hand side of the inequality \eqref{eq:ew_loss_upper_bound} is nondeterministic but it admits a high-probability upper bound, following directly from Lemma \ref{lem:a_sum_norm} and \ref{lem:b_partial_sum_norm}. Due to Lemma \ref{lem:b_partial_sum_norm} and the union bound, with probability at least $(1 - \delta)$, it holds that
\[
    \left\|\btheta_{1:t}^\circ\right\|
    \leq \frac{1}{t \gamma} \left\| \sum\limits_{j = 1}^t \bb_j \right\|
    \leq \frac{4B \big\| \overline A_1 \big\|^{1/2}}{\gamma} \left( \frac{\log 2}{2} + \frac{d \log 6}t + \frac{\log(2 \tau^* / \delta)}t \right)
    \leq \frac{\ttR}{\big\| \overline A_1 \big\|^{1/2}}
\]
simultaneously for all $1 \leq t \leq \tau^*$.
On the other hand, Lemma \ref{lem:a_sum_norm} yields that
\[
    \left\| \sum\limits_{j = 1}^{\tau^*} \big(\overline{A}_1\big)^{-1/2} A_j \big(\overline{A}_1\big)^{-1/2} \right\|
    \le 4B \big( \tau^* \log 2 + d \log6 + \log(1/\delta) \big)
    \leq \frac{\gamma \ttR \tau^*}{\big\| \overline A_1 \big\|}
\]
with probability at least $(1 - \delta)$. Then, on this event, it holds that
\begin{align*}
    \log \det\left( I_d + \frac{\eta_t}{\lambda} \sum\limits_{j = 1}^t A_j \right)
    &
    \leq \log \det\left( I_d + \frac{\eta_t}{\lambda} \sum\limits_{j = 1}^{\tau^*} A_j \right)
    \\&
    \leq d \log\left(1 + \frac{\eta_t \|\overline{A}_1\|}{\lambda} \left\| \sum\limits_{j = 1}^{\tau^*} \big(\overline{A}_1\big)^{-1/2} A_j \big(\overline{A}_1\big)^{-1/2} \right\| \right)
    \\&
    \leq d \log\left(1 + \frac{\gamma B \, \ttR \, \eta_t \tau^*}{\lambda} \right)
    \quad \text{for all $1 \leq t \leq \tau^*$.}
\end{align*}
Thus, with probability at least $(\p(\cE_1) - 2\delta) \geq (1 - 5 \delta)$, it holds that
\begin{align*}
    \widehat L_{1:t}^{EW} - L_{1:t}(\btheta_1^\star) 
    &
    \leq \frac{\lambda \, \ttR^2}{2 \big\| \overline A_1 \big\| \eta_t}
    + \frac{d}{2\eta_t} \log\left(1 + \frac{\gamma B \, \ttR \, \eta_t \tau^*}{\lambda} \right)
\end{align*}
simultaneously for all $t \in \{1, \dots, \tau^*\}$.

\medskip

\noindent
\textbf{Step 2: lower bound on $\widehat L_{1:t}^{FS} - L_{1:t}(\btheta_1^\star)$.}
\quad
Thanks to
Lemma \ref{lem:b_partial_sum_norm} and the union bound, simultaneously for all \(1 \le s < t \le \tau^*\) and \(\delta \in (0, 1)\) with probability at least \( (1 - \delta) \) it holds that 
\begin{align*}
    \left\|\overline A_t^{1/2} \widehat{\btheta}_{s:t-1}(\eta_t)\right\|
    &
    \le \big\| \overline A_t \big\|^{1/2} \left\| \left( \sum\limits_{j = s}^{t - 1} A_j + \frac{\lambda}\eta I_d \right)^{-1} \sum\limits_{j = s}^{t - 1} \bb_j \right\|
    \le \frac{\big\| \overline A_t \big\|^{1/2}}{(t - s) \gamma} \left\| \sum\limits_{j = s}^{t - 1} \bb_j \right\|
    \\&
    \le \frac{4B \big\| \overline A_t \big\|}{(t - s) \gamma} \left(d\log6 + (t - s)\frac{\log2}{2} + 2\log \tau^* + \log\frac{2}{\delta}\right)
    \leq \ttR.
\end{align*}
Since $\widehat \btheta{}_t^{FS}$ is a convex combination of $\widehat \btheta_{1 : t - 1}(\eta_t), \dots$, $\widehat \btheta_{t - 1 : t - 1}(\eta_t)$, and $\bzero$ (see \eqref{eq:widetilde_theta} and \eqref{eq:v}), this implies that 
\[
    \left\|\overline A_t^{1/2} \widehat \btheta{}_t^{FS} \right\| \leq \ttR
    \quad \text{for all $1 \leq t \leq \tau^*$}
\]
with probability at least $(1 - \delta)$.
In view of \eqref{eq:r_stationary}, we also have
\[
    \left\|\overline A_t^{1/2} \btheta_t^\star \right\|
    = \left\| \big( \overline A_t \big)^{-1/2} \overline \bb_t \right\|
    \leq \ttR
    \quad \text{for any $1 \leq t \leq \tau^*$.}
\]
Hence, we can apply Theorem \ref{th:lower_bound}, claiming that 
\begin{align*}
    \widehat L_{1:t}^{FS} - L_{1:t}(\btheta_1^\star)
    \geq - 4B \left(3e \ttR^2 \log(8B) + 4B\right) \log(\tau^*/\delta)
    - 12B \, \ttR \ttZ (1 + \ttR) \log(\tau^*/\delta)
\end{align*}
with probability at least $(1 - 2\delta)$ simultaneously for all $1 \leq t \leq \tau^*$.
Hence, applying the union bound, we obtain that
\begin{align*}
    \widehat S_t
    = \widehat L_{1:t}^{EW} - \widehat L_{1:t}^{FS}
    &
    \leq \frac{\lambda \, \ttR^2}{2 \big\| \overline A_1 \big\| \eta_t}
    + \frac{d}{2\eta_t} \log\left(1 + \frac{\gamma B \, \ttR \, \eta_t \tau^*}{\lambda} \right)
    \\&\quad
    + 4B \left(3e \ttR^2 \log(8B) + 4B\right) \log(\tau^*/\delta)
    + 12B \, \ttR \ttZ (1 + \ttR) \log(\tau^*/\delta)
\end{align*}
for all $t \in \{1, \dots, \tau^*\}$ on an even of probability at least $(1 - 7\delta)$.

\myendproof

\subsection{Proof of Theorem \ref{th:main_dd}}

Similarly to the proof of Theorem \ref{th:main_rl}, we represent the test statistic $\widehat S_t$ in the following form:
\[
    \widehat S_t
    = \left( \widehat L_{1:t}^{EW} - L_{1:\tau^*}(\btheta_1^\star) - L_{\tau^* + 1:t}(\btheta_{\tau^* + 1}^\star) \right) - \left( \widehat L_{1:t}^{FS} - L_{1:\tau^*}(\btheta_1^\star) - L_{\tau^* + 1:t}(\btheta_{\tau^* + 1}^\star) \right).
\]
The difference is that we have to bound the former term in the right-hand side from below and the latter one from above.

\medskip

\noindent
\textbf{Step 1: upper bound on $\widehat L_{1:t}^{FS} - L_{1:\tau^*}(\btheta_1^\star) - L_{\tau^* + 1:t}(\btheta_{\tau^* + 1}^\star)$.}
\quad
Due to Lemma \ref{lem:ew_prediction_high-probability_bound} and the union bound, with probability at least $(1 - 2\delta)$, we have
\begin{align*}
    \left\|A_t^{1/2} \widehat{\btheta}_{s:t-1}(\eta_t)\right\|
    &
    \leq 8B^{3/2} \|\overline{A}_t\|^{1/2} \sqrt{\|\overline{A}_t\| \vee \|\overline A_1\|} \big( d\log6 + 2 \log T + \log(3/\delta) \big)^{3/2}
    \\&
    = 2\gamma \ttR B^{1/2} \big( d\log6 + 2 \log T + \log(3/\delta) \big)^{1/2}
\end{align*}
simultaneously for all $1 \leq s < t \leq T$. On the other hand, according to Lemma \ref{lem:scaled_bt_bound} and the union bound, it holds that
\[
    \left\| A_t^{-1/2} \bb_t \right\|
    \le 2B\sqrt{\frac{\|\overline{A}_t\|}{\gamma}} \big( d\log6 + \log(2T/\delta) \big)
    \leq \ttR^{1/2} B^{1/2} \big( d\log6 + 2 \log T + \log(3/\delta) \big)^{1/2}
\]
with probability at least $(1 - \delta)$ simultaneously for all $t \in \{1, \dots, \tau^*\}$. This and \eqref{eq:eta_condition_change_point} yield that there is an event $\mathcal E_2$ of probability at least $(1 - 3\delta)$, such that
\[
    \left\|A_t^{1/2} \widehat{\btheta}_{1:t-1}(\eta_t)\right\|^2 \vee \left\| A_t^{-1/2} \bb_t \right\|^2
    \leq \ttR B \left(1 \vee 4\gamma^2 \ttR \right) \big( d\log6 + 2\log T + \log(3/\delta) \big)
    \leq \frac1{4 \eta_{t}}
\]
for all $t \in \{1, \dots, T\}$. Thus, the conditions of Theorem \ref{th:fs_regret} are satisfied on $\cE_2$, and we have
\begin{align}
    \label{eq:fs_loss_upper_bound}
    \widehat L_{1:t}^{FS} - L_{1:\tau^*}(\btheta_1^\star) - L_{\tau^* + 1:t}(\btheta_{\tau^* + 1}^\star)
    &\notag
    \leq \frac{\log(1 / \alpha)}{\eta_t} + \frac{(t - 2) \log\big(1 / (1 - \alpha) \big)}{\eta_t}
    \\&
    + \left[ \frac{\lambda \|\btheta_{1 : \tau^*}^\circ\|^2}{2 \eta_t} + \frac1{2 \eta_t} \log \det\left( I_d + \frac{\eta_t}\lambda \sum\limits_{j = 1}^{\tau^*} A_j \right) \right]
    \\&\notag
    + \left[ \frac{\lambda \|\btheta_{\tau^* + 1 : t}^\circ\|^2}{2 \eta_t} + \frac1{2 \eta_t} \log \det\left( I_d + \frac{\eta_t}\lambda \sum\limits_{j = \tau^* + 1}^{t} A_j \right) \right],
\end{align}
where
\[
    \btheta_{1:\tau^*}^\circ
    = \left( \sum\limits_{j = 1}^{\tau^*} A_j \right)^{-1} \sum\limits_{j = 1}^{\tau^*} \bb_j
    \quad \text{and} \quad
    \btheta_{\tau^* + 1:t}^\circ
    = \left( \sum\limits_{j = \tau^* + 1}^t A_j \right)^{-1} \sum\limits_{j = \tau^* + 1}^t \bb_j.
\]
Let us note that
\[
    \frac{\log(1 / \alpha)}{\eta_t} + \frac{(t - 2) \log\big(1 / (1 - \alpha) \big)}{\eta_t}
    = \frac{\log T}{\eta_t} + \frac{(t - 2) \log\big(1 + 1 / (T - 1)) \big)}{\eta_t}
    \leq \frac{1 + \log T}{\eta_t}.
\]
We use Lemmata \ref{lem:a_sum_norm} and \ref{lem:b_partial_sum_norm} to bound the expression in the right-hand side of \eqref{eq:fs_loss_upper_bound}. First, due to Lemma \ref{lem:b_partial_sum_norm} and the union bound, with probability at least $(1 - \delta)$, it holds that
\[
    \left\|\btheta_{1:\tau^*}^\circ\right\|
    \leq \frac{1}{\tau^* \gamma} \left\| \sum\limits_{j = 1}^{\tau^*} \bb_j \right\|
    \leq \frac{4B \big\| \overline A_1 \big\|^{1/2}}{\gamma} \left( \frac{\log 2}{2} + \frac{d \log 6}{\tau^*} + \frac{\log(2 / \delta)}{\tau^*} \right)
    \leq \frac{\ttR}{\big\| \overline A_1 \big\|^{1/2}}
\]
and
\begin{align*}
    \left\|\btheta_{\tau^* + 1 : t}^\circ\right\|
    &
    \leq \frac{1}{(t - \tau^*) \gamma} \left\| \sum\limits_{j = \tau^* + 1}^t \bb_j \right\|
    \\&
    \leq 4B \left(\|\overline{A}_t\|^{1/2} \vee \|\overline{A}_1\|^{1/2} \right) \left(d\log6 + (t-\tau^*)\frac{\log2}{2} + \log\frac{2(T - \tau^*)}{\delta}\right)
    \\&
    \leq \frac{\ttR}{\big\| \overline A_{\tau^* + 1} \big\|^{1/2}}
\end{align*}
simultaneously for all $\tau^* < t \leq T$.
Second, Lemma \ref{lem:a_sum_norm} yields that
\[
    \left\| \sum\limits_{j = 1}^{\tau^*} \big(\overline{A}_1\big)^{-1/2} A_j \big(\overline{A}_1\big)^{-1/2} \right\|
    \le 4B \big( \tau^* \log 2 + d \log6 + \log(1/\delta) \big)
    \leq \frac{\gamma \ttR \tau^*}{\big\| \overline A_1 \big\|}
\]
with probability at least $(1 - \delta)$, and, similarly,
\[
    \left\| \sum\limits_{j = \tau^* + 1}^{T} \big(\overline{A}_{\tau^* + 1} \big)^{-1/2} A_j \big(\overline{A}_{\tau^* + 1} \big)^{-1/2} \right\|
    \le 4B \big( (T - \tau^*) \log 2 + d \log6 + \log(1/\delta) \big)
    \leq \frac{\gamma \ttR (T - \tau^*)}{\big\| \overline A_{\tau^* + 1} \big\|}
\]
with the same probability.
Then, with probability at least $(1 - 2\delta)$, it holds that
\begin{align*}
    \log \det\left( I_d + \frac{\eta_t}{\lambda} \sum\limits_{j = 1}^{\tau^*} A_j \right)
    &
    \leq d \log\left(1 + \frac{\eta_t \|\overline{A}_1\|}{\lambda} \left\| \sum\limits_{j = 1}^{\tau^*} \big(\overline{A}_1\big)^{-1/2} A_j \big(\overline{A}_1\big)^{-1/2} \right\| \right)
    \\&
    \leq d \log\left(1 + \frac{\gamma B \, \ttR \, \eta_t \tau^*}{\lambda} \right)
\end{align*}
and
\begin{align*}
    \log \det\left( I_d + \frac{\eta_t}{\lambda} \sum\limits_{j = \tau^* + 1}^{t} A_j \right)
    &
    \leq \log \det\left( I_d + \frac{\eta_t}{\lambda} \sum\limits_{j = \tau^* + 1}^{T} A_j \right)
    \\&
    \leq d \log\left(1 + \frac{\eta_t \|\overline{A}_{\tau^* + 1}\|}{\lambda} \left\| \sum\limits_{j = \tau^* + 1}^{T} \big(\overline{A}_{\tau^* + 1} \big)^{-1/2} A_j \big(\overline{A}_{\tau^* + 1} \big)^{-1/2} \right\| \right)
    \\&
    \leq d \log\left(1 + \frac{\gamma B \, \ttR \, \eta_t (T - \tau^*)}{\lambda} \right)
\end{align*}
Thus, with probability at least $(\p(\cE_2) - 3\delta) \geq (1 - 6 \delta)$, it holds that
\begin{align*}
    \widehat L_{1:t}^{FS} - L_{1:\tau^*}(\btheta_1^\star) - L_{\tau^* + 1:t}(\btheta_{\tau^* + 1}^\star)
    &
    \leq \frac{\lambda \, \ttR^2}{2 \eta_t} \left( \frac1{\big\| \overline A_1 \big\|} + \frac1{\big\| \overline A_{\tau^* + 1} \big\|} \right)
    + \frac{d}{2\eta_t} \log\left(1 + \frac{\gamma B \, \ttR \, \eta_t \tau^*}{\lambda} \right)
    \\&\quad
    + \frac{d}{2\eta_t} \log\left(1 + \frac{\gamma B \, \ttR \, \eta_t (T - \tau^*)}{\lambda} \right) + \frac{1 + \log T}{\eta_t}
\end{align*}
simultaneously for all $t \in \{\tau^* + 1, \dots, T\}$.

\medskip

\noindent
\textbf{Step 2: lower bound on $\widehat L_{1:t}^{EW} - L_{1:\tau^*}(\btheta_1^\star) - L_{\tau^* + 1:t}(\btheta_{\tau^* + 1}^\star)$.}
\quad
It remains to provide a lower bound on the cumulative loss $\widehat L_{1:t}^{EW}$. According to
Lemma \ref{lem:b_partial_sum_norm} and the union bound, simultaneously for all \(1 \le t \le T\) and \(\delta \in (0, 1)\) with probability at least \( (1 - \delta) \) it holds that 
\begin{align*}
    \left\|\overline A_t^{1/2} \widehat{\btheta}_{1:t-1}(\eta_t)\right\|
    &
    \le \big\| \overline A_t \big\|^{1/2} \left\| \left( \sum\limits_{j = s}^{t - 1} A_j + \frac{\lambda}\eta I_d \right)^{-1} \sum\limits_{j = s}^{t - 1} \bb_j \right\|
    \le \frac{\big\| \overline A_t \big\|^{1/2}}{(t - s) \gamma} \left\| \sum\limits_{j = s}^{t - 1} \bb_j \right\|
    \\&
    \le \frac{4B \big\| \overline A_t \big\|^{1/2}}{(t - s) \gamma} \left( \big\| \overline A_t \big\|^{1/2} \vee \big\| \overline A_1 \big\|^{1/2} \right) \left(d\log6 + (t - 1)\frac{\log2}{2} + \log\frac{2T}{\delta}\right)
    \leq \ttR.
\end{align*}
Since, by the condition of the theorem (see \eqref{eq:r_change_point}), we also have 
\[
    \left\|\overline A_t^{1/2} \btheta_t^\star \right\|
    = \left\| \big( \overline A_t \big)^{-1/2} \overline \bb_t \right\|
    \leq \ttR
    \quad \text{for any $1 \leq t \leq T$,}
\]
we can apply Theorem \ref{th:lower_bound}. Then, with probability at least $(1 - 2\delta)$, it holds that
\begin{align*}
    \widehat L_{1:t}^{EW} - L_{1:\tau^*}(\btheta_1^\star) - L_{\tau^* + 1:t}(\btheta_{\tau^* + 1}^\star)
    &
    \geq \sum\limits_{j = \tau^* + 1}^t \left( \overline\ell_j \big( \widehat\btheta{}_j^{EW} \big) - \overline\ell_j \big( \btheta_{\tau^* + 1}^\star \big) \right)
    \\&
    - 4B \left(3e \ttR^2 \log(8B) + 4B\right) \log(T/\delta)
    \\&
    - 12B \, \ttR \ttZ (1 + \ttR) \log(T/\delta)
\end{align*}
simultaneously for all $1 \leq t \leq T$.
Hence, applying the union bound, we obtain that
\begin{align*}
    \widehat S_t
    &
    \geq \sum\limits_{j = \tau^* + 1}^t \left( \overline\ell_j \big( \widehat\btheta{}_j^{EW} \big) - \overline\ell_j \big( \btheta_{\tau^* + 1}^\star \big) \right) - \frac{1 + \log T}{\eta_t} - \frac{\lambda \, \ttR^2}{2 \eta_t} \left( \frac1{\big\| \overline A_1 \big\|}
    + \frac1{\big\| \overline A_{\tau^* + 1} \big\|} \right)
    \\&\quad
    - \frac{d}{2\eta_t} \log\left(1 + \frac{\gamma B \, \ttR \, \eta_t \tau^*}{\lambda} \right)
    - \frac{d}{2\eta_t} \log\left(1 + \frac{\gamma B \, \ttR \, \eta_t (T - \tau^*)}{\lambda} \right)
    \\&\quad
    - 4B \left(3e \ttR^2 \log(8B) + 4B\right) \log(T/\delta)
    - 12B \, \ttR \ttZ (1 + \ttR) \log(T/\delta)
\end{align*}
for all $t \in \{\tau^* + 1, \dots, T\}$ on an even of probability at least $(1 - 8\delta)$.
Finally, since $\btheta_{\tau^* + 1}^\star$ minimizes the quadratic loss $\overline{\ell}_j$ for all $j \in \{\tau^* + 1, \dots, T\}$, it holds that
\[
    \overline\ell_j \big( \widehat\btheta{}_j^{EW} \big) - \overline\ell_j \big( \btheta_{\tau^* + 1}^\star \big)
    = \frac12 \left\| \big(\overline{A}_j\big)^{1/2} \big( \widehat\btheta{}_j^{EW} - \btheta_{\tau^* + 1}^\star \big) \right\|^2.
\]
Due to the definition of $\overline{A}_j$, we have
\begin{align*}
    \overline\ell_j \big( \widehat\btheta{}_j^{EW} \big) - \overline\ell_j \big( \btheta_{\tau^* + 1}^\star \big)
    &
    = \E_{\bX_j} \left\| \nabla \bPsi(\bX_j)^\top \big( \widehat\btheta{}_j^{EW} - \btheta_{\tau^* + 1}^\star \big) \right\|^2 + \gamma \left\| \widehat\btheta{}_j^{EW} - \btheta_{\tau^* + 1}^\star \right\|^2
    \\&
    = \E_{\bX_j} \left\| \bnabla \log \sfp_{\widehat\btheta{}_j^{EW}}(\bX_j) - \bnabla \log \sfp_{\btheta_{\tau^* + 1}^\star}(\bX_j) \right\|^2 + \gamma \left\| \widehat\btheta{}_j^{EW} - \btheta_{\tau^* + 1}^\star \right\|^2.
\end{align*}

\myendproof

\subsection{Proof of Theorem \ref{th:lower_bound}}

\begin{proof}  
    Since \( \{ \widecheck{\btheta}_t : 1 \leq t \leq T\}\) is a non-anticipating sequence, we can define a martingale difference sequence
    \begin{align*}
        \left\{ \xi_t = \overline{\ell}_t(\widecheck{\btheta}_t) - \overline{\ell}_t(\btheta^\star_t) - \ell_t(\widecheck{\btheta}_t) + \ell_t(\btheta_t^\star) : 1 \leq t \leq T \right\}
    \end{align*}
    adapted to the filtration \(\{\cF_t : 0 \leq t \leq T\}\).
    Let us fix arbitrary $\delta \in (0, 1)$ and $s \in \{1, \dots, T\}$ and show that
    \begin{align*}
        \sum\limits_{t=1}^s \big(\ell_t(\widecheck{\btheta}_t) - \ell_t(\btheta_t^\star)\big)
        &
        \geq \frac12 \sum\limits_{t = 1}^s \big( \overline{\ell}_t(\widecheck{\btheta}_t) - \overline{\ell}_t(\btheta_t^\star) \big) - 4B \left(3eR^2\log(8B) + 4B\right) \log(1/\delta)
        \\&\quad
        - 12BRz (1 + R) \log(1/\delta)
    \end{align*}
    with probability at least \(\p(\cE) - \delta\).
    According to Lemma \ref{lem:loss_diff_var_bound}, we have 
    \begin{align*}
        \E \left[ (\ell_t(\widecheck{\btheta}_t) - \ell_t(\btheta_t^\star))^2 \,\big|\, \cF_{t-1} \right] 
        &
        \le B\left\|\overline{A}_t^{1/2}(\widecheck{\btheta}_t-\btheta_t^\star)\right\|^2\left(\frac{3e}{2}\left\|\overline{A}_t^{1/2}(\widecheck{\btheta}_t+\btheta_t^\star)\right\|^2 \log(8B) + 8B\right)
        \\&
        \le B \left(6eR^2\log(8B) + 8B\right) \left\|\overline{A}_t^{1/2} (\widecheck{\btheta}_t - \btheta_t^\star)\right\|^2.
    \end{align*}
    This yields that
    \[
        \sum\limits_{t = 1}^s \E[\xi^2_t \mid \cF_{t-1}]
        \le 2B \left(3eR^2\log(8B) + 4B\right) \sum\limits_{t = 1}^s \left\|\overline{A}_t^{1/2} (\widecheck{\btheta}_t - \btheta_t^\star)\right\|^2.
    \]
    On the other hand, Lemma \ref{lem:loss_diff_psi1_bound} implies that
    \begin{align*}
         \|\ell_t(\widecheck{\btheta}_t) - \ell_t(\btheta_t^\star)\mid \cF_{t-1} \|_{\psi_1} &
         \le B \left\|\overline{A}_t^{1/2}(\widecheck{\btheta}_t - \btheta_t^\star)\right\| \left(1 + \frac12 \left\|\overline{A}_t^{1/2} (\widecheck{\btheta}_t + \btheta_t^\star) \right\| \right)
         \\&
         \le B \left(1 + R\right) \left\| \overline{A}_t^{1/2} (\widecheck{\btheta}_t - \btheta_t^\star) \right\|. 
    \end{align*}
    Since any sub-exponential random variable \(\eta\) satisfies
    \[
        \|\eta - \E\eta\|_{\psi_1}
        \le \|\eta\|_{\psi_1} + \left\| \E \eta \right\|_{\psi_1}
        = \|\eta\|_{\psi_1} + \frac{\left| \E \eta \right|}{\log 2}
        \leq \left( 1 + \frac{2}{\log 2} \right) \|\eta\|_{\psi_1}
        \leq 4 \|\eta\|_{\psi_1},
    \]
    we obtain that
    \begin{align*}
        \sum\limits_{t = 1}^s \left\| \xi_t \mid \cF_{t - 1} \right\|_{\psi_1}^2
        \leq 16 B^2 (1 + R)^2 \sum\limits_{t = 1}^s \left\| \overline{A}_t^{1/2} (\widecheck{\btheta}_t - \btheta_t^\star) \right\|^2
    \end{align*}
    and 
    \begin{align*}
        \max_{1 \le t \le s} \left\| \xi_t \mid \cF_{t - 1} \right\|_{\psi_1}
        \le 4BR (1 + R).
    \end{align*}
    Hence, we are in position to apply Theorem 2.1 from \citet{kroshnin25} claiming that with probability at least \(\p(\cE) - \delta\) 
    \[
        \sum\limits_{t = 1}^s \xi_t
        \le \upvarsigma \sqrt{2\log(1/\delta)} + 12BRz (1 + R) \log(1/\delta),
    \]
    where
    \[
        \upvarsigma^2 = 2B \left(3eR^2\log(8B) + 4B\right) \sum\limits_{t = 1}^s \left\|\overline{A}_t^{1/2} (\widecheck{\btheta}_t - \btheta_t^\star)\right\|^2
        \quad \text{and} \quad
        z = 1 \vee \log \frac{e(1 + R)\sqrt{2B}}{\sqrt{3eR^2\log(8B) + 4B}}.
    \]
    Note that, since \(\btheta_t^\star\) minimizes the expected loss \(\overline{\ell}_t(\btheta) = \btheta^\top \overline{A}_t \btheta / 2 - \bb_t^\top \btheta\), we have
    \[
        \overline{\ell}_t(\widecheck{\btheta}_t) - \overline{\ell}_t(\btheta^\star_t)
        = \frac12 \left\|\overline{A}_t^{1/2} (\widecheck{\btheta}_t - \btheta_t^\star) \right\|^2
        = \frac{\upvarsigma^2}{4B} \left(3eR^2\log(8B) + 4B\right)^{-1}.
    \]
    Using the Cauchy-Schwarz inequality, we obtain that
    \begin{align*}
        \sum\limits_{t = 1}^s \xi_t
        &
        \le \frac{\upvarsigma^2}{8B} \left(3eR^2\log(8B) + 4B\right)^{-1} + 4B \left(3eR^2\log(8B) + 4B\right) \log(1/\delta)
        \\&\quad
        + 12BRz (1 + R) \log(1/\delta)
        \\&
        = \frac12\sum\limits_{t=1}^s\big(\overline{\ell}_t(\widecheck{\btheta}_t) - \overline{\ell}_t(\btheta_t^\star)\big) + 4B \left(3eR^2\log(8B) + 4B\right) \log(1/\delta)
        \\&\quad
        + 12BRz (1 + R) \log(1/\delta)
    \end{align*}
    with probability at least \(\p(\cE) - \delta\).
    Recalling the definition of $\xi_t$ and rearranging the terms, we conclude that    
    \begin{align*}
        \sum\limits_{t=1}^s \big(\ell_t(\widecheck{\btheta}_t) - \ell_t(\btheta_t^\star)\big)
        &
        \geq \frac12 \sum\limits_{t = 1}^s \big( \overline{\ell}_t(\widecheck{\btheta}_t) - \overline{\ell}_t(\btheta_t^\star) \big) - 4B \left(3eR^2\log(8B) + 4B\right) \log(1/\delta)
        \\&\quad
        - 12BRz (1 + R) \log(1/\delta)
    \end{align*}
    on the same event.
    
\end{proof}

\begin{Lem}
    \label{lem:loss_diff_var_bound}
    Under Assumption \ref{as:subexp}, for any \(\btheta \in \R^d\) it holds that
    \begin{align*}
        \E\left(\ell_t(\btheta) - \ell_t(\btheta_t^\star)\right)^2 \le B\left\|\overline{A}_t^{1/2}(\btheta-\btheta_t^\star)\right\|^2 \left(3e \left\|\overline{A}_t^{1/2} (\btheta + \btheta_t^\star) \right\|^2 \log(8B) + 8B\right).
    \end{align*}
\end{Lem}

\begin{proof}
    According to the definition of $\ell_t$ and the Cauchy-Schwarz inequality, it holds that
    \begin{align*}
        \E\left(\ell_t(\btheta) - \ell_t(\btheta_t^\star)\right)^2 \le \E(\btheta^\top A_t \btheta - \btheta_t^{\star\top} A_t \btheta_t^\star)^2 + 2\E\left(\bb_t^\top(\btheta-\btheta_t^\star) \right)^2.
    \end{align*}
    In the rest of the proof, we are going show that
    \begin{equation}
        \label{eq:sq_loss_diff_bound}
        \E(\btheta^\top A_t \btheta - \btheta_t^{\star\top} A_t \btheta_t^\star)^2
        \le 3eB  \log(8B)\left\|\overline{A}_t^{1/2}(\btheta+\btheta_t^\star)\right\|^2 \left\|\overline{A}_t^{1/2}(\btheta-\btheta_t^\star)\right\|^2
    \end{equation}
    and
    \begin{equation}
        \label{eq:lin_loss_diff_bound}
        \E\big(\bb_t^\top(\btheta-\btheta_t^\star)\big)^2 
        \le 
        4B^2 \left\|\overline{A}_t^{1/2}(\btheta-\btheta_t^\star)\right\|^2.
    \end{equation}

    \medskip

    \noindent
    \textbf{Step 1: bound on $\E(\btheta^\top A_t \btheta - \btheta_t^{\star\top} A_t \btheta_t^\star)^2$.}
    \quad
    We start with the proof of \eqref{eq:sq_loss_diff_bound}.
    Let us fix an arbitrary $\btheta \in \R^d$ and introduce unit vectors
    \begin{align}
        \label{eq:uv_def}
        \bu = \frac{\overline{A}_t^{1/2}(\btheta + \btheta_t^\star)}{\left\|\overline{A}_t^{1/2}(\btheta + \btheta_t^\star)\right\|}
        \quad \text{and} \quad
        \bv = \frac{\overline{A}_t^{1/2}(\btheta - \btheta_t^\star)}{\left\|\overline{A}_t^{1/2}(\btheta - \btheta_t^\star)\right\|}.
    \end{align}
    Using the identity
    \[
        \left(\btheta^\top A_t \btheta - \btheta_t^{\star\top} A_t \btheta_t^\star\right)^2
        = \left(\left(A_t^{1/2}(\btheta+\btheta_t^\star)\right)^\top\left(A_t^{1/2}(\btheta-\btheta_t^\star)\right)\right)^2
    \]
    and the Cauchy-Schwarz inequality, we note that
    \begin{align}
        \label{eq:squared_excess_loss_cauchy_schwarz_bound}
        &\notag
        \left(\btheta^\top A_t \btheta - \btheta_t^{\star\top} A_t \btheta_t^\star\right)^2
        \le \left\|A_t^{1/2}(\btheta+\btheta_t^\star)\right\|^2  \left\|A_t^{1/2}(\btheta-\btheta_t^\star)\right\|^2
        \\&
        =  \left\|\overline{A}_t^{1/2} (\btheta + \btheta_t^\star)\right\|^2  \left\|\overline{A}_t^{1/2}(\btheta-\btheta_t^\star)\right\|^2  \left\|A_t^{1/2}\big(\overline{A}_t\big)^{-1/2} \bu \right\|^2 \left\| A_t^{1/2} \big(\overline{A}_t\big)^{-1/2} \bv \right\|^2.
    \end{align}
    Let us focus on the third and the fourth terms in the right-hand side.
    Applying Lemma \ref{lem:expectation_prod_bound}, we obtain that
    \begin{align*}
        &
        \E\left\| A_t^{1/2} \big(\overline{A}_t\big)^{-1/2} \bu \right\|^2 \left\|A_t^{1/2} \big(\overline{A}_t\big)^{-1/2} \bv \right\|^2
        \\
        &
        \le 3e  \E \left\|A_t^{1/2} \big(\overline{A}_t\big)^{-1/2 } \bu \right\|^2 \;
        \left\| \left\| A_t^{1/2} \big(\overline{A}_t\big)^{-1/2} \bv \right\|^2 \right\|_{\psi_1} 
        \log\frac{8 \left\| \left\|A_t^{1/2} \big(\overline{A}_t\big)^{-1/2} \bu \right\|^2 \right\|_{\psi_1}}{\E\left\|A_t^{1/2}\big(\overline{A}_t\big)^{-1/2} \bu \right\|^2}.
    \end{align*}
    Due to Assumption \ref{as:subexp}, the Orlicz norms of $\left\|A_t^{1/2}\big(\overline{A}_t\big)^{-1/2} \bv\right\|^2$ and $\left\|A_t^{1/2} \big(\overline{A}_t\big)^{-1/2} \bu \right\|^2$ do not exceed $B$. On the other hand, it is straightforward to observe that
    \begin{align*}
        \E\left\| A_t^{1/2} \big(\overline{A}_t\big)^{-1/2} \bu \right\|^2 
        = \bu^\top \big(\overline{A}_t\big)^{-1/2} \E A_t \big(\overline{A}_t\big)^{-1/2} \bu
        = \bu^\top \bu
        = 1.
    \end{align*}
    Thus, it holds that
    \begin{align*}
        &
        \E \left\| A_t^{1/2} \big(\overline{A}_t\big)^{-1/2} \bu \right\|^2  
        \left\| A_t^{1/2} \big(\overline{A}_t\big)^{-1/2} \bv \right\|^2 
        \\&
        \le 3e \left\| \left\| A_t^{1/2} \big(\overline{A}_t\big)^{-1/2} \bv \right\|^2 \right\|_{\psi_1} \log \left( 8\left\|\left\|A_t^{1/2}\big(\overline{A}_t\big)^{-1/2} \bu\right\|^2\right\|_{\psi_1} \right)
        \\&
        \le 3eB\log(8B).
    \end{align*}
    Summing up this bound with \eqref{eq:squared_excess_loss_cauchy_schwarz_bound}, we deduce that
    \begin{align*}
        \E(\btheta^\top A_t \btheta - \btheta_t^{\star\top} A_t \btheta_t^\star)^2 \le 3eB \log(8B)\left\|\overline{A}_t^{1/2}(\btheta+\btheta_t^\star)\right\|^2 \left\|\overline{A}_t^{1/2}(\btheta-\btheta_t^\star)\right\|^2,
    \end{align*}
    and \eqref{eq:sq_loss_diff_bound} follows.

    \medskip

    \noindent
    \textbf{Step 2: bound on $\E\big(\bb_t^\top(\btheta-\btheta_t^\star)\big)^2$.}
    \quad
    It remains to check \eqref{eq:lin_loss_diff_bound}. Its proof is straightforward. Recalling the definition of $v$ (see eq.~\ref{eq:uv_def}) and applying Lemma \ref{lem:moment_subexp_norm_bound}, we obtain that
    \[
        \E \big(\bb_t^\top(\btheta - \btheta_t^\star)\big)^2 
        = \left\| \overline{A}_t^{1/2} (\btheta-\btheta_t^\star) \right\|^2 \E\left( \bb_t^\top \big(\overline{A}_t\big)^{-1/2} \bv \right)^2
        \le 
        4 \left\| \overline{A}_t^{1/2} (\btheta-\btheta_t^\star) \right\|^2 \left\| \bb_t^\top \big(\overline{A}_t\big)^{-1/2} \bv \right\|^2_{\psi_1}.
    \]
    Due to Assumption \ref{as:subexp}, the expression in the right-hand side does not exceed $\smash{4 B^2 \left\|\overline{A}_t^{1/2}(\btheta-\btheta_t^\star) \right\|^2}$. Hence, \eqref{eq:lin_loss_diff_bound} holds. Taking together \eqref{eq:sq_loss_diff_bound} and \eqref{eq:lin_loss_diff_bound}, we obtain the desired bound.
    
\end{proof}

\begin{Lem}\label{lem:loss_diff_psi1_bound}
    Under Assumption \ref{as:subexp}, for any \(\btheta \in \R^d\) it holds that
    \[
        \big\|\ell_t(\btheta) - \ell_t(\btheta_t^\star) \big\|_{\psi_1}
        \le B \left\|\overline{A}_t^{1/2}(\btheta - \btheta_t^\star)\right\|\left(1+\frac12 \left\|\overline{A}_t^{1/2}(\btheta + \btheta_t^\star)\right\|\right).
    \]
\end{Lem}

\begin{proof}
    Due to the triangle inequality, it holds that
    \begin{align*}
         \|\ell_t(\btheta) - \ell_t(\btheta_t^\star)\|_{\psi_1}
         \le \frac12 \left\|\btheta^\top A_t \btheta - \btheta_t^{\star\top} A_t \btheta_t^\star \right\|_{\psi_1} + \left\|\bb_t^\top(\btheta-\btheta_t^\star) \right\|_{\psi_1}.
    \end{align*}
    Similarly to the proof of Lemma \ref{lem:loss_diff_var_bound} (see eq.~\ref{eq:uv_def}), let us introduce
    \[
        \bu = \frac{\overline{A}_t^{1/2}(\btheta + \btheta_t^\star)}{\left\|\overline{A}_t^{1/2}(\btheta + \btheta_t^\star)\right\|}
        \quad \text{and} \quad
        \bv = \frac{\overline{A}_t^{1/2}(\btheta - \btheta_t^\star)}{\left\|\overline{A}_t^{1/2}(\btheta - \btheta_t^\star)\right\|}.
    \]
    Then we deduce that
    \begin{align}
        \label{eq:loss_diff_psi1}
        \|\ell_t(\btheta) - \ell_t(\btheta_t^\star)\|_{\psi_1}
        & \notag
        \leq \frac12\left\|A_t^{1/2}(\btheta + \btheta_t^\star)\right\| \, \left\|A_t^{1/2}(\btheta - \btheta_t^\star)\right\| \,
        \left\| \bu^\top\big(\overline{A}_t\big)^{-1/2} A_t \big(\overline{A}_t\big)^{-1/2} \bv \right\|_{\psi_1} \\&\quad
        + \left\| \overline{A}_t^{1/2}(\btheta - \btheta_t^\star) \right\| \, \left\|\bb_t^\top \big(\overline{A}_t\big)^{-1/2} \bv \right\|_{\psi_1}.
    \end{align}
    Since any two random variables \(\xi\) and \(\eta\) with finite \(\psi_2\)-norms satisfy
    \begin{align}\label{eq:psi_1_prod}
        \|\xi\eta\|_{\psi_1} \le \|\xi\|_{\psi_2} \|\eta\|_{\psi_2} = \sqrt{\|\xi\|^2_{\psi_1}}\sqrt{\|\eta\|^2_{\psi_1}},
    \end{align}
    the right-hand side of \eqref{eq:loss_diff_psi1} does not exceed
    \begin{align*}
        &
        \frac12 \left\|A_t^{1/2}(\btheta + \btheta_t^\star) \right\| \,
        \left\| A_t^{1/2} (\btheta - \btheta_t^\star) \right\|
        \sqrt{\left\| \big\|A_t^{1/2} \big(\overline{A}_t\big)^{-1/2} \bu \big\|^2 \right\|_{\psi_1}} \cdot \sqrt{\left\| \big\| A_t^{1/2} \big(\overline{A}_t\big)^{-1/2} \bv \big\|^2 \right\|_{\psi_1}}
        \\&\quad
        + \left\|\overline{A}_t^{1/2} (\btheta - \btheta_t^\star) \right\| \left\|\bb_t^\top \big(\overline{A}_t\big)^{-1/2} \bv \right\|_{\psi_1}.
    \end{align*}
    In view of Assumption \ref{as:subexp}, we obtain that
    \[
        \max\left\{ \left\| \big\|A_t^{1/2} \big(\overline{A}_t\big)^{-1/2} \bu \big\|^2 \right\|_{\psi_1}, \left\| \big\| A_t^{1/2} \big(\overline{A}_t\big)^{-1/2} \bv \big\|^2 \right\|_{\psi_1}, \left\|\bb_t^\top \big(\overline{A}_t\big)^{-1/2} \bv \right\|_{\psi_1} \right\}
        \leq B,
    \]
    and then
    \[
         \big\|\ell_t(\btheta) - \ell_t(\btheta_t^\star) \big\|_{\psi_1}
         \le \frac{B}2\left\|A_t^{1/2}(\btheta + \btheta_t^\star)\right\| \left\|A_t^{1/2}(\btheta - \btheta_t^\star)\right\| + B \left\|\overline{A}_t^{1/2}(\btheta - \btheta_t^\star)\right\|.
    \]
\end{proof}

\begin{Lem}
    \label{lem:scaled_bt_bound}
    Under Assumption~\ref{as:subexp}, for any \(1 \le t \le T\) and for any \(\delta \in (0,1)\) it holds that
    \[
        \left\| A_t^{-1/2} \bb_t \right\|
        \le 2B\sqrt{\frac{\|\overline{A}_t\|}{\gamma}} \big( d\log6 + \log(2/\delta) \big)
    \]
    with probability at least \(1 - \delta\).
\end{Lem}

\begin{proof}
    Let us fix an arbitrary \(\delta \in (0, 1)\). Since, by the definition,
    \[
        A_t = \nabla \bPsi(\bX_t) \nabla \bPsi(\bX_t)^\top + \gamma I_d \succeq \gamma I_d
        \quad \text{almost surely,}
    \]
    it holds that
    \[
         \left\|A_t^{-1/2} \bb_t \right\|
         \le \left\| A_t^{-1/2} \big(\overline{A}_t\big)^{1/2} \right\| \, \left\| \big(\overline{A}_t\big)^{-1/2} \bb_t \right\|
         \leq \left\| \big(\overline{A}_t\big)^{-1/2} \bb_t \right\| \sqrt{\frac{\|\overline{A}_t\|}{\gamma}}.
    \]
    In the rest of the proof, we establish a large deviation bound for the norm of $\big(\overline{A}_t\big)^{-1/2} \bb_t$. Applying the $\eps$-net argument (see \citeauthor{rigollet23}, \citeyear[the proof of Theorem 1.19]{rigollet23}), we obtain that for any $\z > 0$
    \begin{align*}
        \p\left( \big\| \big(\overline{A}_t\big)^{-1/2} \bb_t \big\| > \z \right)
        = \p\left( \sup_{\|\bu\| = 1} \bu^\top \big(\overline{A}_t\big)^{-1/2} \bb_t > \z \right) 
        \le 6^d \sup_{\|\bu\| = 1} \p\left( \bu^\top \big(\overline{A}_t\big)^{-1/2} \bb_t > \z/2 \right).
    \end{align*}
    Due to Assumption \ref{as:subexp}, for any unit vector $\bu$ we have
    \begin{align*}
        \p\left( \bu^\top \big(\overline{A}_t\big)^{-1/2} \bb_t > \z/2 \right)
        &
        = \p\left( \exp\left\{ \bu^\top \big(\overline{A}_t\big)^{-1/2} \bb_t / B \right\} > \exp\left\{ \frac{\z}{2B} \right\} \right)
        \\&
        \le \exp\left\{-\frac{\z}{2B}\right\} \, \E \exp\left\{ |\bu^\top \big(\overline{A}_t\big)^{-1/2} \bb_t| / B \right\}
        \\&
        \le 2\exp\left\{-\frac{\z}{2B}\right\}.
    \end{align*}
    This yields that
    \[
        \p\left( \big\| \big(\overline{A}_t\big)^{-1/2} \bb_t \big\| > \z \right)
        \leq 6^d \sup_{\|\bu\| = 1} \p\left( \bu^\top \big(\overline{A}_t\big)^{-1/2} \bb_t > \z/2 \right)
        \le 2 \cdot 6^d \cdot \exp\left\{-\frac{\z}{2B}\right\}.
    \]
    Hence, with probability at least \(1-\delta\), it holds that
    \[
        \left\| \big(\overline{A}_t\big)^{-1/2} \bb_t \right\|
        \le 2B \big( d\log6 + \log(2/\delta) \big),
    \]
    and, consequently,
    \begin{align*}
         \left\|A_t^{-1/2}\bb_t \right\|
         \le
         \left\| \big(\overline{A}_t\big)^{-1/2} \bb_t \right\| \sqrt{\frac{\|\overline{A}_t\|}{\gamma}} \le 2B \sqrt{\frac{\|\overline{A}_t\|}{\gamma}} \big( d\log6 + \log(2/\delta) \big).
    \end{align*}
\end{proof}

\begin{Lem}
    \label{lem:a_sum_norm}
    For any $1 \leq s \leq t \leq T$ and any $\delta \in (0, 1)$, with probability at least $(1 - \delta)$ it holds that
    \[
        \left\| \sum\limits_{j = s}^t \big(\overline{A}_j\big)^{-1/2} A_j \big(\overline{A}_j\big)^{-1/2} \right\|
        \le 4B \big( (t - s + 1) \log 2 + d \log6 + \log(1/\delta) \big).
    \]
\end{Lem}

\begin{proof}
    Let us fix an arbitrary \(\z > 0\) and consider the probability 
    \[
        \p\left( \left\| \sum\limits_{j = s}^t \big(\overline{A}_j\big)^{-1/2} A_j \big(\overline{A}_j\big)^{-1/2} \right\| > \z \right).
    \]
    By the definition of the matrix spectral norm, 
    \begin{align*}
        \left\| \sum\limits_{j = s}^t \big(\overline{A}_j\big)^{-1/2} A_j \big(\overline{A}_j\big)^{-1/2} \right\|
        &
        = \sup_{\|\bu\|=1} \left| \bu^\top \sum\limits_{j = s}^t \big(\overline{A}_j\big)^{-1/2} A_j \big(\overline{A}_j\big)^{-1/2} \bu \right|
        \\&
        = \sup_{\|\bu\|=1} \left\| \left( \sum\limits_{j = s}^t \big(\overline{A}_j\big)^{-1/2} A_j \big(\overline{A}_j\big)^{-1/2} \right)^{1/2} \bu \right\|^2.
    \end{align*}
    Applying the standard \(\eps\)-net argument (see \citeauthor{rigollet23}, \citeyear[proof of Theorem 1.19]{rigollet23}), 
    \begin{align*}
        \p\left( \left\| \sum\limits_{j = s}^t \big(\overline{A}_j\big)^{-1/2} A_j \big(\overline{A}_j\big)^{-1/2} \right\| > \z \right)
        = \p\left( \sup_{\|\bu\|=1} \left\| \left( \sum\limits_{j = s}^t \big(\overline{A}_j\big)^{-1/2} A_j \big(\overline{A}_j\big)^{-1/2} \right)^{1/2} \bu \right\| > \sqrt{\z} \right)
        &\\
        \le 6^d \sup_{\|\bu\|=1} \p\left( \left\| \left( \sum\limits_{j = s}^t \big(\overline{A}_j\big)^{-1/2} A_j \big(\overline{A}_j\big)^{-1/2} \right)^{1/2} \bu \right\| > \frac{\sqrt{\z}}{2} \right)
        &.
    \end{align*}
    Moreover, in view of Assumption \ref{as:subexp}, for any unit vector $\bu$ it holds that
    \begin{align*}
        &
        \p\left( \left\| \left( \sum\limits_{j = s}^t \big(\overline{A}_j\big)^{-1/2} A_j \big(\overline{A}_j\big)^{-1/2} \right)^{1/2} \bu \right\| > \frac{\sqrt{\z}}{2} \right)
        \\&
        = \p\left( \exp\left\{ \frac1B \left( \bu^\top \sum\limits_{j = s}^t \big(\overline{A}_j\big)^{-1/2} A_j \big(\overline{A}_j\big)^{-1/2} \bu \right) \right\} > \exp\left\{\frac{\z}{4B} \right\} \right) 
        \\&
        \le e^{-\z / (4B)} \, \E \exp\left\{ \frac1B \left( \bu^\top \sum\limits_{j = s}^t \big(\overline{A}_j\big)^{-1/2} A_j \big(\overline{A}_j\big)^{-1/2} \bu \right) \right\}
        \\&
        \le e^{-\z / (4B)} \, \prod\limits_{j = s}^t \E \exp\left\{ \frac{\bu^\top \big(\overline{A}_j\big)^{-1/2} A_j \big(\overline{A}_j\big)^{-1/2} \bu}B \right\}
        \\&
        \leq 2^{t - s + 1} \, e^{-\z / (4B)}.
    \end{align*}
    This yields that
    \[
        \p\left( \left\| \sum\limits_{j = s}^t \big(\overline{A}_j\big)^{-1/2} A_j \big(\overline{A}_j\big)^{-1/2} \right\| > \z \right)
        \leq 2^{t - s + 1} \cdot 6^d \cdot e^{-\z / (4B)}.
    \]
    Therefore, for any \(\delta \in (0,1)\) with probability at least \(1-\delta\) it holds that
    \[
        \left\| \sum\limits_{j = s}^t \big(\overline{A}_j\big)^{-1/2} A_j \big(\overline{A}_j\big)^{-1/2} \right\|
        \le 4B \big( (t - s + 1) \log 2 + d \log6 + \log(1/\delta) \big).
    \]
    
\end{proof}

\begin{Lem}
    \label{lem:b_partial_sum_norm}
    Under Assumption~\ref{as:subexp}, for any \(1 \le s < t \le T\) and \(\delta \in (0,1)\), with probability at least \((1 - \delta)\) it holds that
    \[
        \left\| \sum\limits_{j = s}^{t - 1} \bb_j \right\|
        \le 4B \left(\|\overline{A}_t\|^{1/2} \vee \|\overline{A}_1\|^{1/2} \right) \left(d\log6 + (t-s)\frac{\log2}{2} + \log\frac{2}{\delta}\right)
    \]
\end{Lem}

\begin{proof}
    We split the sum of interest into pre-change and post-change segments
    \[
        \big(\overline{A}_t\big)^{-1/2} \sum\limits_{j = s}^{\tau} \bb_j
        \quad \text{and} \quad
        \big(\overline{A}_t\big)^{-1/2} \sum\limits_{j = \tau + 1}^{t - 1} \bb_j,
        \quad \text{respectively,}
    \]
    and study these two terms separately.
    Let us note that
    \[
        \overline{A}_k
        = \begin{cases}
            \overline{A}_1, \quad \text{if \(k \le \tau\),}\\
            \overline{A}_t, \quad \text{otherwise.}
        \end{cases}
    \]
    Then, due to the triangle inequality, it holds that
    \begin{align*}
        \left\| \sum\limits_{j = s}^{t - 1} \bb_j \right\|
        &
        \le \left\| \overline{A}_1 \right\|^{1/2}
        \cdot 
        \left\| \big(\overline{A}_1\big)^{-1/2} \sum\limits_{j = s}^{\tau} \bb_j \right\|
        + \left\| \overline{A}_t \right\|^{1/2}
        \cdot \left\| \big(\overline{A}_t\big)^{-1/2} \sum\limits_{j = \tau + 1}^{t - 1} \bb_j \right\|
        \\&
        = \left\| \overline{A}_1 \right\|^{1/2}
        \cdot \left\| \sum\limits_{j = s}^{\tau} \big(\overline{A}_j\big)^{-1/2}  \bb_j \right\|
        + \left\| \overline{A}_t \right\|^{1/2}
        \cdot \left\| \sum\limits_{j = \tau + 1}^{t - 1} \big(\overline{A}_j\big)^{-1/2} \bb_j \right\|.
    \end{align*}
    Let us elaborate on the norm of 
    \[
        \sum\limits_{j = s}^{\tau} \big(\overline{A}_j\big)^{-1/2} \bb_j.
    \]
    Applying the standard \(\eps\)-net argument (see \citeauthor{rigollet23}, \citeyear[proof of Theorem 1.19]{rigollet23}), we obtain that
    \[
        \p\left( \left\| \big(\overline{A}_1\big)^{-1/2} \sum\limits_{j = s}^{\tau} \bb_j \right\| > \z \right)
        \le 6^d \sup_{\|\bu\| = 1} \p\left( \bu^\top \big(\overline{A}_1\big)^{-1/2} \sum\limits_{j = s}^{\tau} \bb_j  > \frac{\z}{2} \right).
    \]
    Due to Assumption~\ref{as:subexp}, for any unit vector $u$ it holds that
    \begin{align*}
        \p\left( \bu^\top \big(\overline{A}_1\big)^{-1/2} \sum\limits_{j = s}^{\tau} \bb_j  > \frac{\z}{2} \right) 
        \le e^{-\z / (2B)} \prod_{j = s}^\tau \E \exp\left\{ \frac{\left|\bu^\top \left(\overline{A}_1\right)^{-1/2} \bb_j\right|}{B} \right\} 
        \le 2^{\tau - s + 1} \cdot e^{-\z / (2B)}.
    \end{align*}
    Therefore, for any \(\delta \in (0,1)\) with probability at least \((1 - \delta/2)\)
    \begin{align*}
        \left\| \big(\overline{A}_1\big)^{-1/2} \sum\limits_{j = s}^{\tau} \bb_j \right\|
        \le 2B \left( d\log6 + (\tau - s + 1) \log2 + \log\frac2\delta \right).
    \end{align*}
    Similarly, with probability at least \((1 - \delta/2)\)
    \begin{align*}
        \left\| \big(\overline{A}_t\big)^{-1/2} \sum\limits_{j = \tau + 1}^{t - 1} \bb_j \right\|  \le 2B \left(d\log6 + (t - \tau - 1)\log2 + \log\frac2\delta \right).
    \end{align*}
    Hence, due to the union bound, we obtain that
    \begin{align*}
        \left\| \sum\limits_{j = s}^{t - 1} \bb_j \right\|
        \le 4B \left(\|\overline{A}_t\|^{1/2} \vee \|\overline{A}_1\|^{1/2} \right) \left( d\log6 + (t - s) \frac{\log2}{2} + \log\frac{2}{\delta}\right)
    \end{align*}
    with probability at least $(1 - \delta)$.
    
\end{proof}

\begin{Lem}
    \label{lem:ew_prediction_high-probability_bound}
    Under Assumption~\ref{as:subexp}, for any \(1 \le s < t \le T\) and \(\delta \in (0,1/2)\) with probability at least \(1-2\delta\) it holds that 
    \begin{align*}
        \left\|A_t^{1/2} \widehat{\btheta}_{s:t-1}(\eta)\right\|
        \le 8B^{3/2} \big( d\log6 + \log(2/\delta) \big)^{1/2}
        & \cdot \|\overline{A}_t\|^{1/2} \sqrt{\|\overline{A}_t\| \vee \|\overline A_1\|}
        \\&
        \cdot \left(\frac{\log2}{2\gamma} \vee \frac{d\log6 + \log(2\sqrt{2}/\delta)}{\gamma + \lambda/\eta}\right).
    \end{align*}
\end{Lem}

\begin{proof}
    We start with the inequality
    \begin{align*}
        \left\| A_t^{1/2} \widehat{\btheta}_{s:t-1}(\eta) \right\|^2 
        \le \left\| \big(\overline{A}_t\big)^{-1/2} A_t \big(\overline{A}_t\big)^{-1/2} \right\|
        \cdot \left\| \overline{A}_t^{1/2} \widehat{\btheta}_{s:t-1}(\eta) \right\|^2.
    \end{align*}
    According to Lemma \ref{lem:a_sum_norm}, 
    for any \(\delta \in (0,1)\), with probability at least \(1-\delta\) the first factor in the right-hand side does not exceed
    \begin{equation}
        \label{eq:At_meanA_norm}
        \left\| \big(\overline{A}_t\big)^{-1/2} A_t \big(\overline{A}_t\big)^{-1/2} \right\|
        \le 4B \big( d\log6 + \log(2/\delta) \big).
    \end{equation}
    The rest of the proof is devoted to analysis of $\overline{A}_t^{1/2} \widehat{\btheta}_{s:t-1}(\eta)$.
    Recalling the definition of the exponentially weighed average forecaster (see eq.~\ref{eq:ew_st}) and using the fact that \(A_j \succeq \gamma I_d\) for all $j \in \{1, \dots, T\}$ (see eq.~\eqref{eq:ellt}), we obtain that
    \[
        \left\|\overline{A}_t^{1/2}\widehat{\btheta}_{s:t-1}(\eta)\right\|
        \le \left\| \overline{A}_t^{1/2} \left(\sum\limits_{j = s}^{t - 1} A_j + \frac{\lambda}{\eta} I_d \right)^{-1} \right\| 
        \, \left\| \sum\limits_{j = s}^{t - 1} \bb_j \right\|
        \le \frac{\|\overline{A}_t\|^{1/2}}{(t - s) \gamma + \lambda / \eta} \left\| \sum\limits_{j = s}^{t - 1} \bb_j \right\|.
    \]
    According to Lemma~\ref{lem:b_partial_sum_norm}, with probability at least $(1 - \delta)$ it holds that
    \begin{align}
        \label{eq:meanA_ew_norm}
        \left\|\overline{A}_t^{1/2}\widehat{\btheta}_{s:t-1}(\eta)\right\|
        &\notag
        \le \frac{4B\|\overline{A}_t\|}{\gamma(t-s) + \lambda/\eta}\left(1 \vee \sqrt{\frac{\|\overline{A}_1\|}{\|\overline{A}_t\|}}\;\right) \cdot \left(d\log6 + (t-s)\frac{\log2}{2} + \log\frac{2}{\delta}\right)
        \\&
        \le 4B \|\overline{A}_t\|^{1/2} \sqrt{\|\overline{A}_t\| \vee \|\overline{A}_1\|}  \cdot \left( \frac{\log2}{2\gamma} \vee \frac{d\log6 + \log(2\sqrt{2}/\delta)}{\gamma + \lambda/\eta} \right).
    \end{align}
    Combining the inequalities \eqref{eq:At_meanA_norm} and \eqref{eq:meanA_ew_norm} and using the union bound, we conclude that with probability at least \( (1 - 2\delta) \) 
    \begin{align*}
        \left\|A_t^{1/2} \widehat{\btheta}_{s:t-1}(\eta)\right\| \le 2\sqrt{B} \big(d\log6 + \log(2/\delta)\big)^{1/2}
        &
        \cdot 4B\|\overline{A}_t\|^{1/2} \sqrt{\|\overline{A}_t\| \vee \|A_1\|}
        \\&
        \cdot \left(\frac{\log2}{2\gamma} \vee \frac{d\log6 + \log(2\sqrt{2} / \delta)}{\gamma + \lambda / \eta}\right).
    \end{align*}
\end{proof}

\section{Regret bounds}

In this section, we present regret bounds for the exponentially weighted average and the fixed share forecasters. The results are valid for arbitrary $A_t \succeq O_d$ and $\bb_t \in \R^d$ (not necessarily related to $\nabla \bPsi(\bX_t)$ and $\Delta \bPsi(\bX_t)$).

\subsection{Regret Bounds for Exponentially Weighted Average Forecaster}
\label{sec:ew_appendix}

We start with auxiliary results on the regret of the exponentially weighted average forecaster with Gaussian prior. They will be helpful in the analysis of the fixed share algorithm. Our approach relies on the mixability argument of \cite{vovk01} and has an important advantage over the online convex optimization techniques \citep{hazan07, hoeven18}. In \citealp{hazan07, hoeven18}, the authors require the parameter space to be bounded in order to exploit exp-concavity of the quadratic loss on a compact set. In addition, \citet*{hazan07} compare the cumulative loss of the exponential weighting with the best expert over a bounded region. In contrast, we are interested in the case of unbounded domain.  

We proceed with regret analysis of the exponentially weighted forecaster. It is easy to observe that
\[
    \ell_t(\btheta)
    = \frac12 \btheta^\top A_t \btheta - \bb_t^\top \btheta
    = \frac12 \left\|A_t^{1/2} \left(\btheta - A_t^\dag \bb_t \right) \right\|^2 - \btheta^\top \bb_t^\perp,
\]
where $\bb_t^\perp$ is the projection of $\bb_t$ onto $\Im(A_t)^\perp$.
In other words, $\ell_t$ consists of mixable and linear parts. It is known that worst-case regret bounds for mixable and linear losses are different. If $\bb_1^\perp, \dots, \bb_T^\perp$ are negligible, we can expect the regret
\[
    R_{1:T}^{EW} = \sum\limits_{t = 1}^T \ell_t(\widehat \btheta_t^{EW}) - \inf\limits_{\btheta \in \R^d} \sum\limits_{t = 1}^T \ell_t(\btheta)
\]
to be of order $O(\log T)$. Otherwise, $R_{1:T}^{EW}$ will be as large as $O(\sqrt{T})$ provided that the parameters are tuned properly. Our analysis based on the following technical lemma captures both favourable and worst-case scenarios.

\begin{Lem}
    \label{lem:ew_mixability}
    Consider a quadratic loss function
    \[
        \ell(\bups) = \frac12 \bups^\top A \bups - \bb^\top \bups
    \]
    with arbitrary symmetric positive semidefinite matrix $A \in \R^{d \times d}$ and $\bb \in \R^d$. Let $\btheta \sim \cN(\bmu, \Omega^{-1})$ be a Gaussian random vector in $\R^d$. Then, for any $\eta > 0$ satisfying the inequality
    \begin{equation}
        \label{eq:ew_eta_condition_simplified}
        \left\| A^{1/2} \left(\bmu - A^\dag \bb \right) \right\|^2 \leq \frac1{2\eta},
    \end{equation}
    it holds that
    \[
        \exp\left\{-\eta \ell(\bmu) + \eta^2 \|\Omega^{-1/2} \, \bb^\perp\|^2 \right\}
        \geq \E \exp\big\{ -\eta \ell (\btheta) \big\},
    \]
    where $\bb^\perp$ is the projection of $\bb$ onto the orthogonal complement of $\Im(A)$.
\end{Lem}

We postpone the proof of Lemma~\ref{lem:ew_mixability} to Section~\ref{sec:lem_ew_mixability_proof}. 
Let us recall that the exponentially weighted average $\widehat\btheta{}_t^{EW}$ is the mean of the posterior measure~\eqref{eq:posterior}, which is Gaussian in our case. Using Lemma~\ref{lem:ew_mixability} and the argument of \citet{vovk01}, we derive the following bound on $R_{1:T}^{EW}$.

\begin{Th}
    \label{th:ew_regret}
    Assume that the parameters $\lambda > 0$ and $\eta_1 \geq \dots \geq \eta_T > 0$ are chosen in such a way that
    \begin{equation}
        \label{eq:ew_condition}
        \left\| A_t^{1/2} \left(\widehat\btheta_{1:t - 1}(\eta_t) - A_t^\dag \bb_t \right) \right\|^2 \leq \frac1{2\eta_t}
        \quad \text{for all $t \in \{1, \dots, T\}$,}
    \end{equation}
    where $\widehat \btheta_{1:t-1}$ is defined in~\eqref{eq:ew_st}. For any $t \in \{1, \dots, T\}$, let $\bb_t^\perp$ stand for the projection of $\bb_t$ onto $\Im(A_t)^\perp$ and denote
    \[
        \Omega_t = \lambda I_d + \eta_t \sum\limits_{j = 1}^{t - 1} A_j.
    \]
    Then the regret $R_{1:T}^{EW}$ of the exponentially weighted average forecaster satisfies the inequality
    \begin{align*}
        R_{1:T}^{EW} 
        &
        \leq \frac{\lambda \|\btheta_{1:T}^\circ\|^2}{2 \eta_T} + \frac1{2\eta_T} \log \det\left( I_d + \frac{\eta_T}{\lambda} \sum\limits_{t = 1}^T A_t \right) + \sum\limits_{t = 1}^T \eta_t \| \Omega_t^{-1/2} \, \bb_t^\perp \|^2
        \\&
        \leq \frac{\lambda \|\btheta_{1:T}^\circ\|^2}{2 \eta_T} + \frac1{2\eta_T} \log \det\left( I_d + \frac{\eta_T}{\lambda} \sum\limits_{t = 1}^T A_t \right) + \sum\limits_{t = 1}^T \frac{\eta_t \|\bb_t^\perp \|^2}{\lambda},
    \end{align*}
    where $\btheta_{1:T}^\circ \in \argmin\limits_{\btheta \in \R^d} L_{1:T}(\btheta)$.
\end{Th}

The proof of Theorem~\ref{th:ew_regret} is moved to Section~\ref{sec:th_ew_regret_proof}. Despite its simplicity, the result of Theorem~\ref{th:ew_regret} does not follow from the existing literature. In \citep{vovk01}, the author studies a more specific loss $\ell_t(\btheta) = (\bx_t^\top \btheta - y_t)^2$, which is a particular case of~\eqref{eq:loss}. Many other papers restrict their attention on bounded domains, such as \citep{hazan07, erven21}. In the case of quadratic loss with $\bb_t^\perp = \bzero$, a finite domain size yields exp-concavity of the loss $\ell_t(\btheta)$, which is also necessary in \citep{hazan07} for the analysis. In \citep[Theorem 5]{hoeven18}, the authors show that, under minimal assumptions, the excess loss of the exponentially weighted average forecaster with the fixed learning rate $\eta$ satisfies the inequality
\[
    \widehat L_{1:T}^{EW} - L_{1:T}(\btheta)
    \leq \frac1{2 \eta \lambda} \left\|\widehat\btheta_{1:t - 1}(\eta_t) - \btheta \right\|^2 + \frac{\eta}2 \sum\limits_{t = 1}^T \big(A_t \widehat\btheta_{1:t - 1}(\eta_t) - \bb_t\big)^\top \Omega_{t + 1}^{-1} \big(A_t \widehat\btheta_{1:t - 1}(\eta_t) - \bb_t\big),
\]
where $\Omega_1, \dots, \Omega_{T + 1}$ are the same as in Theorem~\ref{th:ew_regret}. However, \citet*{hoeven18} impose additional assumptions (in particular, bounded domain and exp-concavity of the loss) to derive explicit dependence on $T$ from the general regret bound. 
As we announced, if $\bb_t^\perp = \bzero$ for all $t$ from $1$ to $T$, then the choice $\eta_T \gtrsim 1$, $\lambda = 1$ leads to logarithmic regret bounds. However, the relation $\bb_t^\perp \neq \bzero$ does not necessarily yield that $R^{EW}_{1:T}$ is of order $\sqrt{T}$. In favourable situations (for instance, if $A_1, \dots, A_T$ are i.i.d. random matrices with non-degenerate expectation $\overline A = \E A_1 \succ O_d$ and $\bb_1, \dots, \bb_T$ are bounded almost surely) we can still have $\| \Omega_t^{-1/2} \, \bb_t^\perp \|^2 \lesssim 1/t$ and $R^{EW}_{1:T} \lesssim \log T$.

\subsection{Regret Bounds for the Fixed Share Forecaster}

We move to main theoretical results of the paper, the regret bound for Algorithm~\ref{alg:fs} presented in Theorem~\ref{th:fs_regret}. The main ingredient of its proof is the following technical result, establishing a mixability-type property of the fixed share forecaster.

\begin{Lem}
    \label{lem:fs_mixability}
    Let us fix any $t \in \{1, \dots, T\}$. Assume that $\bb_t \in \Im(A_t)$ for all $t \in \{1, \dots, T\}$ and that the parameter $\eta > 0$ satisfies the inequality
    \begin{equation}
        \label{eq:exp-concavity_requirement}
        \left\| A_t^{1/2} \widehat\btheta_{s:t - 1}(\eta) \right\|^2 \vee \left\| (A_t^\dag)^{-1/2} \bb_t \right\|^2
        \leq \frac1{4 \eta}
        \quad \text{for all $1 \leq s \leq t - 1$,}
    \end{equation}
    where $\widehat\btheta_{s:t - 1}(\eta)$ is defined in~\eqref{eq:ew_st}. Then it holds that
    \[
        \exp\left\{-\eta \ell_t\big(\widetilde\btheta_t(\eta) \big) \right\}
        \geq \frac{V_t(\eta)}{V_{t - 1}(\eta)}
        = \frac{W_t(\eta)}{W_{t - 1}(\eta)},
    \]
    where,  $\widetilde \btheta_t(\eta)$ is defined in~\eqref{eq:exponential_weights_compound} and~\eqref{eq:widetilde_theta}.
\end{Lem}
We provide the proof of Lemma~\ref{lem:fs_mixability} in Section~\ref{sec:lem_fs_mixability_proof}. With this lemma at hand, the proof of the next regret bound is almost straightforward.

\begin{Th}
    \label{th:fs_regret}
    Assume that the outcomes $\{(A_t, \bb_t) : 1 \leq t \leq T\}$ and the learning rates $\eta_1 \geq \eta_2 \geq \dots \geq \eta_T$ are such that $\bb_t \in \Im(A_t)$ for all $t \in \{1, \dots, T\}$ and
    \begin{equation}
        \label{eq:eta_t_requirement}
        \left\| A_t^{1/2} \widehat\btheta_{s:t - 1}(\eta_t) \right\|^2 \vee \left\| (A_t^\dag)^{-1/2} \bb_t \right\|^2
        \leq \frac1{4 \eta_t}
        \quad \text{for all $1 \leq s < t \leq T$.}
    \end{equation}
    Fix an arbitrary $m \in \{1, \dots, T - 1\}$ and $0 = \tau_0 < \tau_1 < \ldots < \tau_m = T$.
    Then the adaptive regret of the fixed share forecaster with respect to the best compound expert of size $(m - 1)$ switching at $\tau_1, \tau_2, \dots, \tau_{m - 1}$ does not exceed
    \begin{align*}
        \cR_{1:T}^{FS}
        &
        = \widehat L_{1:T}^{FS} - \sum\limits_{k = 0}^{m - 1} L_{\tau_k + 1 : \tau_{k + 1}}(\btheta_{\tau_k + 1 : \tau_{k + 1}}^\circ)
        \\&
        \leq \frac{(m - 1) \log(1 / \alpha)}{\eta_T} + \frac{(T - m) \log\big(1 / (1 - \alpha) \big)}{\eta_T}
        \\&\quad
        + \sum\limits_{k = 0}^{m - 1} \left[ \frac{\lambda \|\btheta_{\tau_k + 1 : \tau_{k + 1}}^\circ\|^2}{2 \eta_T} + \frac1{2 \eta_T} \log \det\left( I_d + \frac{\eta_T}\lambda \sum\limits_{j = \tau_k + 1}^{\tau_{k + 1}} A_j \right) \right],
    \end{align*}
    where, for each $k \in \{0, \dots, m - 1\}$,
    \[
        \btheta_{\tau_k + 1 : \tau_{k + 1}}^\circ \in \argmin\limits_{\btheta} L_{\tau_k + 1 : \tau_{k + 1}}(\btheta).
    \]
\end{Th}

The proof of Theorem~\ref{th:fs_regret} is moved to Section~\ref{sec:th_fs_regret_proof}. Let us note that the fixed share forecaster enjoys the logarithmic dynamic regret $\cR_{1:T}^{FS} = \cO(m \log(T / m))$ provided that $\eta_1 \geq \ldots \geq \eta_T \gtrsim 1$, $\lambda = 1$ and $\alpha = m/T$. This agrees with the result of \cite{herbster98} (Corollary 1) who derived a similar bound for the case of a finite number of experts and mixable losses and with Theorem 3.1 of \cite{hazan07b} implying a $\cO(m \log T)$ bound on the shifting regret of the follow-the-leading-history algorithm (FLH1) in the case of exp-concave losses and infinite number of experts.

Besides \citep{hazan07b}, the problem of tracking the best out of infinite number of experts was studied in \citep{herbster01, cavallanti07, hazan07b, kozat08, hazan09}. The paper \citep{cavallanti07} is not really relevant to the present work, because \citet{cavallanti07}
consider binary loss functions. The switching regret bound in \cite[Theorem 1]{kozat08}, where the authors study the online linear regression problem, coincides with the one in Theorem~\ref{th:fs_regret} if we put $A_t = \bx_t \bx_t^\top$ and $\bb_t = y_t \bx_t$, $t \in \{1, \dots, T\}$. However, the conditions of \cite[Theorem 1]{kozat08} are milder than ours. It is not surprising, such a phenomenon was already discussed in \citep[Section 3.4]{vovk01}. The reason is that in the online linear regression the learner has an access to the feature vector $\bx_t$ before he makes the prediction. Hence, he can use the nonlinear Vovk-Azoury-Warmuth forecaster \citep{vovk01, azoury01} getting an advantage over standard exponential weighting with Gaussian prior. In contrast to \citep{cavallanti07, kozat08}, \citet{herbster01} consider general loss functions. Their approach to regret analysis is based on the properties of Bregman divergences and generalized projections. Unfortunately, the oracle inequality in \citep{herbster01} for the cumulative loss of the learner is not sharp and can lead to extremely poor guarantees in the worst-case scenario. Finally, we have already compared our dynamic regret bounds with the guarantees of \cite{hazan07b, hazan09} in the previous paragraph.
In conclusion, we would like to emphasize that
their analysis relies on exp-concavity of the loss functions. Clearly, the quadratic loss~\eqref{eq:loss} is not exp-concave on $\R^d$.

\subsection{Proofs of the regret bounds}
\label{sec:proofs}

\subsubsection{Proof of Lemma~\ref{lem:ew_mixability}}
\label{sec:lem_ew_mixability_proof}

The proof of Lemma~\ref{lem:ew_mixability} relies on the following auxiliary result derived
in Appendix~\ref{sec:lem_ew_mixability_aux_proof} below.

\begin{Lem}
    \label{lem:ew_mixability_aux}
    Consider a quadratic loss function
    \[
        \ell(\bups) = \frac12 \bups^\top A \bups - \bb^\top \bups
    \]
    with arbitrary symmetric positive semidefinite matrix $A \in \R^{d \times d}$ and $\bb \in \R^d$. Let $\btheta \sim \cN(\bmu, \Omega^{-1})$ be a Gaussian random vector in $\R^d$. Then, for any $\eta > 0$, satisfying the inequality
    \begin{equation}
        \label{eq:ew_eta_condition}
        \eta \left\| (\Omega + \eta A)^{-1/2} (A \bmu - \bb) \right\|^2
        \leq \left\| A^{1/2} (\Omega + \eta A)^{-1} A^{1/2} \right\|,
    \end{equation}
    it holds that
    \[
        e^{-\eta \ell(\bmu)}
        \geq \E e^{-\eta \ell (\btheta)}. 
    \]
\end{Lem}

Let us represent $\bb = \bb^{\parallel} + \bb^{\perp}$, where $\bb^{\parallel} \in \Im(A)$ and $\bb^{\perp} \in \Im(A)^{\perp}$ and apply Lemma~\ref{lem:ew_mixability_aux} to the loss
\[
    \widetilde \ell(\bups) = \frac12 \bups^\top A \bups - \bups^\top \bb^{\parallel}.
\]
Note that the condition~\eqref{eq:ew_eta_condition_simplified} yields that
\begin{align*}
    2 \eta \left\| (\Omega + \eta A)^{-1/2} (A \bmu - \bb^\parallel) \right\|^2
    &
    = 2 \eta \left\| (\Omega + \eta A)^{-1/2} A (\bmu - A^\dag \bb^\parallel) \right\|^2
    \\&
    \leq 2 \eta \left\| A^{1/2} (\Omega + \eta A)^{-1} A^{1/2} \right\| \left\| A^{1/2} (\bmu - A^\dag \bb^\parallel) \right\|^2
    \\&
    \leq \left\| A^{1/2} (\Omega + \eta A)^{-1} A^{1/2} \right\|.
\end{align*}
Hence, the assumptions of Lemma~\ref{lem:ew_mixability_aux} are satisfied, and we obtain that
\[
    e^{-2\eta \widetilde \ell(\bmu)}
    \geq \E e^{-2\eta \widetilde \ell (\btheta)}.
\]
Then, due to the Cauchy-Schwarz inequality, we have
\begin{align*}
    \E \exp\left\{-\eta \ell (\btheta) \right\}
    &
    = \E \exp\left\{-\eta \widetilde \ell (\btheta) + \eta \, \btheta^\top \bb^\perp \right\}
    \\&
    \leq \sqrt{\E \exp\left\{-2\eta \, \widetilde \ell (\btheta) \right\}} \cdot \sqrt{\E \exp\left\{ 2 \eta \, \btheta^\top \bb^\perp \right\}}
    \\&
    \leq \sqrt{\exp\left\{-2\eta \, \widetilde \ell (\bmu) \right\}} \cdot \sqrt{\exp\left\{2\eta \bmu^\top \bb^\perp + 2 \eta^2 \| \Omega^{-1/2} \, \bb^\perp \|^2 \right\}}
    \\&
    = \exp \left\{-\eta \widetilde \ell (\bmu) + \eta \bmu^\top \bb^\perp + \eta^2 \| \Omega^{-1/2} \, \bb^\perp \|^2 \right\}.
\end{align*}
Taking into account that
\[
    \widetilde \ell (\bmu) - \bmu^\top \bb^\perp
    = \frac12 \bmu^\top A \bmu - \bmu^\top \bb^{\parallel} - \bmu^\top \bb^\perp
    = \frac12 \bmu^\top A \bmu - \bmu^\top \bb
    = \ell (\bmu),
\]
we obtain the desired bound:
\begin{align*}
    \E \exp\big\{-\eta \ell (\btheta) \big\}
    &
    \leq \exp \left\{-\eta \widetilde \ell (\bmu) + \eta \bmu^\top \bb^\perp + \eta^2 \| \Omega^{-1/2} \, \bb^\perp \|^2 \right\}
    \\&
    = \exp \left\{-\eta \ell (\bmu) + \eta^2 \| \Omega^{-1/2} \, \bb^\perp \|^2 \right\}.
\end{align*}
\myendproof

\subsubsection{Proof of Theorem~\ref{th:ew_regret}}
\label{sec:th_ew_regret_proof}

Note that the exponentially weighted average $\widehat\btheta{}^{EW}_t = \widehat\btheta_{1:{t-1}}(\eta_t)$ is the mean of the posterior distribution~\eqref{eq:posterior},
which coincides with (see the proof of Lemma~\ref{lem:ew}) 
\[
    \cN\left( \left( \sum\limits_{j = 1}^{t - 1} A_j + \frac{\lambda}{\eta_t} I_d \right)^{-1} \sum\limits_{j = 1}^{t - 1} \bb_j, \frac1{\eta_t} \left( \sum\limits_{j = 1}^{t - 1} A_j + \frac{\lambda}{\eta_t} I_d \right)^{-1} \right).
\]
Then the requirement~\eqref{eq:ew_condition} yields that the conditions of Lemma~\ref{lem:ew_mixability} are fulfilled with
\[
    \eta = \eta_t,
    \quad 
    \bmu = \widehat\btheta{}^{EW}_t,
    \quad
    A = A_t,
    \quad \text{and} \quad
    \Omega = \Omega_t = \lambda I_d + \eta_t \sum\limits_{j = 1}^{t - 1} A_j.
\]
Applying this lemma, we obtain that
\begin{align*}
    \exp\left\{ -\eta_t \ell_t(\widehat\btheta{}_t^{EW}) + \eta_t^2 \| \Omega_t^{-1/2} \, \bb_t^\perp \|^2 \right\}
    &
    \geq \frac1{Z_{1:{t - 1}}(\eta_t)} \int\limits_{\R^d} e^{-\eta_t \ell_t(\btheta) - \eta_t L_{1:t - 1}(\btheta)} \pi(\btheta) \, \dd \btheta
    \\&
    = \frac1{Z_{1:{t - 1}}(\eta_t)} \int\limits_{\R^d} e^{-\eta_t L_{1:t}(\btheta)} \pi(\btheta) \, \dd \btheta
    = \frac{Z_{1:t}(\eta_t)}{Z_{1:{t - 1}}(\eta_t)}.
\end{align*}
Thus, for any $t \in \{1, \dots, T\}$, it holds that
\[
    \ell_t(\widehat\btheta{}_t^{EW})
    \leq -\frac1{\eta_t} \log Z_{1:t}(\eta_t) + \frac1{\eta_t} \log Z_{1:t - 1}(\eta_t) + \eta_t \| \Omega_t^{-1/2} \, \bb_t^\perp \|^2
\]
This inequality immediately implies an upper bound on the cumulative loss $\widehat L_{1:T}^{EW}$ of the exponentially weighted average forecaster:
\begin{align*}
    \widehat L_{1:T}^{EW}
    = \sum\limits_{t = 1}^T \ell_t(\widehat\btheta{}_t^{EW})
    &
    \leq - \sum\limits_{t = 1}^T \frac1{\eta_t} \log Z_{1:t}(\eta_t) + \sum\limits_{t = 1}^T \frac1{\eta_t} \log Z_{1:t - 1}(\eta_t) + \sum\limits_{t = 1}^T \eta_t \| \Omega_t^{-1/2} \, \bb_t^\perp \|^2
    \\&
    = -\frac1{\eta_T} \log Z_{1:T}(\eta_T) + \frac1{\eta_1} \log Z_{1:0}(\eta_t) + \sum\limits_{t = 1}^T \eta_t \| \Omega_t^{-1/2} \, \bb_t^\perp \|^2
    \\&\quad
    + \sum\limits_{t = 1}^{T - 1} \left( \frac1{\eta_{t + 1}} \log Z_{1:t}(\eta_{t + 1}) - \frac1{\eta_t} \log Z_{1:t}(\eta_t) \right).
\end{align*}
Since $\eta_1 \geq \eta_2 \geq \dots \geq \eta_T > 0$, due to the H\"older inequality, we have that
\begin{align*}
    \frac1{\eta_{t + 1}} \log Z_{1:t}(\eta_{t + 1})
    &
    = \frac1{\eta_{t + 1}} \log \int\limits_{\R^d} e^{-\eta_{t + 1} L_{1:t}(\btheta)} \pi(\btheta) \, \dd \btheta
    \\&
    \leq \frac1{\eta_t} \log \int\limits_{\R^d} e^{-\eta_t L_{1:t}(\btheta)} \pi(\btheta) \, \dd \btheta
    = \frac1{\eta_t} \log Z_{1:t}(\eta_t),
\end{align*}
Hence, it holds that
\[
    \widehat L_{1:T}^{EW}
    \leq -\frac1{\eta_T} \log Z_{1:T}(\eta_T) + \frac1{\eta_1} \log Z_{1:0}(\eta_t) + \sum\limits_{t = 1}^T \eta_t \| \Omega_t^{-1/2} \, \bb_t^\perp \|^2. 
\] 
Taking into account the equality
\[
    Z_{1:0}(\eta_1) 
    = \int\limits_{\R^d} e^{-\eta_1 L_{1:0}(\btheta)} \pi(\btheta) \, \dd \btheta
    = \int\limits_{\R^d} \pi(\btheta) \, \dd \btheta
    = 1,
\]
we obtain that
\[
    \widehat L_{1:T}^{EW}
    \leq -\frac1{\eta_T} \log Z_{1:T}(\eta_T) + \sum\limits_{t = 1}^T \eta_t \| \Omega_t^{-1/2} \, \bb_t^\perp \|^2.
\]
The assertion of the theorem follows from the next lemma.
\begin{Lem}
    \label{lem:z_1t_bound}
    For any positive numbers $\eta$ and $\lambda$, it holds that
    \[
        -\frac1{\eta} \log Z_{1:T}(\eta)
        \leq L_{1:T}(\btheta_{1:T}^\circ)
        + \frac{\lambda \|\btheta_{1:T}^\circ\|^2}{2 \eta}
        + \frac1{2\eta} \log \det\left( I_d + \frac{\eta}{\lambda} \sum\limits_{t = 1}^T A_t \right),
    \]
    where $\btheta_{1:T}^\circ \in \argmin\limits_{\btheta \in \R^d} L_{1:T}(\btheta)$.
\end{Lem}
The proof of Lemma~\ref{lem:z_1t_bound} is deferred to Appendix~\ref{sec:lem_z_1t_bound_proof}.
Substituting $\eta$ with $\eta_T$, we get an upper bound on the regret of the exponentially weighted average forecaster:
\begin{align*}
    R_{1:T}^{EW} 
    &
    = \widehat L_{1:T}^{EW} - L_{1:T}(\btheta_{1:T}^\circ)
    \\&
    \leq \frac{\lambda \|\btheta_{1:T}^\circ\|^2}{2 \eta_T} + \frac1{2\eta_T} \log \det\left( I_d + \frac{\eta_T}{\lambda} \sum\limits_{t = 1}^T A_t \right) + \sum\limits_{t = 1}^T \eta_t \| \Omega_t^{-1/2} \, \bb_t^\perp \|^2.
\end{align*}
The inequality
\begin{align*}
    &
    \frac{\lambda \|\btheta_{1:T}^\circ\|^2}{2 \eta_T} + \frac1{2\eta_T} \log \det\left( I_d + \frac{\eta_T}{\lambda} \sum\limits_{t = 1}^T A_t \right) + \sum\limits_{t = 1}^T \eta_t \| \Omega_t^{-1/2} \, \bb_t^\perp \|^2
    \\&
    \leq \frac{\lambda \|\btheta_{1:T}^\circ\|^2}{2 \eta_T} + \frac1{2\eta_T} \log \det\left( I_d + \frac{\eta_T}{\lambda} \sum\limits_{t = 1}^T A_t \right) + \sum\limits_{t = 1}^T \frac{\eta_t \|\bb_t^\perp \|^2}{\lambda}
\end{align*}
is straightforward since
\[
    \Omega_t
    = \lambda I_d + \eta_t \sum\limits_{j = 1}^{t - 1} A_j
    \succeq \lambda I_d.
\]
\myendproof

\subsubsection{Proof of Lemma~\ref{lem:fs_mixability}}
\label{sec:lem_fs_mixability_proof}

First, note that
\[
    \exp\left\{-\eta \ell_1 \big(\widetilde\btheta_1(\eta) \big) \right\}
    = e^{-\eta \ell_1(\bzero)}
    = 1
    \geq \frac{W_1(\eta)}{W_0(\eta)}.
\]
Here we took into account that $W_1(\eta) \leq 1$ and $W_0(\eta) = 1$ for all $\eta > 0$. It remains to consider the case $t \geq 2$.
In view of~\eqref{eq:v} and~\eqref{eq:widetilde_theta}, the prediction $\widetilde \btheta_t(\eta)$ can be considered as a convex combination of $\widehat \btheta_{1 : t - 1}(\eta), \dots$, $\widehat \btheta_{t - 1 : t - 1}(\eta)$, and $\bzero$:
\begin{align*}
    \widetilde\btheta_{t}(\eta)
    &
    = \alpha \, \bzero
    + \frac{(1 - \alpha)^{t - 1} Z_{1 : t - 1}(\eta) \widehat\btheta_{1 : t - 1}(\eta)}{V_{t - 1}(\eta)}
    \\&\quad
    + \frac{\alpha (1 - \alpha)}{V_{t - 1}(\eta)} \sum\limits_{s = 0}^{t - 2} (1 - \alpha)^s \, V_{t - 2 - s}(\eta) Z_{t - 1 - s : t - 1}(\eta) \widehat\btheta_{t - 1 - s : t - 1}(\eta),
\end{align*}
where
\[
    V_{t - 1}(\eta) = (1 - \alpha)^{t - 2} Z_{1 : t - 1}(\eta) + \alpha \sum\limits_{s = 0}^{t - 2} (1 - \alpha)^s \, V_{t - 2 - s}(\eta) Z_{t - 1 - s : t - 1}(\eta). 
\]
The requirement~\eqref{eq:exp-concavity_requirement} ensures that
\begin{equation}
    \label{eq:boundedness}
    \left\|A_t^{1/2} ( \btheta - A_t^{\dag} \bb_t) \right\|^2
    \leq 2 \left\|A_t^{1/2} \btheta \right\|^2 + 2 \left\| (A_t^\dag)^{1/2} \bb_t \right\|^2
    \leq \frac1{2 \eta}
\end{equation}
for all $\btheta \in \{\bzero, \widehat \btheta_{1 : t - 1}(\eta), \dots, \widehat \btheta_{t - 1 : t - 1}(\eta)\}$. This yields that $\ell_t(\btheta)$ is $\eta$-exp-concave on the convex hull of $\bzero, \widehat \btheta_{1 : t - 1}(\eta), \dots, \widehat \btheta_{t - 1 : t - 1}(\eta)$, and, hence, it holds that
\begin{align*}
    &
    \exp\left\{-\eta \ell_t\big(\widetilde\btheta_t(\eta) \big) \right\}
    \\&
    \geq \alpha \exp\left\{-\eta \ell_t(\bzero) \right\} + \frac{(1 - \alpha)^{t - 1} Z_{1 : t - 1}(\eta)}{V_{t - 1}(\eta)} \exp\left\{-\eta \ell_t \big(\widehat\btheta_{1 : t - 1}(\eta) \big)\right\}
    \\&\quad
    + \frac{\alpha (1 - \alpha)}{V_{t - 1}(\eta)} \sum\limits_{s = 0}^{t - 2} (1 - \alpha)^s V_{t - 2 - s}(\eta) Z_{t - 1 - s : t - 1}(\eta) \exp\left\{-\eta \ell_t \big( \widehat\btheta_{t - 1 - s : t - 1}(\eta) \big) \right\}.
\end{align*}
Moreover, the inequality~\eqref{eq:boundedness} implies that the conditions of Lemma~\ref{lem:ew_mixability} are fulfilled. Applying this lemma and taking into account the inequalities
\[
    \exp\big\{-\eta \ell_t(\bzero) \big\}
    = 1
    \geq \int\limits_{\R^d} e^{-\eta \ell_t(\btheta)} \pi(\btheta) \, \dd \btheta
    = Z_{t:t}(\eta)
    \quad \text{and} \quad
    \lambda I_d + \eta \sum\limits_{j = s}^{t - 1} A_j \succeq \lambda I_d,
\]
we obtain that
\begin{align*}
    \exp\left\{-\eta \ell_t\big(\widetilde\btheta_t(\eta)\big)\right\}
    &
    \geq \alpha Z_{t : t}(\eta) + \frac{(1 - \alpha)^{t - 1} Z_{1 : t - 1}(\eta)}{V_{t - 1}(\eta)} \cdot \frac{Z_{1 : t}(\eta)}{Z_{1 : t - 1}(\eta)}
    \\&\quad
    + \frac{\alpha (1 - \alpha)}{V_{t - 1}(\eta)} \sum\limits_{s = 0}^{t - 2} (1 - \alpha)^s V_{t - 2 - s}(\eta) Z_{t - 1 - s : t - 1}(\eta) \cdot \frac{Z_{t - 1 - s : t}(\eta)}{Z_{t - 1 - s : t - 1}(\eta)}
    \\&
    = \frac{\alpha V_{t - 1}(\eta) Z_{t : t}(\eta)}{V_{t - 1}(\eta)} + \frac{(1 - \alpha)^{t - 1} Z_{1 : t}(\eta)}{V_{t - 1}(\eta)}
    \\&\quad
    + \frac{\alpha (1 - \alpha)}{V_{t - 1}(\eta)} \sum\limits_{s = 0}^{t - 2} (1 - \alpha)^s V_{t - 2 - s}(\eta) Z_{t - 1 - s : t}(\eta).
\end{align*}
According to Lemma~\ref{lem:v}, it holds that
\begin{align*}
    \exp\left\{-\eta \ell_t\big(\widetilde\btheta_t(\eta) \big) \right\}
    &
    \geq \frac{(1 - \alpha)^{t - 1} Z_{1 : t}(\eta)}{V_{t - 1}(\eta)} + \frac{\alpha}{V_{t - 1}(\eta)} \sum\limits_{s = 0}^{t - 1} (1 - \alpha)^s V_{t - 1 - s}(\eta) Z_{(t - s) : t}(\eta)
    \\&
    = \frac{V_t(\eta)}{V_{t - 1}(\eta)}.
\end{align*}
Finally, the assertion of the lemma follows from~\eqref{eq:w_v_equality}:
\[
    \exp\left\{-\eta \ell_t\big(\widetilde\btheta_t(\eta) \big) \right\}
    \geq \frac{V_t(\eta)}{V_{t - 1}(\eta)}
    = \frac{W_t(\eta)}{W_{t - 1}(\eta)}.
\]
\myendproof

\subsubsection{Proof of Theorem~\ref{th:fs_regret}}
\label{sec:th_fs_regret_proof}

Note that, due to the inequality~\eqref{eq:eta_t_requirement}, the sequence $\{\eta_t : 1 \leq t \leq T\}$ meets the requirements of Lemma~\ref{lem:fs_mixability}.
Thus, the cumulative loss of the fixed share forecaster does not exceed
\begin{align*}
    \widehat L_{1:T}^{FS}
    &
    = \sum\limits_{t = 1}^T \ell_t(\widehat\btheta{}_t^{FS})
    = \sum\limits_{t = 1}^T \ell_t\big(\widetilde\btheta_t(\eta_t) \big)
    \leq \sum\limits_{t = 1}^T \frac1{\eta_t} \log \frac{W_{t - 1}(\eta_t)}{W_t(\eta_t)}
    \\&
    = \frac1{\eta_1} \log W_0(\eta_1) - \frac1{\eta_T} \log W_T(\eta_T)
    + \sum\limits_{t = 1}^{T - 1} \left( \frac1{\eta_{t + 1}} \log W_t(\eta_{t + 1}) - \frac1{\eta_t} \log W_t(\eta_t) \right)
    \\&
    = - \frac1{\eta_T} \log W_T(\eta_T) + \sum\limits_{t = 1}^{T - 1} \left( \frac1{\eta_{t + 1}} \log W_t(\eta_{t + 1}) - \frac1{\eta_t} \log W_t(\eta_t) \right).
\end{align*}
Proposition~\ref{prop:representation} yields that
\[
    W_t(\eta) = \int \exp\big\{-\eta \cL_t(\btheta_1, \dots, \btheta_T) \big\} \rho(\btheta_1, \dots, \btheta_T) \, \dd \btheta_1 \dots \dd \btheta_T.
\]
Applying H\"older's inequality and taking into account that $\eta_{t + 1} \leq \eta_t$, we obtain that
\begin{align*}    
    \frac1{\eta_{t + 1}} \log W_t(\eta_{t + 1})
    &
    = \frac1{\eta_{t + 1}} \log \int \exp\big\{-\eta_{t + 1} \cL_t(\btheta_1, \dots, \btheta_T) \big\} \rho(\btheta_1, \dots, \btheta_T) \, \dd \btheta_1 \dots \dd \btheta_T
    \\&
    \leq \frac1{\eta_t} \log \int \exp\big\{-\eta_t \cL_t(\btheta_1, \dots, \btheta_T) \big\} \rho(\btheta_1, \dots, \btheta_T) \, \dd \btheta_1 \dots \dd \btheta_T
    \\&
    \leq \frac1{\eta_t} \log W_t(\eta_t).
\end{align*}
Hence, it holds that
\[
    \widehat L_{1:T}^{FS} \leq - \frac1{\eta_T} \log W_T(\eta_T).
\]
In Appendix~\ref{sec:lem_w_1t_bound_proof}, we prove the  following auxiliary result, which helps us to bound the expression in the right-hand side.

\begin{Lem}
    \label{lem:w_1t_bound}
    Under the conditions of Theorem~\ref{th:fs_regret}, for any $\eta > 0$, it holds that
    \begin{align*}
        -\frac1{\eta} \log W_T(\eta)
        &
        \leq \frac{(m - 1) \log(1 / \alpha)}{\eta} + \frac{(T - m) \log\big(1 / (1 - \alpha) \big)}{\eta}
        + \sum\limits_{k = 0}^{m - 1} L_{\tau_k + 1 : \tau_{k + 1}}(\btheta_{\tau_k + 1 : \tau_{k + 1}}^\circ)
        \\&\quad
        + \sum\limits_{k = 0}^{m - 1} \left[ \frac{\lambda \|\btheta_{\tau_k + 1 : \tau_{k + 1}}^\circ\|^2}{2 \eta} + \frac1{2 \eta} \log \det\left( I_d + \frac{\eta}\lambda \sum\limits_{j = \tau_k + 1}^{\tau_{k + 1}} A_j \right) \right].
    \end{align*}
\end{Lem}

Applying Lemma~\ref{lem:w_1t_bound} with $\eta = \eta_T$, we get the desired bound:
\begin{align*}
    &
    \widehat L_{1:T}^{FS} - \sum\limits_{k = 0}^{m - 1} L_{\tau_k + 1 : \tau_{k + 1}}(\btheta_{\tau_k + 1 : \tau_{k + 1}}^\circ)
    \\&
    \leq \frac{(m - 1) \log(1 / \alpha)}{\eta_T} + \frac{(T - m) \log\big(1 / (1 - \alpha) \big)}{\eta_T}
    \\&\quad
    + \sum\limits_{k = 0}^{m - 1} \left[ \frac{\lambda \|\btheta_{\tau_k + 1 : \tau_{k + 1}}^\circ\|^2}{2 \eta_T} + \frac1{2 \eta_T} \log \det\left( I_d + \frac{\eta_T}\lambda \sum\limits_{j = \tau_k + 1}^{\tau_{k + 1}} A_j \right) \right].
\end{align*}
    
\myendproof

\subsubsection{Proof of Lemma~\ref{lem:ew_mixability_aux}}
\label{sec:lem_ew_mixability_aux_proof}

Let us elaborate on the expectation of $e^{-\eta \ell (\btheta)}$:
\begin{align*}
    \E e^{-\eta \ell (\btheta)}
    &
    = \E \exp\left\{ -\frac{\eta}2 \btheta^\top A \btheta + \eta \bb^\top \btheta \right\}
    \\&
    = (2\pi)^{-d/2} \det(\Omega)^{1/2}
    \int\limits_{\R^d} \exp\left\{ -\frac12 \bups^\top (\eta A + \Omega) \bups + (\eta \bb + \Omega \bmu)^\top \bups - \frac12 \bmu^\top \Omega \bmu \right\} \dd \bups.
\end{align*}
Using the representation
\begin{align*}
    \bups^\top (\eta A + \Omega) \bups - 2(\eta \bb + \Omega \bmu)^\top \bups 
    &
    = \left\|(\eta A + \Omega)^{1/2} \bups - (\eta A + \Omega)^{-1/2} (\eta \bb + \Omega \bmu) \right\|^2
    \\&\quad
    - \left\|(\eta A + \Omega)^{-1/2} (\eta \bb + \Omega \bmu) \right\|^2,
\end{align*}
we can rewrite $\E e^{-\eta \ell (\btheta)}$ in the following form:
\begin{align*}
    \E e^{-\eta \ell (\btheta)}
    &
    = (2\pi)^{-d/2} \det(\Omega)^{1/2}
    \exp\left\{ \frac12 \left\|(\eta A + \Omega)^{-1/2} (\eta \bb + \Omega \bmu) \right\|^2 - \frac12 \bmu^\top \Omega \bmu \right\}
    \\&\quad
    \cdot \int\limits_{\R^d} \exp\left\{ -\frac12 \left\|(\eta A + \Omega)^{1/2} \bups - (\eta A + \Omega)^{-1/2} (\eta \bb + \Omega \bmu) \right\|^2 \right\} \dd \bups
    \\&
    = \det(\Omega)^{1/2} \det(\eta A + \Omega)^{-1/2} \exp\left\{ \frac12 \left\|(\eta A + \Omega)^{-1/2} (\eta \bb + \Omega \bmu) \right\|^2 - \frac12 \bmu^\top \Omega \bmu \right\}
    \\&
    = \det\left(I_d + \eta \Omega^{-1/2} A \Omega^{-1/2} \right)^{-1/2} \exp\left\{ \frac12 \left\|(\eta A + \Omega)^{-1/2} (\eta \bb + \Omega \bmu) \right\|^2 - \frac12 \bmu^\top \Omega \bmu \right\}.
\end{align*}
The end of the proof is straightforward due to the auxiliary results presented in Appendix~\ref{sec:aux}.
Indeed, Lemma~\ref{lem:squared_norm_identity} claims that
\[
    \left\|(\eta A + \Omega)^{-1/2} (\eta \bb + \Omega \bmu) \right\|^2 - \bmu^\top \Omega \bmu + \eta \bmu^\top A \bmu - 2 \eta \bb^\top \bmu
    = \eta^2 \left\| (\Omega + \eta A)^{-1/2} (A \bmu - \bb) \right\|^2.
\]
In view of the condition~\eqref{eq:ew_eta_condition}, this yields that
\begin{align*}
    \log \frac{e^{-\eta \ell(\bmu)}}{\E e^{-\eta \ell (\btheta)}}
    &
    = \frac12 \log \det\left(I_d + \eta \Omega^{-1/2} A \Omega^{-1/2} \right) - \frac{\eta^2}2 \left\| (\Omega + \eta A)^{-1/2} (A \bmu - \bb) \right\|^2
    \\&
    \geq \frac12 \log \det\left(I_d + \eta \Omega^{-1/2} A \Omega^{-1/2} \right) - \frac{\eta}2 \left\| A^{1/2} (\Omega + \eta A)^{-1} A^{1/2} \right\|.
\end{align*}
Applying Lemma~\ref{lem:log_det_lower_bound} with $B = \eta A$, we obtain that
\begin{align*}
    \log \frac{e^{-\eta \ell(\bmu)}}{\E e^{-\eta \ell (\btheta)}}
    \geq \frac12 \log \det\left(I_d + \eta \Omega^{-1/2} A \Omega^{-1/2} \right) - \frac{\eta}2 \left\| A^{1/2} (\Omega + \eta A)^{-1} A^{1/2} \right\|
    \geq 0.
\end{align*}
\myendproof

\subsubsection{Proof of Lemma~\ref{lem:z_1t_bound}}
\label{sec:lem_z_1t_bound_proof}

Since the cumulative loss
\[
    L_{1:T}(\btheta)
    = \sum\limits_{t = 1}^T \ell_t(\btheta)
    = \frac12 \btheta^\top \left( \sum\limits_{t = 1}^T A_t \right) \btheta - \btheta^\top \left( \sum\limits_{t = 1}^T \bb_t \right)
\]
is a quadratic function and $\btheta_{1:T}^\circ$ minimizes $L_{1:T}(\btheta)$ over all $\btheta \in \R^d$, Taylor's formula implies that
\[
    L_{1:T}(\btheta) - L_{1:T}(\btheta_{1:T}^\circ) = \frac12 \left\| \left( \sum\limits_{t = 1}^T A_t \right)^{1/2} (\btheta - \btheta_{1:T}^\circ) \right\|^2.
\]
Then we can represent $Z_{1:T}(\eta)$ in the following form:
\begin{align}
    \label{eq:z_1t_representation}
    Z_{1:T}(\eta)
    &\notag
    = \int\limits_{\R^d} e^{-\eta L_{1:T}(\btheta)} \, \pi(\btheta) \dd \btheta
    \\&
    = e^{-\eta L_{1:T}(\btheta_{1:T}^\circ)} \int\limits_{\R^d} \exp\left\{-\frac{\eta}2 \left\| \left( \sum\limits_{t = 1}^T A_t \right)^{1/2} (\btheta - \btheta_{1:T}^\circ) \right\|^2 \right\} \, \pi(\btheta) \dd \btheta
    \\&\notag
    = \left( \frac{\lambda}{2\pi} \right)^{d/2} e^{-\eta L_T(\btheta_{1:T}^\circ)} \int\limits_{\R^d} \exp\left\{-\frac{\eta}2 \left\| \left( \sum\limits_{t = 1}^T A_t \right)^{1/2} (\btheta - \btheta_{1:T}^\circ) \right\|^2 - \frac{\lambda}2 \|\btheta\|^2 \right\} \, \dd \btheta.
\end{align}
Let us elaborate on the power of the exponent under the integral. It holds that
\begin{align*}
    &
    \eta \left\| \left( \sum\limits_{t = 1}^T A_t \right)^{1/2} (\btheta - \btheta_{1:T}^\circ) \right\|^2 + \lambda \|\btheta\|^2
    \\&
    = \btheta^\top \left( \lambda I_d + \eta \sum\limits_{t = 1}^T A_t \right) \btheta
    - 2 \eta \sum\limits_{t = 1}^T \btheta^\top A_t \btheta_{1:T}^\circ + \eta \left\| \left( \sum\limits_{t = 1}^T A_t \right)^{1/2} \btheta_{1:T}^\circ \right\|^2
    \\&
    = \btheta^\top \left( \lambda I_d + \eta \sum\limits_{t = 1}^T A_t \right) \btheta + \eta \left\| \left( \sum\limits_{t = 1}^T A_t \right)^{1/2} \btheta_{1:T}^\circ \right\|^2
    \\&\quad
    - 2 \eta \btheta^\top \left( \lambda I_d + \eta \sum\limits_{t = 1}^T A_t \right)^{1/2} \left( \lambda I_d + \eta \sum\limits_{t = 1}^T A_t \right)^{-1/2} \left( \sum\limits_{t = 1}^T A_t \right) \btheta_{1:T}^\circ.
\end{align*}
Since
\begin{align*}
    &
    \eta \sum\limits_{t = 1}^T A_t - \left( \eta \sum\limits_{t = 1}^T A_t \right) \left( \lambda I_d + \eta \sum\limits_{t = 1}^T A_t \right)^{-1} \left( \eta \sum\limits_{t = 1}^T A_t \right)
    \\&
    = \lambda \left( \eta \sum\limits_{t = 1}^T A_t \right)^{1/2} \left( \lambda I_d + \eta \sum\limits_{t = 1}^T A_t \right)^{-1} \left( \eta \sum\limits_{t = 1}^T A_t \right)^{1/2}
    \preceq \lambda I_d,
\end{align*}
we obtain that
\begin{align*}
    &
    \eta \left\| \left( \sum\limits_{t = 1}^T A_t \right)^{1/2} (\btheta - \btheta_{1:T}^\circ) \right\|^2 + \lambda \|\btheta\|^2
    \\&
    = \left\| \left( \lambda I_d + \eta \sum\limits_{t = 1}^T A_t \right)^{1/2} \btheta - \eta \left( \lambda I_d + \eta \sum\limits_{t = 1}^T A_t \right)^{-1/2} \left( \sum\limits_{t = 1}^T A_t \right) \btheta_{1:T}^\circ \right\|^2
    \\&\quad
    + \lambda \left\| \left( \lambda I_d + \eta \sum\limits_{t = 1}^T A_t \right)^{-1/2} \left( \eta \sum\limits_{t = 1}^T A_t \right)^{1/2} \btheta_{1:T}^\circ \right\|^2
    \\&
    \leq \left\| \left( \lambda I_d + \eta \sum\limits_{t = 1}^T A_t \right)^{1/2} \btheta - \eta \left( \lambda I_d + \eta \sum\limits_{t = 1}^T A_t \right)^{-1/2} \left( \sum\limits_{t = 1}^T A_t \right) \btheta_{1:T}^\circ \right\|^2
    + \lambda \|\btheta_{1:T}^\circ\|^2.
\end{align*}
This inequality and~\eqref{eq:z_1t_representation} yield that
\begin{align*}
    Z_{1:T}(\eta)
    &
    \geq \left( \frac{\lambda}{2\pi} \right)^{d/2} \exp\left\{-\eta L_T(\btheta_{1:T}^\circ) - \frac{\lambda \|\btheta_{1:T}^\circ\|^2}2 \right\}
    \\&\quad
    \cdot \int\limits_{\R^d} \exp\left\{-\frac12 \left\| \left( \lambda I_d + \eta \sum\limits_{t = 1}^T A_t \right)^{1/2} \btheta - \left( \frac{\lambda}{\eta} I_d + \sum\limits_{t = 1}^T A_t \right)^{-1/2} \left( \sum\limits_{t = 1}^T A_t \right) \btheta_{1:T}^\circ \right\|^2 \right\} \dd \btheta
    \\&
    = \lambda^{d/2} \exp\left\{-\eta L_T(\btheta_{1:T}^\circ) - \frac{\lambda \|\btheta_{1:T}^\circ\|^2}2 \right\} \det\left( \lambda I_d + \eta \sum\limits_{t = 1}^T A_t \right)^{-1/2}
    \\&
    = \exp\left\{-\eta L_T(\btheta_{1:T}^\circ) - \frac{\lambda \|\btheta_{1:T}^\circ\|^2}2 \right\} \det\left( I_d + \frac{\eta}{\lambda} \sum\limits_{t = 1}^T A_t \right)^{-1/2}.
\end{align*}
Hence, $-\log\big(Z_{1:T}(\eta) \big) / \eta$ satisfies the inequality
\[
    -\frac1{\eta} \log Z_{1:T}(\eta)
    \leq L_T(\btheta_{1:T}^\circ)
    + \frac{\lambda \|\btheta_{1:T}^\circ\|^2}{2 \eta}
    + \frac1{2\eta} \log \det\left( I_d + \frac{\eta}{\lambda} \sum\limits_{t = 1}^T A_t \right).
\]
\myendproof

\subsubsection{Proof of Lemma~\ref{lem:w_1t_bound}}
\label{sec:lem_w_1t_bound_proof}

Let us remind to the reader that
\[
    \rho(\btheta_1, \dots, \btheta_T) = \pi(\btheta_1) \prod\limits_{t = 2}^T \f(\btheta_t \,\vert\, \btheta_{t - 1}),
    \quad \text{where} \quad
    \f(\btheta_t \,\vert\, \btheta_{t - 1}) = \alpha \pi(\btheta_t) + (1 - \alpha) \delta(\btheta_t - \btheta_{t - 1}).
\]
Let $0 = \tau_0 < \tau_1 < \dots < \tau_m = T$ be as defined in the statement of Theorem~\ref{th:fs_regret}. Note that
\begin{align}
    \label{eq:rho_lower_bound}
    \rho(\btheta_1, \dots, \btheta_T)
    &\notag
    \geq \pi(\btheta_{\tau_0 + 1}) \prod\limits_{t = \tau_0 + 2}^{\tau_1} \big[ (1 - \alpha) \delta(\btheta_t - \btheta_{t - 1}) \big]
    \\&\quad\notag
    \cdot \alpha \pi(\btheta_{\tau_1 + 1}) \prod\limits_{t = \tau_1 + 2}^{\tau_2} \big[ (1 - \alpha) \delta(\btheta_t - \btheta_{t - 1}) \big]
    \cdot \dots
    \\&\quad
    \cdot \alpha \pi(\btheta_{\tau_{m - 1} + 1}) \prod\limits_{t = \tau_{m - 1} + 2}^{\tau_m} \big[ (1 - \alpha) \delta(\btheta_t - \btheta_{t - 1}) \big]
    \\&\notag
    = \alpha^{m - 1} (1 - \alpha)^{T - m} \prod\limits_{k = 0}^m \left( \pi(\btheta_{\tau_k + 1}) \prod\limits_{t = \tau_k + 2}^{\tau_{k + 1}} \delta(\btheta_t - \btheta_{t - 1}) \right).
\end{align}
Here we use the convention
\[
    \prod\limits_{t = \tau_k + 2}^{\tau_{k + 1}} \delta(\btheta_t - \btheta_{t - 1}) = 1
    \quad \text{when $\tau_{k + 1} = \tau_k + 1$.}
\]
The inequality~\eqref{eq:rho_lower_bound} implies that
\begin{align}
    \label{eq:wt_lower_bound}
    &
    W_T(\eta)
    = \int e^{-\eta \cL_T(\btheta_1, \dots, \btheta_T)} \rho(\btheta_1, \dots, \btheta_T) \, \dd \btheta_1 \dots \dd \btheta_T
    \\&\notag
    \geq \alpha^{m - 1} (1 - \alpha)^{T - m} \int e^{-\eta \cL_T(\btheta_1, \dots, \btheta_T)} \prod\limits_{k = 0}^{m - 1} \left( \pi(\btheta_{\tau_k + 1}) \, \dd \btheta_{\tau_k + 1} \prod\limits_{t = \tau_k + 2}^{\tau_{k + 1}} \delta(\btheta_t - \btheta_{t - 1}) \, \dd \btheta_t \right).
\end{align}
Factorizing the integral in the right-hand side, we obtain that $W_T(\eta)$ is not smaller than
\begin{align}
    \label{eq:factorization}
    &\notag
    \alpha^{m - 1} (1 - \alpha)^{T - m} \prod\limits_{k = 0}^{m - 1} \left( \int \exp\left\{-\eta \sum\limits_{t = \tau_k + 1}^{\tau_{k + 1}} \ell_t(\btheta_t) \right\} \pi(\btheta_{\tau_k + 1}) \, \dd \btheta_{\tau_k + 1} \prod\limits_{t = \tau_k + 2}^{\tau_{k + 1}} \delta(\btheta_t - \btheta_{t - 1}) \, \dd \btheta_t \right)
    \\&
    = \alpha^{m - 1} (1 - \alpha)^{T - m} \prod\limits_{k = 0}^{m - 1} \left( \int \exp\left\{-\eta L_{\tau_k + 1 : \tau_{k + 1}} (\btheta_{\tau_k + 1}) \right\} \pi(\btheta_{\tau_k + 1}) \, \dd \btheta_{\tau_k + 1} \right)
    \\&\notag
    = \alpha^{m - 1} (1 - \alpha)^{T - m} \prod\limits_{k = 0}^{m - 1} \left( \int \exp\left\{-\eta L_{\tau_k + 1 : \tau_{k + 1}} (\bups_k) \right\} \pi(\bups_k) \, \dd \bups_k \right).
\end{align}
For any $k \in \{0, \dots, m - 1\}$, the expression
\[
    \int \exp\left\{-\eta L_{\tau_k + 1 : \tau_{k + 1}} (\bups_k) \right\} \pi(\bups_k) \, \dd \bups_k
\]
can be bounded from below in the same way as in the proof of Lemma~\ref{lem:z_1t_bound}:
\begin{align}
    \label{eq:ew_lower_bound}
    &\notag
    \int \exp\left\{-\eta L_{\tau_k + 1 : \tau_{k + 1}} (\bups_k) \right\} \pi(\bups_k) \, \dd \bups_k
    \\&
    \geq
    \exp\left\{ -\eta L_{\tau_k + 1 : \tau_{k + 1}}(\btheta_{\tau_k + 1 : \tau_{k + 1}}^\circ) - \frac{\lambda \|\btheta_{\tau_k + 1 : \tau_{k + 1}}^\circ\|^2}2 \right\}  \det\left( I_d + \frac{\eta}\lambda \sum\limits_{j = \tau_k + 1}^{\tau_{k + 1}} A_j \right)^{1/2}.
\end{align}
Summing up~\eqref{eq:wt_lower_bound}, \eqref{eq:factorization}, and~\eqref{eq:ew_lower_bound}, we obtain that
\begin{align*}
    -\frac1{\eta} \log W_T(\eta)
    &
    \leq \frac{(m - 1) \log(1 / \alpha)}{\eta} + \frac{(T - m) \log\big(1 / (1 - \alpha) \big)}{\eta}
    + \sum\limits_{k = 0}^{m - 1} L_{\tau_k + 1 : \tau_{k + 1}}(\btheta_{\tau_k + 1 : \tau_{k + 1}}^\circ)
    \\&\quad
    + \sum\limits_{k = 0}^{m - 1} \left[ \frac{\lambda \|\btheta_{\tau_k + 1 : \tau_{k + 1}}^\circ\|^2}{2 \eta} + \frac1{2 \eta} \log \det\left( I_d + \frac{\eta}\lambda \sum\limits_{j = \tau_k + 1}^{\tau_{k + 1}} A_j \right) \right].
\end{align*}
\myendproof

\section{Auxiliary Results}

\subsection{Auxiliary Results from Linear Algebra}
\label{sec:aux}

\begin{Lem}
    \label{lem:squared_norm_identity}
    With the notations of Lemma~\ref{lem:ew_mixability_aux}, it holds that
    \begin{align}
        \label{eq:squared_norm_identity}
        &\notag
        \left\|(\eta A + \Omega)^{-1/2} (\eta \bb + \Omega \bmu) \right\|^2 - \bmu^\top \Omega \bmu + \eta \bmu^\top A \bmu - 2 \eta \bb^\top \bmu
        \\&
        = \eta^2 \left\| (\Omega + \eta A)^{-1/2} (A \bmu - \bb) \right\|^2.
    \end{align}
\end{Lem}

\begin{proof}
    We start with rewriting the left-hand side of~\eqref{eq:squared_norm_identity} in the following form:
    \begin{align*}
        &
        \left\|(\eta A + \Omega)^{-1/2} (\eta \bb + \Omega \bmu) \right\|^2 - \bmu^\top \Omega \bmu + \eta \bmu^\top A \bmu - 2 \eta \bb^\top \bmu
        \\&
        = \bmu^\top \left( \Omega (\eta A + \Omega)^{-1} \Omega - \Omega + \eta A \right) \bmu^\top 
        - 2 \eta \bb^\top \left( I_d - (\eta A + \Omega)^{-1} \Omega \right) \bmu + \eta^2 \bb^\top (\eta A + \Omega)^{-1} \bb.
    \end{align*}
    The identity~\eqref{eq:squared_norm_identity} simply follows from Lemma~\ref{lem:inverse_difference} below.
    Applying it to $\Omega (\eta A + \Omega)^{-1} \Omega - \Omega + \eta A$ and to $I_d - (\eta A + \Omega)^{-1} \Omega$, we obtain that
    \begin{align*}
        \Omega (\eta A + \Omega)^{-1} \Omega - \Omega + \eta A
        &
        = \Omega \left( (\eta A + \Omega)^{-1} - \Omega^{-1} \right) \Omega + \eta A
        = \eta A - \eta \Omega (\Omega + \eta A)^{-1} A
        \\&
        = \Omega \left( \Omega^{-1} - (\Omega + \eta A)^{-1} \right) \eta A
        = \eta^2 A (\Omega + \eta A)^{-1} A
    \end{align*}
    and
    \begin{align*}
        I_d - (\eta A + \Omega)^{-1} \Omega
        = \Omega \left( \eta \Omega^{-1} A (\Omega + \eta A)^{-1} \right) 
        = \eta A (\Omega + \eta A)^{-1}.
    \end{align*}
    Thus,
    \begin{align*}
        &
        \bmu^\top \left( \Omega (\eta A + \Omega)^{-1} \Omega - \Omega + \eta A \right) \bmu^\top 
        - 2 \eta \bb^\top \left( I_d - (\eta A + \Omega)^{-1} \Omega \right) \bmu + \eta^2 \bb^\top (\eta A + \Omega)^{-1} \bb
        \\&
        = \eta^2 \bmu^\top A (\Omega + \eta A)^{-1} A \bmu - 2 \eta^2 \bmu^\top A (\Omega + \eta A)^{-1} \bb + \eta^2 \bb^\top (\eta A + \Omega)^{-1} \bb
        \\&
        = \eta^2 \left\| (\Omega + \eta A)^{-1/2} (A \bmu - \bb) \right\|^2.
    \end{align*}
\end{proof}

\begin{Lem}
    \label{lem:inverse_difference}
    Let $\Omega$ and $B$ be symmetric positive semidefinite matrices of size $(d \times d)$, such that $\det(\Omega) \neq 0$. Then it holds that
    \[
        \Omega^{-1} - (\Omega + B)^{-1}
        = (\Omega + B)^{-1} B \Omega^{-1}
        = \Omega^{-1} B (\Omega + B)^{-1}.
    \]
\end{Lem}

\begin{proof}
    Denote the difference $\Omega^{-1} - (\Omega + B)^{-1}$ by $M$. Then it holds that
    \[
        M (\Omega + B) = \Omega^{-1} (\Omega + B) - I_d = \Omega^{-1} B.
    \]
    Thus, $M = \Omega^{-1} B (\Omega + B)^{-1} $. Similarly, one has
    \[
        (\Omega + B) M = (\Omega + B) \Omega^{-1} - I_d = B \Omega^{-1}
    \]
    and, hence, $M = (\Omega + B)^{-1} B \Omega^{-1}$.

\end{proof}

\begin{Lem}
    \label{lem:log_det_lower_bound}
    Let $\Omega$ and $B$ be symmetric positive semidefinite matrices of size $(d \times d)$, such that $\det(\Omega) \neq 0$. Then it holds that
    \[
        \left\| B^{1/2} (\Omega + B)^{-1} B^{1/2} \right\|
        \leq \log \det \left(I_d + \Omega^{-1/2} B \Omega^{-1/2} \right).
    \]
\end{Lem}

\begin{proof}
    The proof is based on an observation that the operator norm of $B^{1/2} (\Omega + B)^{-1} B^{1/2}$ can be expressed through the eigenvalues $\lambda_1( \Omega^{-1/2} B \Omega^{-1/2}), \dots$, $\lambda_d( \Omega^{-1/2} B \Omega^{-1/2})$ of the matrix $\Omega^{-1/2} B \Omega^{-1/2}$ as follows:
    \begin{align*}
        \left\| B^{1/2} (\Omega + B)^{-1} B^{1/2} \right\|
        &
        = \max\limits_{1 \leq j \leq d} \left\{ \frac{\lambda_j(\Omega^{-1/2} B \Omega^{-1/2})}{1 + \lambda_j(\Omega^{-1/2} B \Omega^{-1/2})} \right\}
        \\&
        \leq \max\limits_{1 \leq j \leq d} \log\left( 1 + \lambda_j(\Omega^{-1/2} B \Omega^{-1/2}) \right).
    \end{align*}
    Here the last line follows from the inequality $x / (1 + x) \leq \log(1 + x)$, which holds for all non-negative $x$.
    Since $\log\left( 1 + \lambda_j(\Omega^{-1/2} B \Omega^{-1/2}) \right) \geq 0$ for all $j \in \{1, \dots, d\}$, it holds that
    \begin{align*}
        \max\limits_{1 \leq j \leq d} \log\left( 1 + \lambda_j(\Omega^{-1/2} B \Omega^{-1/2}) \right)
        &
        \leq \sum\limits_{1 \leq j \leq d} \log\left( 1 + \lambda_j(\Omega^{-1/2} B \Omega^{-1/2}) \right)
        \\&
        = \log \det \left(I_d + \Omega^{-1/2} B \Omega^{-1/2} \right).
    \end{align*}
    Hence,
    \[
        \left\| B^{1/2} (\Omega + B)^{-1} B^{1/2} \right\|
        \leq \log \det \left(I_d + \Omega^{-1/2} B \Omega^{-1/2} \right).
    \]
\end{proof}

\subsection{Auxiliary Results from Probability Theory}

\begin{Lem}[for instance, \citeauthor{puchkin25}, \citeyear{puchkin25}, Lemma F.8]
    \label{lem:moment_subexp_norm_bound}
    Let \(\xi\) be an arbitrary random variable with a finite Orlicz norm \(\|\xi\|_{\psi_1} < +\infty\). Then, for any \(p > 0\), it holds that
    \[
        \E|\xi|^p \le 2\Gamma(p+1)\|\xi\|^p_{\psi_1} \le 2(p+1)^p\|\xi\|^p_{\psi_1}.
    \]
\end{Lem}

\begin{Th}[\citeauthor{kroshnin25}, \citeyear{kroshnin25}, Theorem 2.1]
Let \(X_1, \dots, X_n \in \bbH(d)\), where \(\bbH(d)\) denotes the space of \(d \times d\) Hermitian matrices, be a supermartingale difference sequence adapted to a filtration \((\cF_t)_{i=0}^n\) with \(\cF_0 = \{\Omega, \varnothing\}\). Fix \(\alpha > 0\) and set for all \(i \in \{1, \dots, n\}\)
\[
    \Sigma_i := \E[ X_i^2 \mid \cF_{i-1} ], \quad U_i := \| \lambda_{\max} (X_i)_+ \mid \cF_{i-1} \|_{\psi_\alpha}.
\]
Fix \(\upvarsigma > 0, U \ge K > 0\) and define the event
\[
    \cE := \left\{ \lambda_{\max} \left( \sum_{i=1}^n \Sigma_i \right) \le \upvarsigma^2, \quad \sum_{i=1}^n U^2_i \le U^2, \quad \max_{i \in \{1, \dots, n\}} U_i \le K \right\}.
\]
Let 
\[
    z = z(U, \upvarsigma, \alpha) := \begin{cases}
        \left( 4 \vee 4\log\frac{eU}{\upvarsigma} \right)^{1/\alpha}, \hfill \text{ if } \alpha \ge 1, \\
        \left[ \frac4\alpha \log\frac{e}{\alpha} + 4\left(\log\frac{U}{\upvarsigma}\right)_{+} \right], \hfill \text{ if } \alpha < 1.
    \end{cases}
\]
Then for any \(x > 0\), it holds, with probability at least \(\p(\cE) - de^{-x} - e^{-x} \:\I[\alpha < 1]\), that
\[
    \max_{k \in \{1, \dots, n\}} \lambda_{\max} \left( \sum_{i=1}^k X_i \right) \le \upvarsigma \sqrt{2x} + \frac34 Kzx + \frac{3K}{\alpha} x\left( 2x + 2\log\frac{4U}{K} + \frac4\alpha \log\frac{4}{\alpha e} \right)^{\frac{1-\alpha}{\alpha}} \I[\alpha < 1].
\]
\end{Th}

\begin{Lem}
    \label{lem:expectation_prod_bound}
    Let \(\xi\) and \(\eta\) be non-negative random variables with finite \(\psi_1\)-norms. Then it holds that 
    \begin{equation}
        \label{eq:expectation_prod_bound}
        \E \xi\eta
        \le 3e \, \E \xi \, \|\eta\|_{\psi_1} \log\frac{4\sqrt{\E\xi^2}}{\E\xi}
        \le 3e \, \E \xi \, \|\eta\|_{\psi_1} \log\frac{8\|\xi\|_{\psi_1}}{\E\xi}
    \end{equation}
\end{Lem}

\begin{proof}
    Let us take
    \[
        \eps = \left(\log\frac{4\sqrt{\E\xi^2}}{\E\xi}\right)^{-1} \in (0, 1).
    \]
    Applying H\"older's inequality, we obtain that
    \begin{align*}
        \E\xi\eta
        = \E \left( \xi^{1-\eps} \cdot \xi^\eps \eta \right)
        \le \big(\E\xi\big)^{1-\eps} \big(\E \, \xi \eta^{1/\eps} \big)^{\eps}.
    \end{align*}
    Due to the Cauchy-Schwarz bound and Lemma \ref{lem:moment_subexp_norm_bound}, it holds that
    \begin{align*}
        \big(\E\xi\big)^{1 - \eps} \big(\E \, \xi\eta^{1/\eps}\big)^{\eps} 
        \le 
        \big(\E\xi\big)^{1-\eps} \big(\E\xi^2\big)^{\eps/2} \big(\E\eta^{2/\eps}\big)^{\eps/2} 
        \le
        \E\xi \left(\frac{\sqrt{\E\xi^2}}{\E\xi}\right)^\eps \cdot 2^{\eps/2} \left(1+\frac2\eps\right) \|\eta\|_{\psi_1}.
    \end{align*}
    Recalling the definition of $\eps$, we deduce that
    \begin{align*}
        \E\xi \left(\frac{\sqrt{\E\xi^2}}{\E\xi}\right)^\eps \cdot 2^{2/\eps} \left(1+\frac2\eps\right) \|\eta\|_{\psi_1} \le  \frac3\eps \E\xi \left(\frac{4\sqrt{\E\xi^2}}{\E\xi}\right)^\eps \|\eta\|_{\psi_1} = 3e \E\xi \|\eta\|_{\psi_1} \log\frac{4\sqrt{\E\xi^2}}{\E\xi}.
    \end{align*}
    The second inequality in \eqref{eq:expectation_prod_bound} directly follows from Lemma \ref{lem:moment_subexp_norm_bound} claiming that $\E \xi^2 \leq 4 \|\xi\|_{\psi_1}$.
    
\end{proof}

\end{document}